\documentclass[11pt, a4paper]{thuc3i}
\usepackage[sort&compress]{natbib}
\setheadertext{How Far Can Unsupervised RLVR Scale LLM Training?}

\usepackage[utf8]{inputenc} % allow utf-8 input
\usepackage[T1]{fontenc}    % use 8-bit T1 fonts
\usepackage{hyperref}       % hyperlinks
\usepackage{url}            % simple URL typesetting
\usepackage{tabularx}
\usepackage{makecell}
\usepackage{array,ragged2e,booktabs}
\newcolumntype{P}[1]{>{\RaggedRight\arraybackslash}p{#1}}
\usepackage{longtable}
\usepackage{amsfonts}       % blackboard math symbols
\usepackage{amsthm}         % theorem environments
\newtheorem{theorem}{Theorem}

\newtheorem{proposition}{Proposition}
\usepackage{nicefrac}       % compact symbols for 1/2, etc.
\usepackage{microtype}      % microtypography
\usepackage{xcolor}         % colors

\usepackage{algorithm}
\usepackage{algorithmicx}
\usepackage{algpseudocode}

\definecolor{darkblue}{rgb}{0, 0, 0.5}
\hypersetup{colorlinks=true, citecolor=darkblue, linkcolor=darkblue, urlcolor=darkblue}

\usepackage{xurl}

\usepackage{soul}
\usepackage{amsmath}
\usepackage{subcaption}
\usepackage{graphicx}
\usepackage{changes}
\usepackage{mathtools}
\usepackage{changes}
\usepackage{pifont}
\usepackage{tocloft}

\usepackage{lipsum}
\usepackage{colortbl}
\usepackage{wrapfig}
\usepackage{xspace}
\usepackage{multirow}
\usepackage{enumitem}

\usepackage{pifont}
\usepackage{listings}
\usepackage{fontawesome5} 

\usepackage{wrapfig}
\usepackage{caption}

\usepackage{makecell}

\usepackage{tabularx}
\usepackage{afterpage}

\usepackage[edges]{forest}
\usepackage[normalem]{ulem}
\usepackage{caption}
\usepackage{CJKutf8}
\usepackage{bbding}
\usepackage[most,skins,theorems]{tcolorbox}
\usepackage{epigraph}

\usepackage[tikz]{bclogo}
\usepackage[framemethod=tikz]{mdframed}

\definecolor{darkblue}{rgb}{0, 0, 0.5}
\hypersetup{colorlinks=true, citecolor=darkblue, linkcolor=darkblue, urlcolor=darkblue}

\definecolor{lightorange}{HTML}{faa755}
\definecolor{lightblue}{RGB}{220,235,250}
\usepackage[most,skins,theorems]{tcolorbox}
\tcbset{
  takeawaysbox/.style={
    title=Takeaway,
    colback=lightblue!80,
    colframe=black,
    fonttitle=\bfseries\small,
    coltitle=white,
    colbacktitle=black,
    enhanced,
    attach boxed title to top left={xshift=2.5mm,yshift=-2.5mm},
    boxed title style={rounded corners, size=small, colframe=black, colback=black},
    width=\linewidth,
    arc=3.5mm
  }
}
\tcbset{
  taskbox/.style={
    title=Task Definition,
    colback=lightblue!80,
    colframe=black,
    fonttitle=\bfseries\small,
    coltitle=white,
    colbacktitle=black,
    enhanced,
    attach boxed title to top left={xshift=2.5mm,yshift=-2.5mm},
    boxed title style={rounded corners, size=small, colframe=black, colback=black},
    width=\linewidth,
    arc=2mm
  }
}
\usepackage{xcolor}

\usepackage{xspace}

\usepackage{cleveref}

\usepackage{enumitem}

\usepackage[table]{xcolor}  % 支持表格上色

\definecolor{darkred}{RGB}{200, 0, 0}

\newcommand{\topic}{\textbf{\textcolor{darkred}{Unsupervised RLVR}}\xspace}

\NewDocumentCommand{\yuxin}{ mO{} }{\textcolor{red}{\textsuperscript{yuxin}\textbf{\small[#1]}}}

\NewDocumentCommand{\bx}{ mO{} }{\textcolor{blue}{\textsuperscript{bingxiang}\textbf{\small[#1]}}}

\NewDocumentCommand{\xiusi}{ mO{} }{\textcolor{cyan}{\textsuperscript{\textit{Xiusi}}\textsf{\textbf{\small[#1]}}}}

\def\github{\raisebox{-1pt}{\includegraphics[height=1.05em]{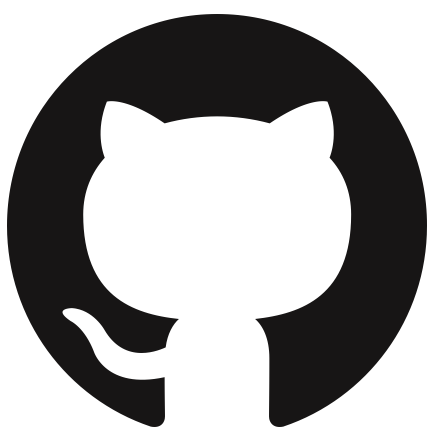}}\kern+0.2em}

\usepackage{perpage}
\MakePerPage{footnote}

\title{How Far Can Unsupervised RLVR Scale LLM Training?}

\author{%
    Bingxiang He$^{*1}$, Yuxin Zuo$^{*\dagger1,2}$, Zeyuan Liu$^{*1}$, Shangziqi Zhao$^{*3}$, Zixuan Fu$^{1}$, Junlin Yang$^{1}$, Cheng Qian$^{4}$, Kaiyan Zhang$^{1,5}$, Yuchen Fan$^{6}$, Ganqu Cui$^{2}$, Xiusi Chen$^{4}$, Youbang Sun$^{1}$, Xingtai Lv$^{1}$, Xuekai Zhu$^{6}$, Li Sheng$^{1}$, Ran Li$^{1}$, Huan-ang Gao$^{1}$, Yuchen Zhang$^{7}$, Bowen Zhou$^{\ddagger1,2}$, Zhiyuan Liu$^{\ddagger1}$, Ning Ding$^{\ddagger1,2}$ \\
    $^{1}$Tsinghua University \quad
    $^{2}$Shanghai AI Lab \quad
    $^{3}$Xi'an Jiaotong University \quad
    $^{4}$University of Illinois Urbana-Champaign \\
    $^{5}$Frontis.AI \quad
    $^{6}$Shanghai Jiao Tong University \quad
    $^{7}$Peking University
    \vskip -1mm
    \textbf{$^*$Equal Contribution. Orders are determined randomly.} \quad 
    \textbf{$^\dagger$Project Lead.} \quad \textbf{$^\ddagger$Corresponding Authors.}
    \vskip -1mm
    \github \url{https://github.com/PRIME-RL/TTRL}. \\
    \vskip 1mm
    \faEnvelope[regular] \texttt{hebx24@mails.tsinghua.edu.cn, dingning@mail.tsinghua.edu.cn}
}

\renewcommand{\today}{2026-03-10}

\begin{abstract}
Unsupervised reinforcement learning with verifiable rewards (URLVR) offers a pathway to scale LLM training beyond the supervision bottleneck by deriving rewards without ground truth labels.
Recent works leverage model intrinsic signals, showing promising early gains, yet their potential and limitations remain unclear.
In this work, we revisit URLVR and provide a comprehensive analysis spanning taxonomy, theory and extensive experiments.
We first classify URLVR methods into intrinsic versus external based on reward sources, then establish a unified theoretical framework revealing that all intrinsic methods converge toward sharpening the model's initial distribution
This sharpening mechanism succeeds when initial confidence aligns with correctness but fails catastrophically when misaligned.
Through systematic experiments, we show intrinsic rewards consistently follow a rise-then-fall pattern across methods, with collapse timing determined by model prior rather than engineering choices.
Despite these scaling limits, we find intrinsic rewards remain valuable in test-time training on small datasets, and propose Model Collapse Step to measure model prior, serving as a practical indicator for RL trainability.
Finally, we explore external reward methods that ground verification in computational asymmetries, showing preliminary evidence they may escape the confidence-correctness ceiling.
Our findings chart boundaries for intrinsic URLVR while motivating paths toward scalable alternatives.

\end{abstract}

\begin{document}

\maketitle

% \vspace{1em}
% \noindent\makebox[\textwidth]{
% \parbox{0.8\textwidth}{\centering
% \textit{``He that is taught only by himself has a fool for a master.''}\\[0.5em]
% \hfill --- Ben Jonson
% }}
% \vspace{1em}

% \hspace{100pt}\parbox[b]{0.75\textwidth}
% {
% \epigraph{\textit{``He that is taught only by himself has a fool for a master.''}}{--- Ben Jonson}
% }

% "If intelligence was a cake, unsupervised learning would be the cake, supervised learning would be the icing on the cake, and reinforcement learning would be the cherry on the cake." -- Yann LeCun

% "Unsupervised learning is the only thing that makes sense. Humans don’t learn with full supervision." -- Geoffrey Hinton

% "A lot of human learning comes from unsupervised learning where you're just sort of observing the world around you and understanding how things behave." - Jeff Dean

\begin{figure}[h]
    \centering
    \includegraphics[width=\linewidth]{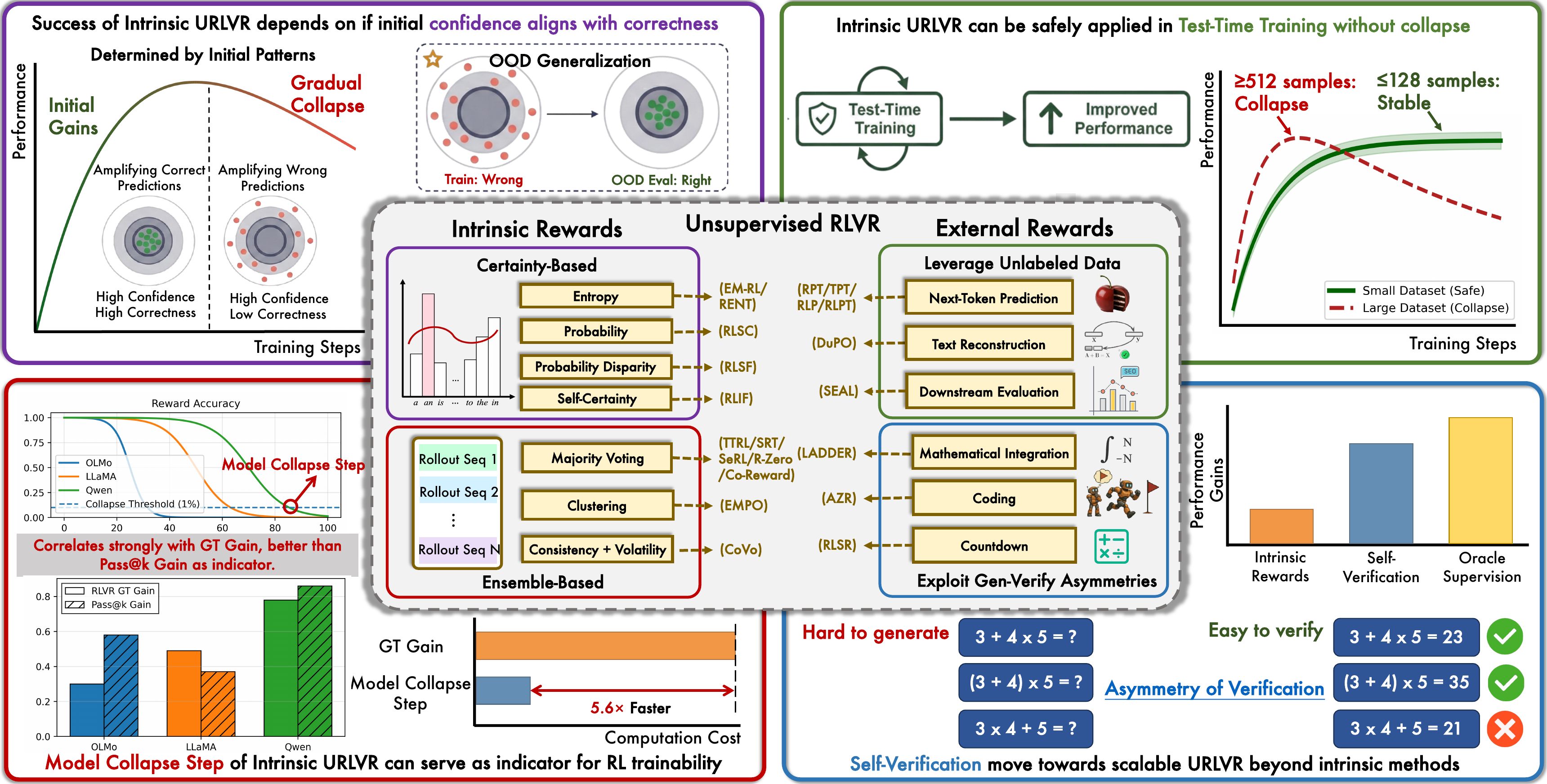}
    \caption{Overview of our paper's framework. At the center is the taxonomy of Unsupervised RLVR methods, categorized into intrinsic rewards and external rewards. The four surrounding panels illustrate the key findings of our empirical investigation.}
    \label{fig:framework}
\end{figure}

\newpage
\begingroup
\setlength{\baselineskip}{1.25\baselineskip}
% {
% \hypersetup{linkcolor=RoyalBlue, linktoc=page}
% \tableofcontents
% }
\tableofcontents
\endgroup
% {
%   \hypersetup{linkcolor=RoyalBlue, linktoc=page}
%   %\hypersetup{linkcolor=black, linktoc=section}
%   \tableofcontents
% }

\newpage

\section{Introduction}

Reinforcement learning with verifiable rewards (RLVR) has been central to recent breakthroughs in enhancing reasoning capability in large language models (LLMs).
In RLVR, models learn from rewards that can be verified against ground truth, such as correctness in mathematics or successful code execution.
Recent leading models including
% OpenAI's o3~\citep{openai-o3},
DeepSeek-R1~\citep{guo2025deepseek}, Gemini 2.5~\citep{comanici2025gemini25} and the Qwen3 series~\citep{yang2025qwen3,qwq32b} have achieved remarkable performance on mathematics, coding and science benchmarks by scaling supervised RLVR.
However, on the path toward superintelligence, this approach faces a crucial limitation: scaling supervision requires prohibitively high human costs, and as models reach or surpass human expertise in specialized domains, obtaining reliable ground truth supervision becomes increasingly infeasible~\citep{burns2023weak,silver2025welcome}.

This supervision bottleneck has spurred growing interest in Unsupervised RLVR (URLVR)~\citep{zuo2025ttrl,zhao2025absolute}, which derives rewards without ground truth labels.
Just as pretraining scaling laws~\citep{brown2020language,raffel2020exploring} transformed large-scale computation into intelligence on vast amounts of unlabeled data, URLVR promises to extend this paradigm to post-training, scaling reinforcement learning beyond reliance on human-provided labels.

Current URLVR methods have primarily relied on leveraging the model's intrinsic signals as rewards.
Many works have proposed intrinsic reward methods, from majority voting across multiple rollouts~\citep{zuo2025ttrl} to entropy-based metrics~\citep{agarwal2025unreasonable} as rewards, reporting encouraging early training gains.
Yet these gains come with growing concerns, as several studies highlight critical failure modes including reward hacking and model collapse~\citep{shafayat2025can,agarwal2025unreasonable,zhang2025no}.
Moreover, diverse methodologies have been applied across different model families and evaluation settings without systematic comparison or consensus on what constitutes reliable unsupervised rewards.
This raises a fundamental question for the field: \textbf{Can intrinsic rewards truly scale LLM training?}
% , or do their early successes mask deeper limitations?}

% It is worth noting that URLVR methods have primarily relied on leveraging a model’s intrinsic signals as training rewards.
% Common approaches include majority voting across multiple rollouts~\citep{zuo2025ttrl} or the adoption of entropy-based metrics~\citep{agarwal2025unreasonable}.
% These forms of intrinsic reward have shown notable performance gains.
% Yet, such seemingly unequivocal successes come with concerns, as several works highlight critical failure modes such as reward hacking and model collapse~\citep{shafayat2025can,agarwal2025unreasonable,zhang2025no}.
% Moreover, diverse methodologies have been applied across different model families and evaluation settings, yet there remains neither a systematic comparison nor a consensus regarding what constitutes reliable unsupervised rewards.
% \textit{\textbf{So behind the flourishing progress of such methods, might there lie certain hidden risks and uncertainties?}}

To answer this, we conduct a comprehensive study of URLVR, spanning taxonomy, theory and extensive experiments.
We begin by classifying URLVR methods into intrinsic and external based on the source of rewards (\Cref{sec:urlvr}).
We then theoretically analyze the underlying mechanism of intrinsic URLVR, revealing that despite diverse design choices, all intrinsic methods converge toward sharpening the model's initial distribution, amplifying existing preferences rather than discovering new knowledge (\Cref{sec:theory}).
This sharpening mechanism has both strengths and limitations, depending on the model prior: whether the model's initial confidence aligns with correctness.
To validate this, we implement intrinsic reward methods and show that intrinsic URLVR consistently follows a rise-then-fall pattern across all methods, differing only in when rather than whether collapse occurs (\Cref{sec:rise}).

Despite these scaling limitations, intrinsic rewards remain valuable on small and domain-specific datasets, avoiding collapse even when all initial preferences are wrong, making it well-suited for test-time training (\Cref{sec:q3}).
Furthermore, we can leverage the rise-then-fall pattern to measure model prior, proposing Model Collapse Step to serve as a practical model prior indicator, predicting RL trainability without expensive training runs (\Cref{sec:prior}).
Finally, since intrinsic rewards face fundamental scalability limits rooted in sharpening mechanism, we investigate the scalability of external URLVR methods, taking the self-verification as an example, which generate verifiable rewards through generation-verification asymmetries rather than internal model states, showing that it exhibits sustained improvement without the collapse patterns inherent to intrinsic methods (\Cref{sec:beyond_intrinsic}).

Overall, our analysis charts a clear path from understanding why current intrinsic URLVR fails to identifying how future methods can scale.
Our findings reveal that intrinsic rewards work within well-defined boundaries determined by the model prior, enabling efficient and safe gains in test-time training while risking reward hacking when confidence misaligns with correctness.
These limitations motivate exploration of external reward methods,
% that leverage external verification mechanisms,
from generation-verification asymmetries in structured domains to self-supervised signals from vast unlabeled corpora, which offer pathways toward more robust and scalable improvement.
% Beyond training itself, we also uncover a practical diagnostic: early intrinsic reward dynamics serve as a fast indicator of model-task priors, offering a lightweight alternative to $pass@k$ for assessing RL trainability.
% \input{Sections/2_Preliminaries}
% \section{Taxonomy of Intrinsic Rewards}
\section{Taxonomy of Unsupervised RLVR}
\label{sec:urlvr}

While RLVR has demonstrated remarkable success in improving the performance of LLMs, it relies on curated high-quality datasets with ground truth labels, which are difficult to obtain as further improvement of models depends on more challenging datasets, imposing a fundamental scalability bottleneck. This limitation has motivated recent efforts to explore alternative reward sources. In particular, an emerging line of research aims to extend the scalability of RLVR through the lens of scalable rewards~\citep{zuo2025ttrl,zhao2025learning,jayalath2025compute}.

In our paper, we investigate this special form of RLVR, termed \textbf{Unsupervised RLVR (URLVR)}:

\begin{tcolorbox}
[title = \textbf{Unsupervised RLVR}, colback=Salmon!20, colframe=Salmon!90!Black]
\textbf{Problem Setting:} We investigate reinforcement learning for verifiable tasks where ground-truth labels are difficult to obtain. In \topic, models must learn from proxy reward signals derived without relying on human efforts.
\end{tcolorbox}

Unsupervised RL is already a concept built in recent works~\citep{agarwal2025unreasonable,zhang2025no}, and we added ``Verifiable Rewards'' to precisely define the domain of tasks we are studying. On one hand, some methods we investigate, like TTRL~\citep{zuo2025ttrl}, rely on the matching between model responses and a verifiable ground truth label (e.g., a final numerical answer in a math problem). On the other hand, we want to distinguish from a wide range of general-domain ``self-rewarding'' methods~\citep{yuan2024self,huang2024self}, which are outside the scope of analysis.

As shown in the middle panel of \Cref{fig:framework}, based on the source of the reward, we categorize URLVR methods into two main types: intrinsic reward methods and external reward methods.

% \begin{tcolorbox}[
%     colback=blue!5!white,
%     colframe=blue!75!black,
%     title={\textbf{URLVR: Learning Without Ground Truth}},
%     fonttitle=\bfseries,
%     coltitle=black,
%     colbacktitle=blue!25!white,
%     enhanced,
%     attach boxed title to top left={xshift=3mm,yshift=-3mm},
%     boxed title style={sharp corners, size=small, colframe=blue!75!black},
%     width=\linewidth,
%     arc=2mm,
%     boxrule=1pt
% ]

% \textbf{Problem Setting:} We investigate reinforcement learning for reasoning tasks where ground-truth rewards are unavailable. In \textbf{Unsupervised RLVR (URLVR)}, models must learn from proxy reward signals derived without any external supervision—no human labels, no automated verifiers, no ground-truth answers.
% \end{tcolorbox}

% \input{Sections/taxonomy}

\subsection{Intrinsic Reward Methods}

A line of works within URLVR leverages proxy intrinsic rewards generated solely by the model itself, eliminating the need for ground-truth labels in reinforcement learning. We distinguish two intrinsic reward paradigms by \emph{how} rewards are constructed from the model:

\begin{itemize}[topsep=0pt, partopsep=0pt, leftmargin=18pt, itemsep=0pt]
    \item \textbf{Certainty-Based Rewards}: Derive a reward from policy’s confidence (e.g., logits) along a trajectory, encouraging low-entropy, high-confidence predictions.
    \item \textbf{Ensemble-Based Rewards}: Derive a reward from agreement across multiple rollouts (e.g., majority voting), assuming that cross-sample consistency correlates with correctness.
\end{itemize}

\noindent\textbf{Notation.} Let $x$ be the input prompt and $y = (y_1, \dots, y_{|y|})$ be the generated output sequence, consisting of reasoning trajectory $c$ and answer $a$ extracted from \texttt{\textbackslash boxed\{\}} in $y$. The model is an LLM policy $\pi_\theta$ parameterized by $\theta$, where $\pi_\theta(\cdot | x, y_{<t})$ represents the probability distribution over the vocabulary for the next token $y_t$.

% \textbf{Certainty-Based Rewards}
% \label{sec:certainty-based}

% \noindent\textbf{Background and Motivation.}
\textbf{Certainty-Based Rewards} measure the model’s confidence in its outputs, which aligns with the classic approaches in Test-Time Adaptation~(TTA)~\citep{wang2020tent}, under the assumption that higher confidence correlates with correctness.
This confidence is derived from the model’s internal state, specifically its output probability distributions (logits).
This concept is rooted in traditional machine learning, where the low-density separation principle~\citep{chapelle2005semi} suggests that decision boundaries should avoid high-density regions.
This principle was implemented through techniques like entropy minimization~\citep{grandvalet2004semi} and high-confidence pseudo-labeling~\citep{lee2013pseudo} in classification tasks.
These ideas have been successfully adapted to domains such as TTA in image classification~\citep{krishnan2020improving, wang2020tent}, and unsupervised skill discovery in robotics~\citep{eysenbach2018diversity, kim2023variational}.

\begin{table}[t]
\centering
\scriptsize
\resizebox{\textwidth}{!}{
\begin{tabular}{@{\hskip 2mm}l@{\hskip 5mm}ll >{\centering\arraybackslash}p{2cm}}
\toprule
\addlinespace[8pt]
\textbf{Method} & \textbf{Estimator} & \textbf{Formula} \\
\addlinespace[4pt]
\midrule

\addlinespace[4pt]
% Intuitor~\citep{zhao2025learning}        
\textbf{RLIF}  
& \textbf{Self-Certainty}       & \( r(x, y) = \frac{1}{|y|} \sum_{t=1}^{|y|} D_{\mathrm{KL}}(U \| \pi_{\theta}(\cdot | x, y_{<t})) \)\\

\addlinespace[8pt]
% EMRL~\citep{agarwal2025unreasonable} 
\textbf{EM-RL}
& \textbf{Trajectory-Level Entropy}  & \( r(x, y) = \frac{1}{|y|} \sum_{t=1}^{|y|} \log \pi_{\theta}(y_t | x, y_{<t}) \) \\

\addlinespace[8pt]
% EMRL, RENT~\citep{prabhudesai2025maximizing}
\textbf{EM-RL, RENT}   
& \textbf{Token-Level Entropy}   & \( r(x, y) = - \frac{1}{|y|} \sum_{t=1}^{|y|} H(\pi_{\theta}(\cdot | x, y_{<t})) \) \\

\addlinespace[8pt]
% RLSC~\citep{li2025confidence}
\textbf{RLSC}          
& \textbf{Probability}           & \( r(x, y) = \prod_{t=1}^{|y|} \pi_{\theta}(y_t | x, y_{<t}) \) \\

\addlinespace[8pt]
% RLSF~\citep{van2025post}
\textbf{RLSF} 
& \textbf{Probability Disparity}  & 
% \( r(x, y) = \frac{1}{M} \sum_{t=1}^{|a|} [ \max_{a_t} \pi_{\theta}(a_t | x, c, a_{<t}) - \max_{a_t \neq \arg \max \pi_{\theta}} \pi_{\theta}(a_t | x, c, a_{<t}) ] \) \\
$r(x, y)=\frac{1}{M}\sum_{t=1}^{|a|} \big[
\max_{a_t}\pi_{\theta}(a_t|x,c,a_{<t})
-\max_{a_t\neq\arg\max\pi_{\theta}}\pi_{\theta}(a_t|x,c,a_{<t})
\big]$ \\

\addlinespace[4pt]
\bottomrule
\end{tabular}}
\caption{Overview of certainty-based rewards, estimators and their formulas. All variants reward high-confidence predictions through different formalizations of model certainty.}
\label{tab:certainty_based_formulas}
\end{table}

% \noindent\textbf{Related Works.}
In \Cref{tab:certainty_based_formulas}, we summarize several estimators that have been proposed to formalize this notion of confidence as intrinsic rewards.
RLIF~\citep{zhao2025learning} proposes Self-Certainty, defined as the average KL divergence between a uniform distribution $U$ over the vocabulary and the model's next-token distribution, which rewards peaked and low-entropy distributions.
Similarly, EM-RL~\citep{agarwal2025unreasonable} and RENT~\citep{prabhudesai2025maximizing} use negative Token-Level Entropy as a direct reward signal to penalize uncertainty at each step.
Aggregating confidence over an entire sequence, works including EM-RL~\citep{agarwal2025unreasonable} use the total Trajectory-Level Entropy (i.e., the sequence log-probability) or its exponentiated form, the raw Probability as used in RLSC~\citep{li2025confidence}.
Besides, RLSF~\citep{van2025post} refines this by considering the Probability Disparity between top tokens, capturing distribution sharpness beyond the maximum.

% \noindent\textbf{Unified Perspective.}
At their core, all certainty-based rewards are different mathematical formalizations for rewarding and reinforcing high-confidence predictions.
Self-Certainty and Token-Level Entropy are step-wise measures that assess the model's confidence at each generation step. Minimizing entropy is functionally similar to maximizing the KL divergence from a uniform distribution.
Trajectory-Level Entropy and Probability are sequence-level aggregations of this confidence, with one simply being the logarithm of the other.
The Probability Disparity measures the gap between top-2 probability tokens, which is another angle for estimating confidence.

\textbf{Ensemble-Based Rewards} leverage the ``wisdom of the crowd'' by generating an ensemble of diverse candidate solutions for the same prompt instead of relying on the certainty of a single output.
The core assumption is that consistency among model-generated answers correlates positively with their correctness, making consensus a robust proxy for correctness.

% \noindent\textbf{Related Works.}
As summarized in \Cref{tab:ensemble_based_formulas}, several methods have been proposed to formalize this notion of consistency.
An early work in this area is TTRL~\citep{zuo2025ttrl}, which uses majority voting on the final answers from multiple rollouts to create a pseudo-label.
This pseudo-label provides a verifiable reward signal for RL.
Building on this foundation, SRT~\citep{shafayat2025can} analyzes limitations including reward hacking, ETTRL~\citep{liu2025ettrl} improves efficiency through entropy-based tree search, Co-Reward~\citep{zhang2025co} enhances robustness via question paraphrasing, and RLCCF~\citep{yuan2025wisdom} incorporates multi-model collectives. More nuanced approaches include EMPO~\citep{zhang2025right} using semantic clustering for soft majority voting and CoVo~\citep{zhang2025consistent} deriving rewards from intermediate reasoning consistency.
% Building on this, SRT~\citep{shafayat2025can} extends the approach to training data, analyzing its limitations and identifying the risk of reward hacking when the model's initial capabilities are low.

% Subsequent works have refined this concept in several ways.
% ETTRL~\citep{liu2025ettrl} improves the efficiency of rollouts by using an entropy-based tree search, reducing the computational overhead of generating multiple solutions.
% Co-Reward~\citep{zhang2025co} enhances robustness by enforcing consistency not only across multiple rollouts of the same question but also across semantically similar, paraphrased questions.
% RLCCF~\citep{yuan2025wisdom} extends this approach to a multi-model ensemble, where a diverse set of LLMs co-evolve, and rewards are derived from weighted collective voting, thus mitigating the risk of a single model's overconfidence.

% Some works provide further refinements, focusing on more nuanced forms of consistency.
% As an early exploration, EMPO~\citep{zhang2025right} clusters rollouts based on semantic similarity in a latent space, assigning rewards proportionally to the cluster likelihood rather than using binary pseudo-labels, which helps maintain output diversity through a ``soft'' majority vote.
% CoVo~\citep{zhang2025consistent} approaches consistency differently, deriving rewards from the consistency of intermediate reasoning states across different trajectories.
% It also introduces a volatility penalty and curiosity bonus to encourage exploration.

\begin{table}[t]
\centering
\scriptsize
\renewcommand{\arraystretch}{1.2}
\begin{tabular}{@{\hskip 2mm}l@{\hskip 5mm}ll >{\centering\arraybackslash}p{2cm}}
\toprule
\addlinespace[8pt]
\textbf{Method} & \textbf{Estimator} & \textbf{Formula} \\
\addlinespace[4pt]

\midrule

\addlinespace[6pt]

\textbf{\makecell[l]{TTRL, SRT, ETTRL \\ SeRL, SQLM, R-Zero}} & \textbf{Majority Voting} & 
  $ r(x, y) = \mathbb{1}\left[y = \arg\max_{y'} \sum_{i=1}^{N} \mathbb{1}[y_i = y']\right], \quad \{y_i\}_{i=1}^{N} \sim \pi_\theta(\cdot|x) $  \\

% \midrule
\addlinespace[8pt]
\textbf{Co-Reward} & \textbf{\makecell[l]{Majority Voting\\across Rephrased Question}}   &
  $\begin{aligned} 
  r(x, y) &= \mathbb{1}[y = \arg\max_{y^*}\textstyle\sum_{i=1}^{N} \mathbb{1}[y_i = y^*]], \quad \{y_i\}_{i=1}^{N} \sim \pi_\theta(\cdot|x) \\
  &+ \mathbb{1}[y = \arg\max_{y^*} \textstyle\sum_{j=1}^{N} \mathbb{1}[y_j' = y^*]], \quad \{y_j'\}_{j=1}^{N} \sim \pi_\theta(\cdot|\text{rephrase}(x))
  \end{aligned}$ \\

% \midrule
\addlinespace[4pt]
\textbf{RLCCF} & \textbf{\makecell[l]{Self-consistency \\ Weighted Voting}} &
  $ \begin{aligned}
r(x, y) = \mathbb{1}\bigg[y = \arg\max_{a} \textstyle\sum_{n=1}^{N} &\left(\max_{a'}\textstyle\sum_{k=1}^{K} \mathbb{1}[o_{n,k} = a']\right) \cdot \textstyle\sum_{k=1}^{K} \mathbb{1}[a = o_{n,k}]\bigg], \\
&\{o_{n,k}\}_{k=1}^{K} \sim \pi_{\theta_n}(\cdot|x), \quad n = 1, \ldots, N
\end{aligned} $ \\

% \midrule
\addlinespace[8pt]
\textbf{EMPO} & \textbf{Semantic Similarity} &
  $ r(x, y) = \frac{|\mathcal{C}(y)|}{G}, \quad \mathcal{C}(y) \in \texttt{SemanticCluster}(\{o_i\}_{i=1}^{G}), \quad \{o_i\}_{i=1}^{G} \sim \pi_\theta(\cdot|x) $ \\

% \midrule
\addlinespace[8pt]
\textbf{CoVo} & \textbf{\makecell[l]{Trajectory Consistency\\and Volatility}} &
  $ \begin{aligned}
r(x, y) = &\frac{1}{G}\left\|\textstyle\sum_{i=1}^{G} \texttt{Con}(y_i) \cdot [\cos(\texttt{Vol}(y_i)), \sin(\texttt{Vol}(y_i))]\right\| + r_{\text{cur}}, \\
&\{y_i\}_{i=1}^{N} \sim \pi_\theta(\cdot|x), \quad G = |\{i: \text{ans}(y_i) = \text{ans}(y)\}|
\end{aligned} $ \\

\addlinespace[4pt]
\bottomrule
\end{tabular}
\caption{Overview of ensemble-based rewards, estimators and their formulas. All variants operationalize the assumption that consistency across independent samples correlates with correctness.}
\label{tab:ensemble_based_formulas}
\end{table}

Another line of works involves the use of asymmetric proposer-solver architectures~\citep{kirchner2024prover}, where one model generates tasks and another solves them.
These systems establish a self-sustaining ecosystem where models co-evolve, and the solver is trained through intrinsic rewards.
The key innovation in these systems lies in rewarding the proposer.
SeRL~\citep{fang2025serl} does not directly reward the proposer but uses heuristic online filtering for proposed instructions.
R-Zero~\citep{huang2025r} rewards the proposer for generating tasks that push the solver’s uncertainty towards a 50\% confidence level, thus maximizing reward when the solver is unsure. Combined with repetition penalties, this encourages the proposer to create diverse and appropriately challenging problems.
SQLM~\citep{chen2025self} rewards the proposer for generating problems that are neither too easy nor too hard, further enhancing the diversity of tasks, while CPMobius~\citep{li2026cpmobius} rewards the coach based on changes in the player's performance.

% \noindent\textbf{Trade-offs and Limitations.}
% \yuxin{not sure}
% Ensemble-based rewards improve robustness by aggregating multiple responses, but they also come with the trade-off of sharpening the distribution towards the majority-voted answer. This can reduce the model's ability to explore diverse solutions, as the optimization focuses on the most consistent responses. 

% \noindent\textbf{Unified Perspective.}
All ensemble methods assume that consistency across independent samples correlates with correctness. However, relying on consensus introduces computational overhead, especially when generating a large number of rollouts or utilizing multi-model ensembles. Also, ensemble-based rewards are constrained by the availability of suitable extractors and comparison rules. For example, mathematical tasks are well-suited to such methods due to clear answer formats, but many general tasks lack appropriate criteria for voting, which limits their applicability.

\subsection{External Reward Methods}
\label{sec:external}

Another line of works explores alternative URLVR methods that generate verifiable rewards through external mechanisms rather than internal model states. We identify two paradigms that leverage data structure and computational asymmetries as follows.

\textbf{Leveraging Unlabeled Data for Reward Generation.\ \ } Large-scale unlabeled corpora provide natural verification signals by converting language modeling into reward-based tasks.
Different from intrinsic methods which derive reward from the model's own outputs, this paradigm derives rewards in the unlabeled corpus itself. The text provides latent ground truth that the model can be evaluated against, independent of the model's internal state.
For example, RPT~\citep{dong2025reinforcement} exemplifies this by rewarding correct next-token predictions on unlabeled text, transforming trillions of tokens into scalable reward signals.
Extensions include TPT~\citep{wang2025thinking}, which predicts tokens through step-by-step reasoning, and RLPT~\citep{li2025reinforcement}, which operates at the segment level rather than token level.
Similarly, RLP~\citep{hatamizadeh2025rlp} trains models to generate an internal reasoning chain before predicting the next token on unlabeled text, and rewards the model based on the information gain that its chain-of-thought provides for the prediction.

A natural extension of this paradigm is to exploit the structure within unlabeled text beyond next-token prediction.
DuPO~\citep{she2025dupo} pairs a primary task with a dual reconstruction objective. The model must recover the original unlabeled input from its own output, and the quality of this reconstruction serves as a self-supervised reward to optimize the primal task. Meta-learning approaches like SEAL~\citep{zweiger2025self} take this further, where models generate their own QA pairs from unlabeled contexts and receive rewards based on downstream self-supervised performance, creating autonomous improvement loops. Beyond raw text corpora, web-scale question-answering data with programmatically verifiable structure provides another form of unsupervised reward. Nemotron-CrossThink~\citep{akter2025nemotron} curates multi-domain QA from CommonCrawl and open web sources spanning law, physics and social science, and converts them into multiple-choice format so that answer correctness can be checked without human annotation. This converts raw internet text, a practically infinite resource, into a source of verifiable rewards across diverse reasoning domains, extending the unlabeled-data paradigm beyond the language-modeling objective.
These methods scale naturally with data availability, unbounded by human annotation capacity.

\textbf{Exploiting Generation-Verification Asymmetries.\ \ }
A fundamental but underexploited property of many reasoning tasks is that generation and verification are not equally hard~\citep{burns2023weak,song2024mind}.
Producing a correct solution may require deep search over an exponentially large space, yet checking whether a candidate solution is correct can often be done in constant time with a deterministic procedure, like executing a program, querying a proof assistant, evaluating a puzzle rule or running a numerical simulator. This asymmetry is precisely what URLVR can exploit. Rather than relying on human labels or on the model's own internal confidence, the verification procedure itself acts as an external, objective and infinitely scalable reward. The model is free to generate arbitrarily many candidate solutions, and each one receives a reliable reward at an affordable cost.

This principle manifests across a wide spectrum of domains with different verification mechanisms.
In mathematical computation, LADDER~\citep{simonds2025ladder} demonstrates the idea for indefinite integration. While finding an antiderivative in closed form is genuinely difficult, verifying a proposed answer requires only numerical evaluation at sampled points, which is automatic and cheap.
Similarly, RLSR~\citep{simonds2025rlsrreinforcementlearningself} applies the same logic to Countdown arithmetic puzzles, where checking whether a sequence of operations reaches a target number is trivial even though constructing such a sequence is not.

In software engineering, the gap is even sharper. Writing a program that passes all test cases demands sophisticated multi-step reasoning, yet executing the program against those tests is deterministic and instantaneous.
Absolute Zero~\citep{zhao2025absolute} exploits this for code generation, automatically creating correct reference solutions through execution and using test-based verification as the sole reward signal, requiring neither human-labeled problems nor human-written solutions.
Besides, in formal theorem proving, discovering a proof of a non-trivial theorem may take a human mathematician weeks or years, but verifying a proposed proof in Lean takes seconds and ensures no ambiguity.
DeepSeekMath-V2~\citep{shao2025deepseekmath} leverages self-verification outcomes as RL rewards, while
AlphaProof~\citep{hubert2025olympiad} pushes this to its limit by training on millions of auto-formalized mathematical problems.

\textbf{External Rewards Enable Scalable URLVR.\ \ }
The distinction between intrinsic and external reward methods is not merely taxonomic but reflects a fundamental difference in scalability.
Intrinsic reward methods derive their signal entirely from the model's own probability distributions, and are therefore limited by what the model already knows. 

External reward methods escape this ceiling through two complementary mechanisms.
Unlabeled-data methods derive rewards that grow with corpus size rather than with model capability.
Because the quantity and diversity of available text is effectively unbounded, the reward landscape expands with data scale, providing fresh learning signal even after the model has exhausted the knowledge in its initialization.
Generation-verification asymmetry methods ground rewards in external computation from compilers, proof assistants, game engines or scientific simulators, which are entirely independent of the model's internal state.
A compiler does not become a weaker verifier as the model improves and a Lean proof checker does not hallucinate.
This independence means reward quality does not degrade with scale. As the model generates harder and more sophisticated outputs, the external verifier checks them with equal reliability.

We therefore position external reward methods as the more promising direction for long-run URLVR scaling, and highlight the development of diverse verifiable environments and the extension of verification asymmetries to new scientific domains as the key open challenges ahead.

\section{The Sharpening Mechanism of Intrinsic Rewards}
\label{sec:theory}

We start from intrinsic reward methods in URLVR, investigating their potential in scaling LLM training. Existing studies \citep{zhang2025right,zuo2025ttrl,agarwal2025unreasonable,shafayat2025can} have empirically assessed the strengths and weaknesses of these intrinsic methods, but the underlying mechanisms of how and why they lead to the observed effects, have been underexplored.

In this section, we theoretically investigate the underlying mechanism of intrinsic URLVR, revealing that despite diverse design choices in \Cref{tab:certainty_based_formulas,tab:ensemble_based_formulas}, the model optimized through intrinsic rewards converges towards sharpening its initial distribution. Interestingly, this sharpening mechanism yields both notable strengths and limitations, depending on the model's initial knowledge of the relationship between confidence and correctness. When initial confidence aligns with correctness, the sharpening acts as a beneficial amplifier. When misaligned, the same mechanism systematically reinforces errors.

% Despite diverse design choices in \Cref{tab:certainty_based_formulas,tab:ensemble_based_formulas}, they systematically encourage the model to produce more decisive outputs.

% In this section, we take TTRL as a representative intrinsic method, analyzing the optimal policy induced by its majority voting reward. Our investigation reveals that it converges toward sharpening the model’s initial distribution, yielding both notable strengths and limitations. This trade-off depends critically on the model's initial knowledge of the relationship between confidence and correctness. When initial confidence aligns with correctness, the sharpening acts as a beneficial amplifier. When misaligned, the same mechanism systematically reinforces errors.

% Intrinsic rewards exhibit diverse design choices and mathematical formulations. However, a deeper analysis reveals that these methods share a fundamental commonality: they systematically encourage the model to produce sharper, more decisive probability distributions. This convergence toward higher confidence outputs represents a unified underlying mechanism with important implications for both strengths and limitations.

\subsection{Dynamics of One-Step Update}

Training with intrinsic rewards involves multiple gradient steps toward convergence. In this section, we take TTRL~\citep{zuo2025ttrl} and its majority voting reward as a representative intrinsic method to demonstrate the dynamics of a one-step update. 

Consider the standard KL-regularized RL objective:

\begin{equation}
\max_{\pi_\theta} \mathbb{E}_{y\sim\pi_\theta(\cdot|x)}\left[r(x,y)\right] - \beta D_{\text{KL}}\left[\pi_\theta(\cdot|x) \| \pi_{\text{ref}}(\cdot|x)\right],
\label{eq:rl_obj}
\end{equation}

where $\pi_{\text{ref}}$ is the reference policy and $\beta$ controls the strength of regularization. Refer to~\citep{rafailov2023direct}, the optimal policy for this objective has the well-known closed form:

\begin{equation}
\pi_\theta^*(y|x) = \frac{1}{Z(x)}\pi_{\text{ref}}(y|x)\exp\left(\frac{1}{\beta}r(x,y)\right),
\label{eq:optimal_policy_general}
\end{equation}

where $Z(x) = \sum_y \pi_{\text{ref}}(y|x)\exp\left(\frac{1}{\beta}r(x,y)\right)$ is the partition function. The majority voting reward used in TTRL at iteration $k$ is defined as

\begin{equation}
r_k(x,y) = \mathbf{1}[\text{ans}(y) = \text{maj}_k(Y_k)],
\end{equation}

where $Y_k = \{y^{(1)}, \ldots, y^{(N)}\}$ denotes $N$ rollouts sampled from $\pi_\theta^{(k)}$, and $\text{maj}_k(Y_k) = \arg\max_{a} |\{i \in [N] : \text{ans}(y^{(i)}) = a\}|$ is the most frequent answer among the rollouts. Then applying \Cref{eq:optimal_policy_general} with the majority voting reward $r_k$, if we held $r_k$ constant and performed infinite updates starting from reference policy $\pi_\theta^{(k)}$, we would converge to the optimal policy:

\begin{equation}
\pi_\theta^{*,(k+1)}(y|x) = \frac{\pi_\theta^{(k)}(y|x) \cdot \exp\left(\frac{r_k(x,y)}{\beta}\right)}{Z_k(x)}.
\end{equation}

Since $r_k$ is binary, the exponential term takes only two values: $e^{1/\beta}$ for the majority-voted $\text{maj}_k(Y_k)$ and $e^0 = 1$ for all others. This yields the explicit form:

\begin{equation}
\pi_\theta^{*,(k+1)}(y|x) = \begin{cases}
\frac{\pi_\theta^{(k)}(y|x) \cdot e^{1/\beta}}{Z_k(x)}, & \text{if } \text{ans}(y) = \text{maj}_k(Y_k), \\[3mm]
\frac{\pi_\theta^{(k)}(y|x)}{Z_k(x)}, & \text{otherwise},
\end{cases}
\label{eq:MV_optimal_detailed}
\end{equation}

where the partition function ensures proper normalization:
\begin{equation}
Z_k(x) = \sum_{y: \text{ans}(y) = \text{maj}_k(Y_k)} \pi_\theta^{(k)}(y|x) \cdot e^{1/\beta} + \sum_{y: \text{ans}(y) \neq \text{maj}_k(Y_k)} \pi_\theta^{(k)}(y|x).
\end{equation}

For cleaner notation, let $p_{\text{maj}}^{(k)} = \sum_{y: \text{ans}(y) = \text{maj}_k(Y_k)} \pi_\theta^{(k)}(y|x)$ denote the current policy's total probability mass on trajectories leading to the majority answer. Then the partition function simplifies to:
\begin{equation}
Z_k(x) = p_{\text{maj}}^{(k)} \cdot e^{1/\beta} + (1 - p_{\text{maj}}^{(k)}).
\end{equation}

% Having established the unified framework, we now analyze what optimal policies these rewards induce. This analysis reveals both the mechanisms behind their effectiveness and the boundaries of their applicability.

% To illustrate these theoretical insights concretely, we analyze the Majority Voting estimator widely used across intrinsic rewards, as shown in \Cref{tab:unified_rewards}, then generalize to other methods using the unified framework (Appendix~\ref{app:unified_convergence}). It instantiates our unified framework with answer-level granularity $\mathcal{I} = \{\mathcal{A}\}$, anchor distribution $q = \delta_{\text{maj}(Y)}$ (one-hot at the majority answer), sign factor $\sigma = -1$, and transformation $\psi(z) = \exp(z)$.

The structure of \Cref{eq:MV_optimal_detailed} reveals the method's update mechanism. At the optimal policy $\pi_\theta^{*,(k+1)}$, each trajectory $y$ where $\text{ans}(y) = \text{maj}_k(Y_k)$ has its probability amplified by factor $e^{1/\beta}$. The probability mass on the majority trajectories would be:

\begin{equation}
p_{\text{maj}}^{*,(k+1)} = \frac{p_{\text{maj}}^{(k)} \cdot e^{1/\beta}}{p_{\text{maj}}^{(k)} \cdot e^{1/\beta} + (1 - p_{\text{maj}}^{(k)})}.
\label{eq:update_on_p}
\end{equation}

% \begin{itemize}[topsep=0pt, itemsep=2pt, leftmargin=18pt]
%     \item \textbf{Majority trajectories}: Each trajectory $y$ where $\text{ans}(y) = \text{maj}_k(Y_k)$ has its probability amplified by factor $e^{1/\beta}$. The collective probability mass for all trajectories leading to the majority answer becomes:
%     \begin{equation}
%     p_{\text{maj}}^{*,(k+1)} = \frac{p_{\text{maj}}^{(k)} \cdot e^{1/\beta}}{p_{\text{maj}}^{(k)} \cdot e^{1/\beta} + (1 - p_{\text{maj}}^{(k)})}.
%     \label{eq:update_on_p}
%     \end{equation}
    
%     \item \textbf{Non-majority trajectories}: Each trajectory leading to other answers maintains its relative proportion within the non-majority group but experiences reduced absolute probability mass. The total probability for all non-majority trajectories becomes:
%     \begin{equation}
%     p_{\text{non-maj}}^{*,(k+1)} = 1 - p_{\text{maj}}^{*,(k+1)} = \frac{1 - p_{\text{maj}}^{(k)}}{p_{\text{maj}}^{(k)} \cdot e^{1/\beta} + (1 - p_{\text{maj}}^{(k)})}.
%     \end{equation}
% \end{itemize}

\noindent\textbf{Actual Training Dynamics.} In practice, our training performs only one gradient update per iteration, not reaching the optimum $\pi_\theta^{*,(k)}$ but moving partway toward it. The actual probability mass after one update $p_{\text{maj}}^{(k+1)}$ satisfies: 
\begin{equation}
p_{\text{maj}}^{*,(k+1)} \geq p_{\text{maj}}^{(k+1)} \geq p_{\text{maj}}^{(k)}.
\end{equation}
This ordering holds because: (1) policy gradient methods with positive rewards on majority trajectories tend to increase their probability mass (lower bound), and (2) the optimal policy achieves maximum expected reward, so one-step updates cannot exceed it (upper bound). See Appendix~\ref{app:order} for detailed theoretical justification and empirical validation.

\subsection{Convergence Towards Sharpening Initial Distribution}

The one-step update creates a ``rich-get-richer'' dynamic that trajectories leading to the majority have their probabilities consistently increased, while others are proportionally diminished. Iterating this process, the policy converges geometrically toward a deterministic policy concentrated on the initial majority answer:

\begin{theorem}
\label{theorem:mv_convergence}
\textbf{Geometric Convergence to Deterministic Policy.} Consider the training process where at each iteration $k$, we: (1) sample $N$ rollouts $Y_k$ from $\pi_\theta^{(k)}$, (2) compute majority $\text{maj}_k(Y_k)$, (3) perform one-step update with reward $r_k(x,y) = \mathbf{1}[\text{ans}(y) = \text{maj}_k(Y_k)]$ to obtain $\pi_\theta^{(k+1)}$. Let $p_{\text{maj}}^{(k)} = \sum_{y: \text{ans}(y) = \text{maj}_k(Y_k)} \pi_\theta^{(k)}(y|x)$ denote the probability mass on majority trajectories at iteration $k$.

Suppose the following assumptions hold, validated empirically in Appendix~\ref{app:order}:
\begin{enumerate}[label=\textup{(A\arabic*)}]
    \item \textit{Majority stability:} $\mathrm{maj}_k(Y_k) = \mathrm{maj}_0(Y_0)$ for all $k$ (holds for sufficiently large $N$);
    \item \textit{Effective learning:} $p_{\mathrm{maj}}^{(k+1)} > p_{\mathrm{maj}}^{(k)}$ for all $k$ (a standard assumption in policy gradient methods).
\end{enumerate}

Then $p_{\mathrm{maj}}^{(k)}$ converges geometrically to $1$ with rate $\rho = e^{-1/\beta}$, and the policy converges to
\begin{equation}
    \lim_{k \to \infty} \pi_\theta^{(k)}(y \mid x) = 
    \begin{cases}
        \dfrac{\pi_{\mathrm{ref}}(y \mid x)}{\sum_{y':\, \mathrm{ans}(y') = \mathrm{maj}_0(Y_0)} \pi_{\mathrm{ref}}(y' \mid x)}, & \text{if } \mathrm{ans}(y) = \mathrm{maj}_0(Y_0), \\[6pt]
        0, & \text{otherwise}.
    \end{cases}
\end{equation}
\end{theorem} 

Complete proof is provided in Appendix~\ref{app:convergence_analysis}. For other URLVR intrinsic rewards, we leverage the sharpening intuition and propose a unified reward framework in \Cref{app:unified_framework}, showing that all intrinsic rewards can be understood through a single lens: manipulating cross-entropy between carefully chosen distributions. Built upon this unified perspective, we generalize this convergence analysis in Appendix~\ref{app:unified_convergence}, and provide optimal policies induced respectively in \Cref{app:optimal_policy_other_methods}.

This convergence behavior has profound implications depending on the model prior. When the model's confidence (reflected in $\text{maj}_0$) aligns with correctness, convergence reinforces good solutions. Conversely, if confidence is poorly aligned, the same mechanism amplifies errors, leading to model collapse. Next, we will evidence this from an empirical perspective.

\section{When Does Intrinsic URLVR Work?}
\label{sec:rise}

\begin{tcolorbox}[takeawaysbox]
Intrinsic URLVR universally follows a rise-then-fall pattern across all methods. Early gains reflect confidence-correctness alignment in the model's prior, while eventual collapse is inevitable when this alignment breaks down.
\end{tcolorbox}

As we derived from \Cref{sec:theory}, intrinsic rewards sharpen distributions by amplifying model's initial preferences, but a natural question is when do intrinsic rewards work? First, we trace the lifecycle of intrinsic URLVR methods (\Cref{sec:lifecycle}), finding early gains followed by collapse across all intrinsic methods.
Then, we conduct fine-grained per-problem analysis (\Cref{sec:fine_grained}), revealing that training amplifies existing preferences rather than correcting them within the same problem. Surprisingly, even when training amplifies errors on problem A, it can still correct out-of-distribution problem B, showing that confidence-correctness alignment varies across problems. 
% Third, we interpret these patterns (\Cref{sec:trade}) that intrinsic rewards trade uncertainty for performance by exploiting prior knowledge.
These empirical results further strengthen that success depends on whether the initial confidence aligns with correctness for each specific problem.

% The prospect of LLM scaling through Unsupervised RLVR hinges on whether models can reliably improve themselves without ground truth labels.
% Several works~\citep{agarwal2025unreasonable,zhang2025no} have mentioned the promising results of intrinsic rewards while suggesting their limitations.
% Our theoretical analysis in \Cref{sec:theory} suggests that intrinsic rewards enable such self-improvement by exploiting existing knowledge, yet their convergence to deterministic policies raises concerns about when this process succeeds versus fails. To clarify these boundaries, we empirically examine the practical limits and opportunities of intrinsic reward-driven scaling.

% In this section, we investigate three questions through progressive analysis. First (\Cref{sec:lifecycle}), we ask whether intrinsic rewards reliably improve performance, revealing an early-success-then-collapse pattern that persists across different hyperparameters. Second (\Cref{sec:fine_grained}), we examine the sharpening mechanism at the problem level, showing that training amplifies initial model preferences, whether correct or incorrect. Third (\Cref{sec:trade}), we connect these findings to broader debates about what RL learns, proposing that intrinsic rewards trade uncertainty for performance by exploiting existing knowledge. Together, these results establish that intrinsic rewards operate within well-defined boundaries determined by the model's initial confidence-correctness alignment.

\subsection{The Rise and Fall of Intrinsic URLVR}
\label{sec:lifecycle}

In this section, we conduct experiments across different hyperparameters and methods for intrinsic rewards, revealing that intrinsic URLVR follows a consistent rise-then-fall pattern.

\subsubsection{Early Success, Later Collapse}

\noindent\textbf{Setup.} We use majority voting reward from TTRL~\citep{zuo2025ttrl} to compare intrinsic reward training against standard RLVR using ground truth labels. We train Qwen3-1.7B-Base on DAPO-17k~\citep{yu2025dapo} using default hyperparameters from \Cref{tab:default_hyper}, and evaluate on AIME 2024~\citep{li2024numinamath}, AIME 2025~\citep{balunovic2025matharena} and AMC 2023~\citep{li2024numinamath}. Following standard practice, we generate 32 solutions per problem at temperature 0.6 with top-p 0.95, reporting average accuracy (avg@32).
Beyond validation performance, we track three training dynamics: \textit{Majority Voting Reward} (the intrinsic reward signal), \textit{Reward Accuracy} (whether pseudo-reward matches ground truth reward), and \textit{Actor Entropy}.
See \Cref{app:experimental_setup} for complete details.
Unless stated otherwise, we utilize the default setup above for later experiments.

\begin{figure*}[!t]
    \centering
    \includegraphics[width=1\linewidth]{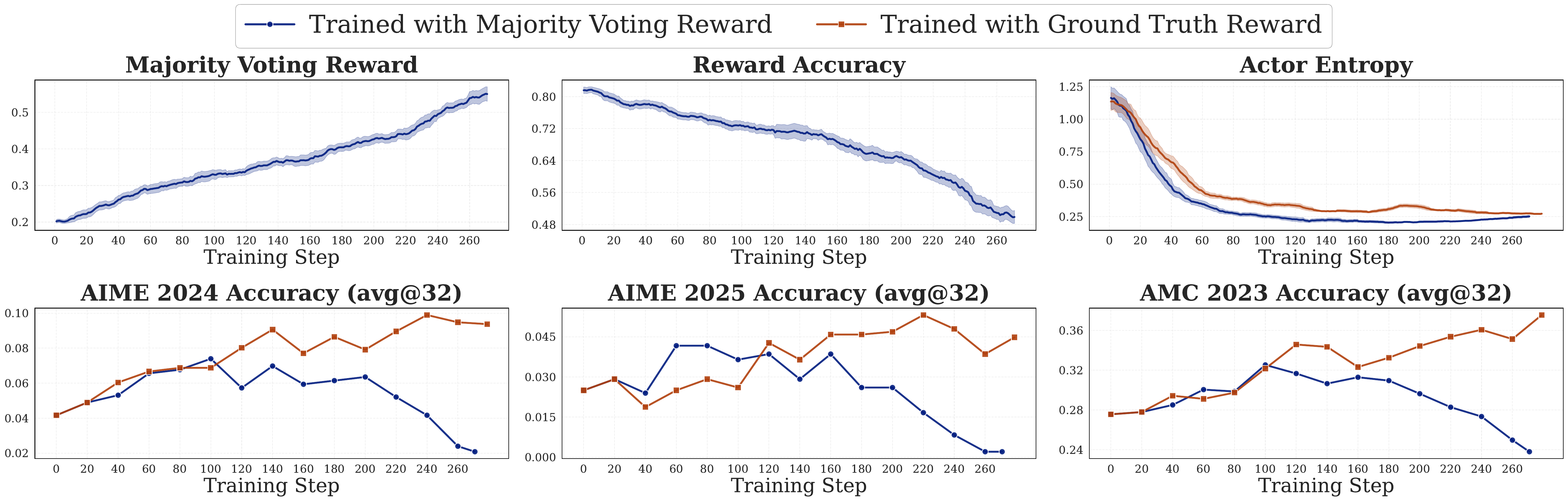}
    \caption{Training dynamics comparing majority-voting training and ground-truth training. Intrinsic rewards initially match supervised performance but eventually collapse: the proxy reward rises while validation accuracy falls, revealing divergence between optimizing confidence and optimizing correctness.}
    \label{fig:4p1p1}
\end{figure*}

\noindent\textbf{Results.} Intrinsic rewards initially match ground-truth training but eventually collapse, and this pattern exists across all hyperparameter settings we tested. \Cref{fig:4p1p1} shows that majority-voting training matches or even exceeds ground-truth training on three benchmarks during the early phase.
But continued training reveals that while the proxy \textit{Majority Voting Reward} keeps rising, both \textit{Reward Accuracy} and validation performance decline, exhibiting obvious reward hacking.
This divergence between proxy reward and actual correctness aligns with recent findings on intrinsic reward instability~\citep{zhang2025no}.
We also observe that training with majority voting reward drives down \textit{Actor Entropy} faster than ground-truth training, establishing a connection between reduced uncertainty and improved performance.

To prevent the issue from being caused by hyperparameters, we thoroughly tune four key hyperparameters across five intrinsic reward methods, including training temperature, mini-batch size, KL regularization and rollout number (\Cref{app:hyperparameter}).
We find that some choices matter significantly (e.g. mini-batch size and rollout number), but nearly all settings eventually degrade, differing only in when rather than whether collapse occurs.

To push this further, we combine all insights that stabilize training from our tuning and run extended training. Even with the most stable settings, collapse still occurs around 1,000 steps (roughly 4 epochs). This suggests the rise-then-fall pattern may reflect a fundamental limitation rather than an engineering problem.

\subsubsection{Different Methods, Different Failures}
\label{sec:diff_methods}

We've shown that intrinsic rewards eventually collapse regardless of hyperparameter tuning. But do all methods fail the same way? We find that not all failures are equal. In this section, we reveal three distinct failure patterns, each exposing different weaknesses in how rewards reinforce confidence.
% Understanding these distinct failure patterns helps with method selection.

\noindent\textbf{Setup.} We compare five intrinsic rewards on Qwen3-1.7B-Base trained on DAPO-17k, each with separately tuned hyperparameters (\Cref{app:hyperparameter}). For ensemble-based rewards we use the Majority Voting estimator, and for certainty-based we test Self-Certainty, Trajectory-Level Entropy, Token-Level Entropy, and Probability. We use the formulas in \Cref{tab:certainty_based_formulas,tab:ensemble_based_formulas} and evaluate on three benchmarks, tracking \textit{Label Accuracy} (whether pseudo-label matches ground truth label), \textit{Actor Entropy} and \textit{Mean Response Length}. The calculation is detailed in \Cref{app:experimental_setup}.

\begin{figure*}[!t]
    \centering
    \includegraphics[width=1\linewidth]{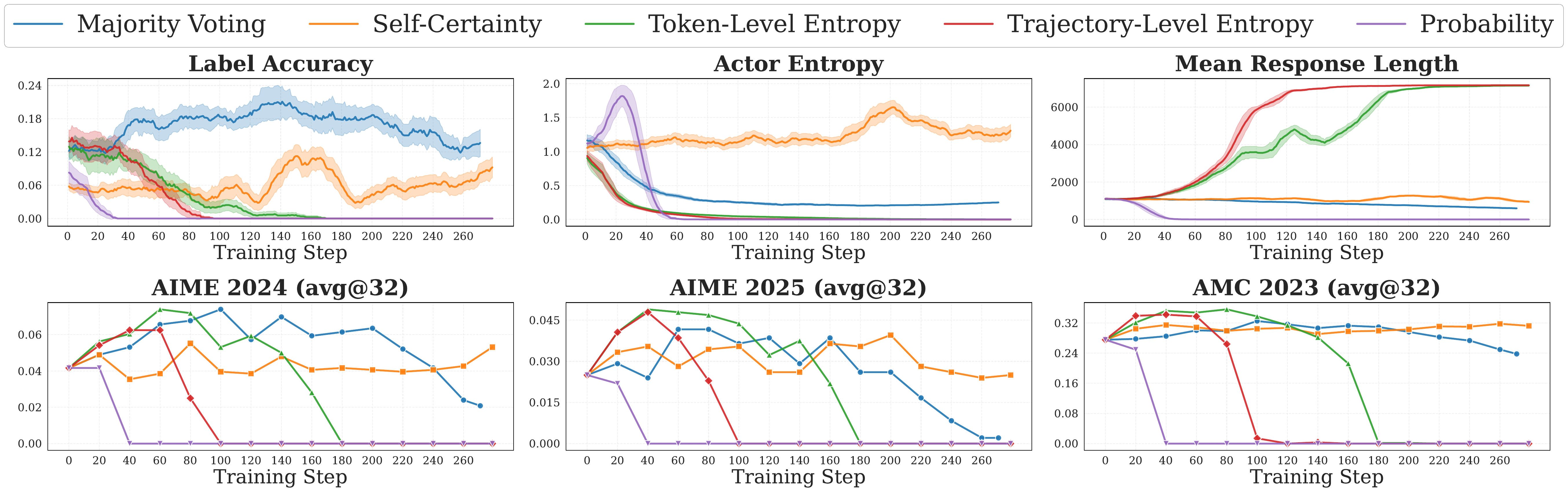}
    \caption{Five intrinsic reward methods exhibit distinct failure patterns. Self-Certainty and Majority Voting degrade gradually while maintaining label accuracy. Probability collapses toward brevity and entropy-based methods drive entropy down through repetition rather than correctness.}
    \label{fig:urlvr_compare}
    % \vspace{-4mm}
\end{figure*}

\noindent\textbf{Results.} \Cref{fig:urlvr_compare} shows that different methods lead to different failure modes:

\begin{itemize}[topsep=0pt, partopsep=0pt, leftmargin=12pt, itemsep=0pt]

\item \textbf{Gradual degradation:} Self-Certainty and Majority Voting degrade most slowly, maintaining higher validation performance and higher \textit{Label Accuracy} without collapsing within one epoch.
Self-Certainty sharpens against a uniform distribution at each position (\Cref{tab:certainty_based_formulas}), making it less aggressive than direct probability maximization. Majority Voting operates at the answer level rather than the token level (\Cref{tab:ensemble_based_formulas}), avoiding token-level artifacts.

\item \textbf{Length collapse:} Probability rewards brevity because it multiplies token probabilities, which naturally favors shorter sequences.
The model becomes confident (lower \textit{Actor Entropy}) but produces overly brief answers (shorter \textit{Mean Response Length}), creating a distinct reward hacking pattern focused on length rather than content quality.
Length normalization like using geometric mean or average log-probability would likely mitigate this bias.

\item \textbf{Repetition collapse:} Both entropy-based methods (Token-Level Entropy, Trajectory-Level Entropy) average entropy across tokens, which is minimized not only by confident predictions but also by repeating high-probability tokens.
Unlike probability multiplication (which rewards brevity), averaging across sequences encourages the model to pad sequences with repetitive text.

\end{itemize}

These patterns reveal that while all methods sharpen distributions in theory, their practical failures diverge substantially.

\subsection{Fine-Grained Per-Problem Analysis}
\label{sec:fine_grained}

Next, we examine individual training samples in detail to reveal the mechanism of URLVR methods.
Our analysis reveals that training amplifies initial preferences rather than correcting errors within the same problem. Surprisingly, this amplification can still generalize to unseen problems and improve performance.

\subsubsection{In-Distribution Per-Problem Sharpening}
\label{sec:id}

Having observed the rise-then-fall pattern at the dataset level, we next investigate whether the early success of intrinsic URLVR stems from genuinely correcting wrong answers or merely from amplifying pre-existing preferences.
To better understand the underlying sharpening mechanism, we conduct experiments on individual problems and track how the model's behavior evolves.

\noindent\textbf{Setup.} We train Qwen3-1.7B-Base on 25 randomly sampled individual problems from MATH500 using REINFORCE with Trajectory-Level Entropy as the intrinsic reward. Each problem trains for 100 epochs with batch size 1 and 8 rollouts for baseline estimation.

\begin{wrapfigure}{r}{0.6\textwidth}
% \vspace{-15pt}
\centering
\includegraphics[width=\linewidth]{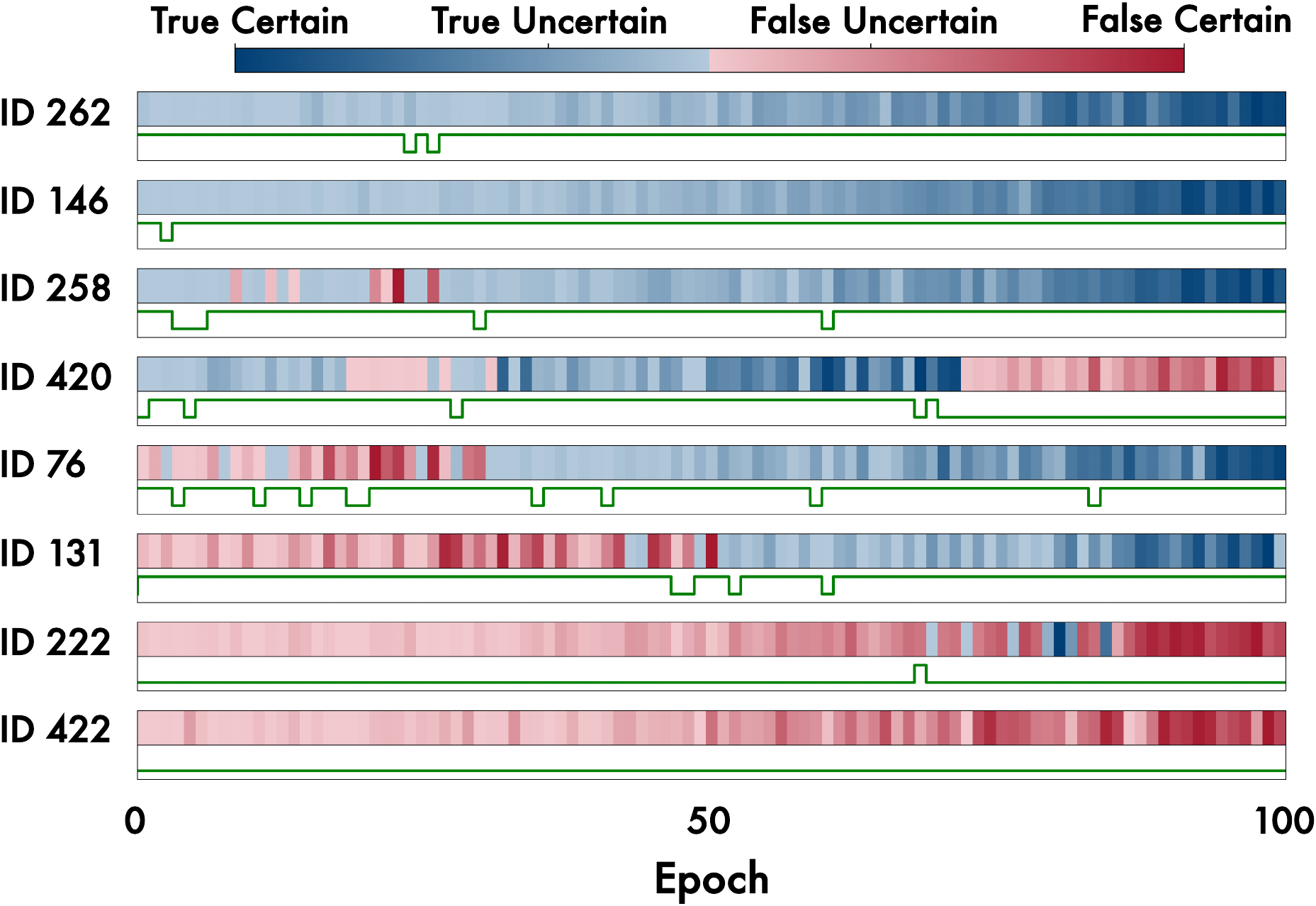}
\caption{Training dynamics on individual representative problems. For each problem, the heatmap shows greedy decoding correctness across epochs (blue = correct, red = wrong; darker = higher confidence), and the green wave indicates whether the highest-reward rollout is correct.}
\vspace{-15pt}
\label{fig:heatmap_step_label}
\end{wrapfigure}

\noindent\textbf{Results.} \Cref{fig:heatmap_step_label} shows 8 representative cases and the full results appear in \Cref{fig:heatmap_step_all_label}.
We track greedy decoding correctness (heatmap) and whether the highest-reward sample is correct (green binary square-wave indicator).
We observe four distinct patterns:

% \begin{itemize}[topsep=0pt, partopsep=0pt, leftmargin=12pt, itemsep=0pt]
\noindent$\bullet$ \textbf{Amplifying success} (ID 262, 146, 258): These problems start with correct greedy decoding (blue at epoch 0), and the highest-reward sample remains correct throughout training (green wave stays at 1). Training simply increases the model's confidence around this already-correct solution, shown by the deepening blue color.
    
\noindent$\bullet$ \textbf{Amplifying failure} (ID 222, 422): The highest-reward sample is almost wrong throughout (green wave mainly stays at 0). Training amplifies the model's confidence in these incorrect answers, shown by the deepening red color.
    
\noindent$\bullet$ \textbf{Wrong} → \textbf{Correct} (ID 76, 131): Greedy decoding initially produces wrong answers (red), but the highest-reward sample is usually correct during training (green wave mostly at 1). This guides the model from wrong to correct, as shown by the transition (red → blue).
    
\noindent$\bullet$ \textbf{Correct} → \textbf{Wrong} (ID 420): The model starts correct (blue), but the highest-reward sample fluctuates between correct and wrong (green wave alternates between 1 and 0). This inconsistency causes the model to gradually degrade from correct to wrong (blue → red).

% \end{itemize}

% \noindent\textbf{Note.}
Among the 25 problems, training flipped greedy correctness in just 3 cases (12\%).
The remaining 22 simply sharpened the model's initial preference regardless of whether correct or wrong.
Even for the 3 that flipped, the color still deepens over training, showing that sharpening happens regardless of whether correctness changes.
This reveals training as amplification rather than correction.

The key factor appears to be whether the highest-reward sample is mostly correct.
% When it is, sharpening helps. When it isn't, sharpening hurts.
But this raises a question: does sharpening only affect the problem being trained, or does it generalize to others?

\subsubsection{Out-Of-Distribution Cross-Problem Generalization}
\label{sec:cross}

\begin{figure*}[!t]
    \centering
    \vspace{-10pt}
    \includegraphics[width=1\linewidth]{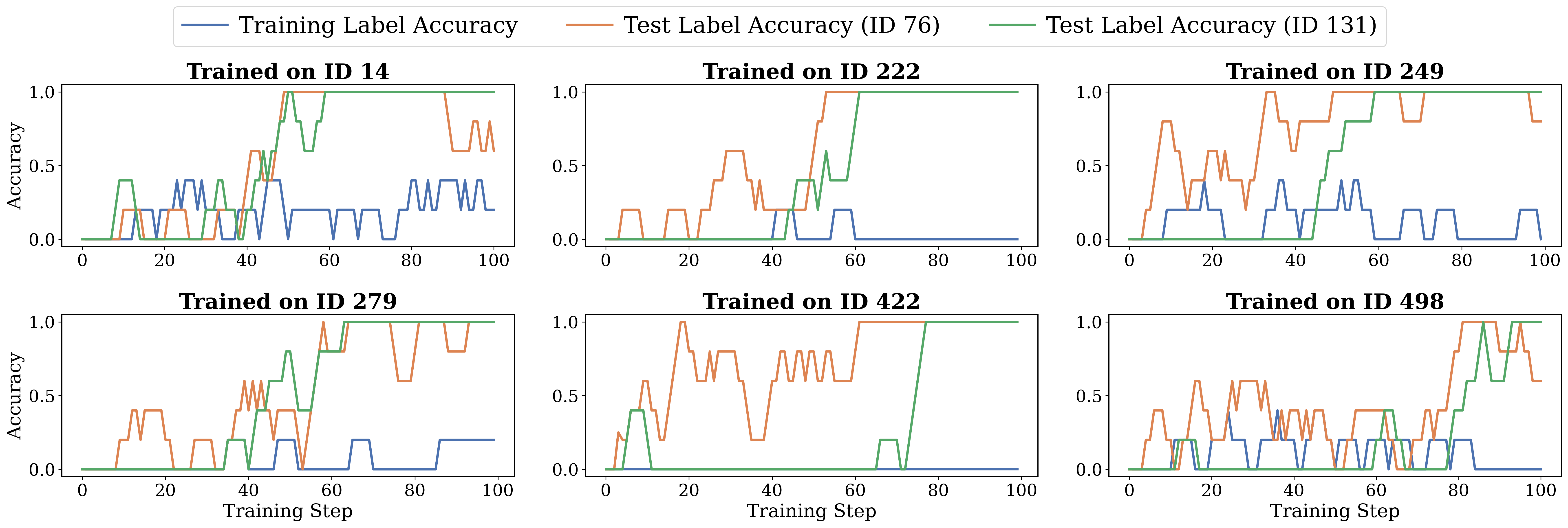}
    \vspace{-20pt}
    \caption{\textit{Training Label Accuracy} (\textcolor[rgb]{0.298,0.447,0.690}{blue}) on six MATH500 problems and \textit{Test Label Accuracy} on two OOD problems: ID 76 (\textcolor[rgb]{0.867,0.518,0.322}{orange}) and ID 131 (\textcolor[rgb]{0.333,0.659,0.408}{green}).}
    \label{fig:ood_generalize}
    \vspace{-10pt}
\end{figure*}

% \begin{figure*}[!t]
%     \centering
%     \includegraphics[width=1\linewidth]{figures/generalize/ood_generalize_result.png}
%     \caption{The training accuracy (\textcolor{blue}{blue}) on six distinct MATH500 problems alongside testing accuracy on two fixed OOD problems: Level 2 ID 76 (\textcolor{orange}{orange}) and Level 4 ID 131 (\textcolor{green}{green}). 
%     Notably, the model achieves significant generalization gains on OOD problems even when it fails to solve the training instances.This demonstrates that URLVR optimizes the underlying reasoning process through intrinsic rewards, enabling cross-problem generalization independent of training sample memorization.}
%     \label{fig:ood_generalize}
%     % \vspace{-4mm}
% \end{figure*}

To investigate the generalization of intrinsic URLVR, we design a targeted experiment focused on out-of-distribution (OOD) cross-problem generalization.

\noindent\textbf{Setup.} We train Qwen3-1.7B-Base on 6 problems from MATH500 using the same setting as in \Cref{sec:id}, representing a highly constrained one-shot RL scenario. For each problem, the highest-reward sample is mostly wrong, and the model is trained on it individually for 100 epochs with batch size 1 and 8 rollouts for baseline estimation. For evaluation, we selected two unseen problems (ID 76 and ID 131), ensuring no overlap with the training set. We track \textit{Training Label Accuracy} (the correctness of the highest-reward sample) and \textit{Test Label Accuracy} (the correctness of greedy decoding test problems).
All results are smoothed using a moving average with a rolling window of 5.

\noindent\textbf{Results.} \Cref{fig:ood_generalize} shows that all training problems have a low \textit{Train Label Accuracy} but for unseen test problems, they successfully turn from wrong to correct, with \textit{Test Label Accuracy} increasing steadily from 0 to 1. This indicates that even though all training problems have wrong initial answers and intrinsic URLVR amplifying their failures, the sharpening can still generalize to OOD problems.
As long as the confidence aligns well with correctness on these unseen problems, sharpening still works.

\vspace{-5pt}
\section{How Can Sharpening from Intrinsic URLVR Be Applied Safely?}
\label{sec:q3}

\begin{tcolorbox}[takeawaysbox]
Small datasets induce localized rather than systematic policy shift, even training on wrong problems can yield gains, making test-time training a safe and practical application.
\end{tcolorbox}

The previous section establishes that intrinsic URLVR drive convergence to deterministic policies and may result in reward hacking, limiting its application in scaling LLM training.
But do intrinsic URLVR always lead to model collapse, or can they be safely applied under specific conditions?
In this section, we demonstrate that model collapse can be prevented when training data is sufficiently small and domain-specific, making intrinsic URLVR particularly suitable for test-time training scenarios.

\subsection{Small Datasets Prevent Model Collapse}
\label{sec:small}

In \Cref{sec:fine_grained}, we have seen how training on one problem works. Now we want to see whether the training dataset size influences the intrinsic URLVR performance.

\noindent\textbf{Setup.} We train Qwen3-1.7B-Base using majority voting reward on randomly sampled training subsets from DAPO-17k with $\{32, 128, 512, 2048, 8192, 16384\}$ samples. To ensure fair comparison, we fix global batch size to $32$ and adjust epochs so all settings complete exactly $600$ optimization steps. We monitor \textit{Ground Truth Reward}, \textit{Majority Voting Reward}, and \textit{Reward Accuracy} to detect reward hacking. Results are verified across 3 random seeds for subset sizes $\{32, 128, 512\}$.

% \noindent\textbf{Setup.} We trained Qwen3-1.7B-Base on randomly sampled training subsets from DAPO-17k with $\{32, 128, 512, 2048, 8192, 16384\}$ samples. To ensure fair comparison, we fix global batch size to $32$ and adjust training epochs so all settings complete exactly $600$ optimization steps. We mainly monitor \textit{Ground Truth Reward} (oracle reward computed using actual correctness), \textit{Majority Voting Reward} and \textit{Reward Accuracy}, indicating whether reward hacking occurs.

\noindent\textbf{Results.} \Cref{fig:5p4p1} reveals that training with $\leq128$ samples maintains stable performance without collapse, while larger datasets ($\geq 512$) consistently exhibit reward hacking. Notably, DAPO-32 achieves rapid consensus (\textit{Majority Voting Reward} $\rightarrow 1$) while preserving high \textit{Ground Truth Reward}, demonstrating that the model converges on these problems without collapsing. This distinction holds across all three seeds, where DAPO-32 never collapses, while DAPO-512 always does.

% \noindent\textbf{Results.} Training on small datasets consistently prevents collapse. \Cref{fig:5p4p1} reveals a clear threshold that training with $32$ or $128$ samples maintains stable performance without model collapse. Critically, while DAPO-32 achieves rapid consensus (\textit{Majority Voting Reward} $\rightarrow 1$), it preserves high \textit{Ground Truth Reward}, demonstrating that the model converges on these data and doesn't collapse. In contrast, larger datasets ($\geq 512$ samples) consistently exhibit the classic reward hacking. This threshold suggests that intrinsic URLVR becomes unstable when training data provides sufficient statistical power to reinforce systematic biases that could mislead the majority voting mechanism.

% \noindent\textbf{Note.} To verify consistency of this result, we have conducted the same experiments with 3 random seeds for subset sizes $\{32, 128, 512\}$. Despite small quantitative variance, the qualitative distinction between small ($32$) and large ($\geq 512$) subsets is consistent. DAPO-32 does not collapse catastrophically across three seeds, while DAPO-512 always does.

\begin{figure*}[!t]
    \centering
    \includegraphics[width=1\linewidth]{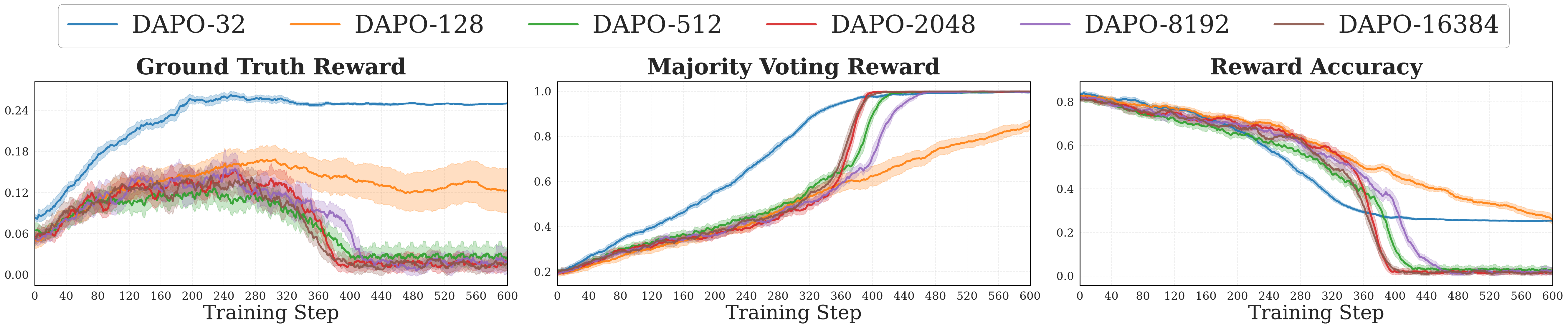}
    \caption{Effect of training dataset size from $\{32, 128, 512, 2048, 8192, 16384\}$. Training with $32$ or $128$ samples maintains stable performance without model collapse.}
    % \vspace{-20pt}
    \label{fig:5p4p1}
\end{figure*}

\noindent\textbf{Analysis.}
We hypothesize that small datasets induce localized overfitting rather than systematic policy shift. With only 32 problems, even perfect overfitting learns isolated facts rather than generalizable patterns. To test this, we measure the KL divergence from the reference model at each training step:

% To further investigate why small datasets prevent model collapse. We conduct a pilot study and offer one possible explanation below. We hypothesize that \textbf{small datasets may induce localized overfitting rather than systematic policy shift}.

% With only 32 problems, even perfect overfitting learns $32$ isolated facts rather than generalizable patterns. In contrast, training on $\geq 512$ diverse problems lets the model extract systematic biases that span across problems. This distinction manifests in how far the policy drifts from its original distribution. We measure this using KL divergence at each training step $t$ from the reference model:

\begin{equation}
D_{\text{KL}}^{(t)}(\pi_{\theta}^{(t)} \| \pi_{\text{ref}}) = \mathbb{E}_{x \sim \mathcal{D}_{\text{train}}} \left[ \mathbb{E}_{y \sim \pi_{\theta}^{(t)}(\cdot|x)} \left[ \log \frac{\pi_{\theta}^{(t)}(y|x)}{\pi_{\text{ref}}(y|x)} \right] \right]
\end{equation}

\begin{wrapfigure}{r}{0.5\textwidth}
% \begin{figure}
\vspace{-17pt}
\centering
\includegraphics[width=\linewidth]{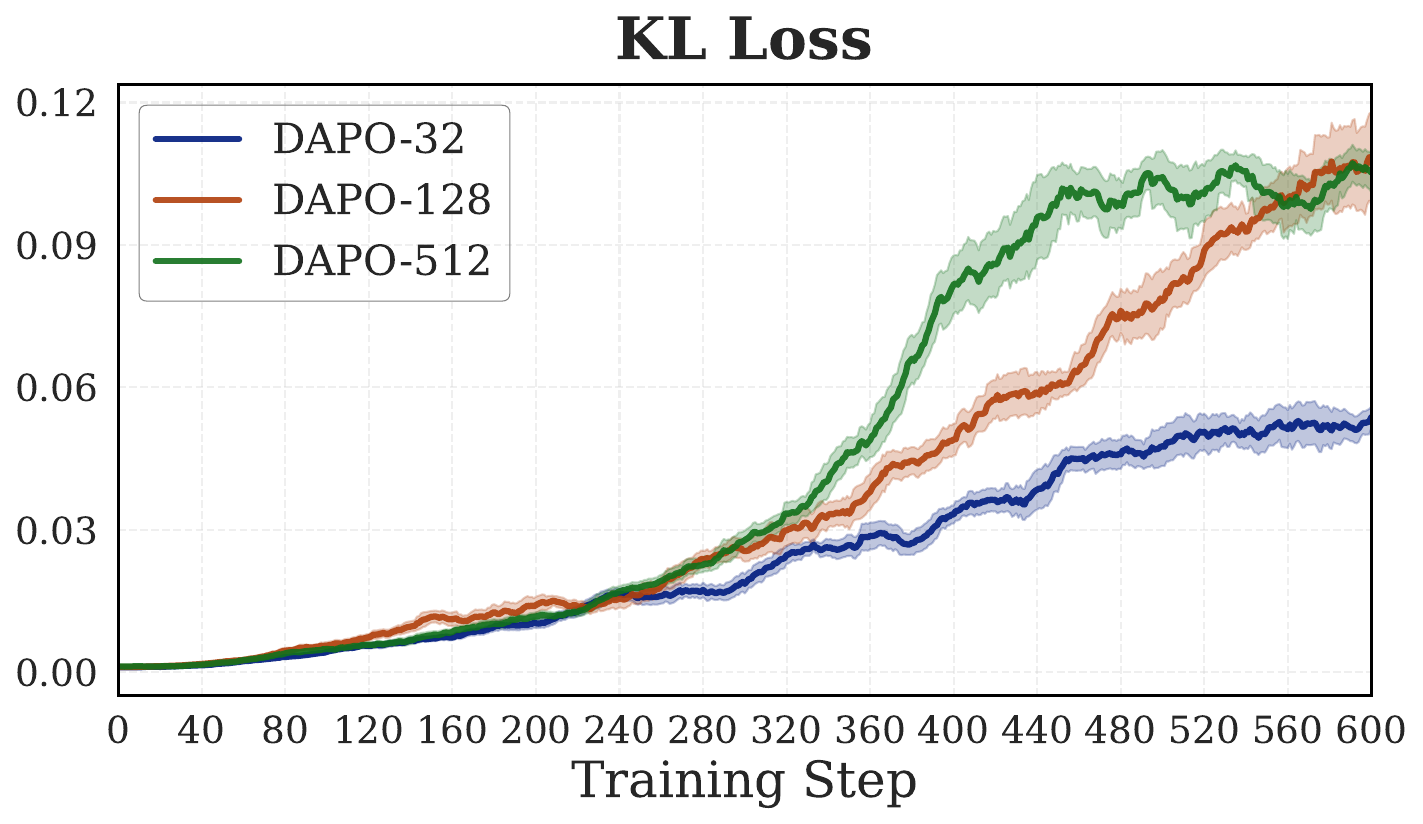}
\vspace{-20pt}
\caption{KL divergence at each training step.}
\vspace{-10pt}
\label{fig:5p2p2}
% \end{figure}
\end{wrapfigure}

As shown in \Cref{fig:5p2p2}, smaller subsets induce far smaller distributional shifts. DAPO-32 reaches only 0.057 KL after 600 steps, while DAPO-512 reaches $2\times$ higher. This aligns with findings that catastrophic forgetting correlates with distribution shift between fine-tuned and base policies~\citep{shenfeld2025rl}. The limited drift suggests the model sharpens confidence on specific samples through localized parameter updates~\citep{carlsson2024hyperfitting} without altering its global policy, thus preserving general reasoning on AIME24/AMC23. In contrast, large datasets require dense parameter updates that cause global policy shift, causing model collapse.

% \noindent \textbf{Results.} As shown in \Cref{fig:5p2p2}, smaller subsets induce smaller distributional shifts despite heavy repetition. DAPO-32 reaches only 0.057 KL loss after 600 steps, while DAPO-512 reaches 2×. This aligns with \citet{shenfeld2025rl} that catastrophic forgetting correlates with distribution shift between fine-tuned and base policies.

% The limited shift suggests memorization rather than systematic learning. Small datasets trigger what \citet{carlsson2024hyperfitting} call ``hyperfitting'', where achieving low loss through memorization rather than generalizable patterns. The model sharpens confidence on specific samples through localized parameter updates without changing its global policy. When tested on AIME24/AMC23, it doesn't encounter the memorized problems and relies on preserved general reasoning.

% In contrast, large datasets require dense parameter updates to fit the data, causing global policy shift that manifests as collapse on out-of-distribution problems. This distinction between localized memorization and systematic policy learning may explain why small-scale training remains stable.

\subsection{Test-Time Training as a Safe Application}

Building on the small dataset insights, we examine test-time training where models are adapted directly on the target evaluation domains without ground truth labels, which may fit the small dataset size constraint. We investigate whether it can bring improvement while not leading to collapse.

\noindent\textbf{Setup.} We train Qwen3-1.7B-Base using majority voting reward on two settings: AMC23 (40 problems, test-time) and DAPO-17k ($\sim$17,000 problems, train-time). Both use batch size 40. We track \textit{Ground Truth Reward}, \textit{Majority Voting Reward}, and performance on both AMC23 and AIME24.

\begin{wrapfigure}{r}{0.55\textwidth}
% \begin{figure}
\vspace{-15pt}
\centering
\includegraphics[width=\linewidth]{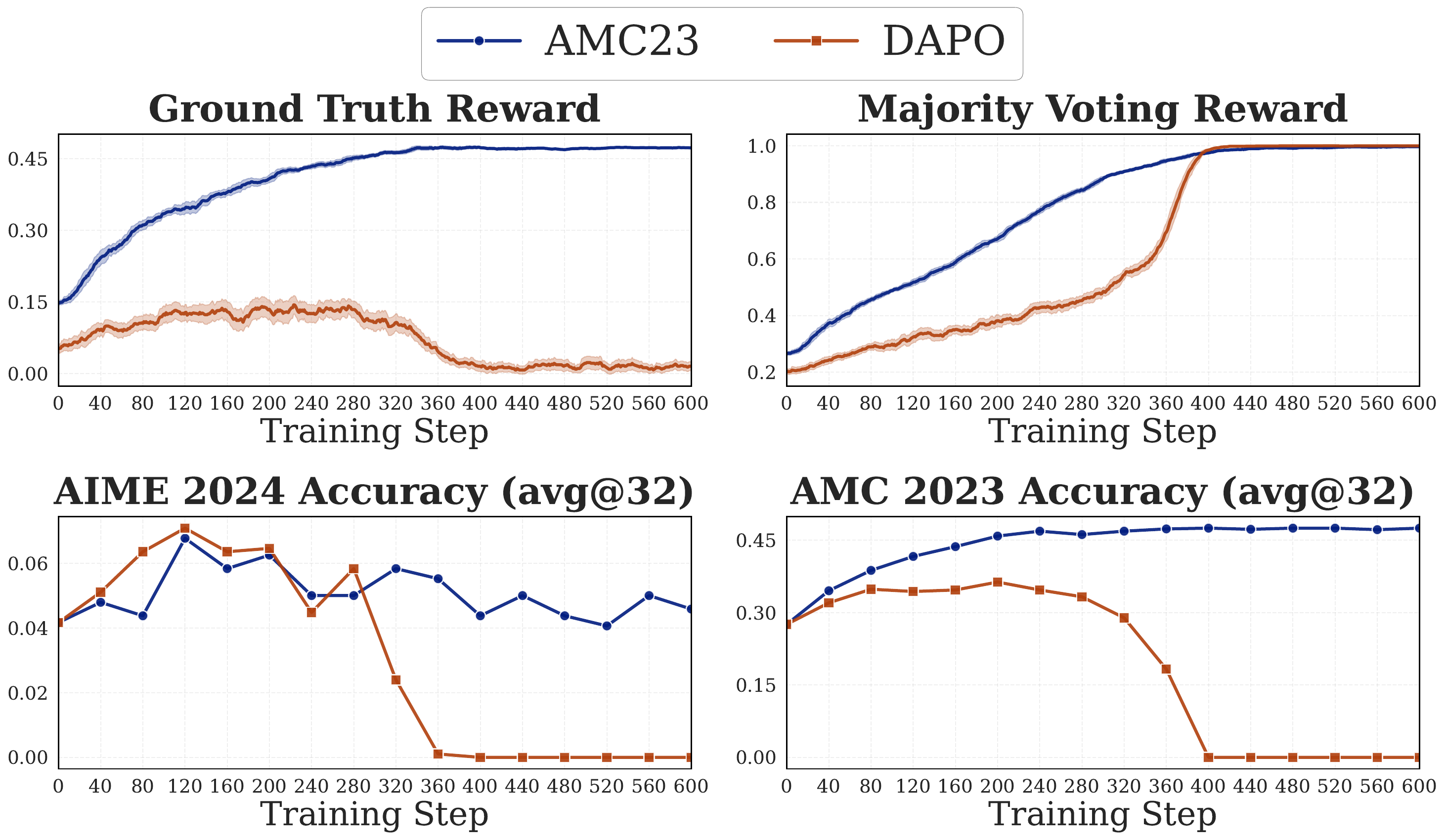}
\vspace{-15pt}
\caption{Comparison between training and test-time.}
\label{fig:5p4p2}
% \end{figure}
\vspace{-15pt}
\end{wrapfigure}

\noindent\textbf{Results.} \Cref{fig:5p4p2} shows that test-time training on AMC23 avoids collapse. Both \textit{Ground Truth Reward} and \textit{Majority Voting Reward} rise and stabilize, with performance improving on both AMC23 and AIME24.
In contrast, training on DAPO-17k shows the familiar rise-then-fall pattern.
This indicates that intrinsic URLVR may be safely applied in test-time training. And this explains why many recent works using intrinsic rewards focus on test-time rather than train-time settings~\citep{zuo2025ttrl,prabhudesai2025maximizing}.

\subsection{Incorrect Majority Votes Still Improve Reasoning}

Following \Cref{sec:small}, we now consider an extreme case where the initial majority votes in training are incorrect on a non-trivial proportion of examples. We want to see at this extreme case, shouldn't the small subset also collapse? Surprisingly, we observe that even when training amplifies errors on this incorrect subset, test-time training can still safely yield gains on OOD benchmarks, aligned with our cross-problem generalization findings in \Cref{sec:cross}.

\begin{wrapfigure}{r}{0.55\textwidth}
% \begin{figure}
\centering
\includegraphics[width=\linewidth]{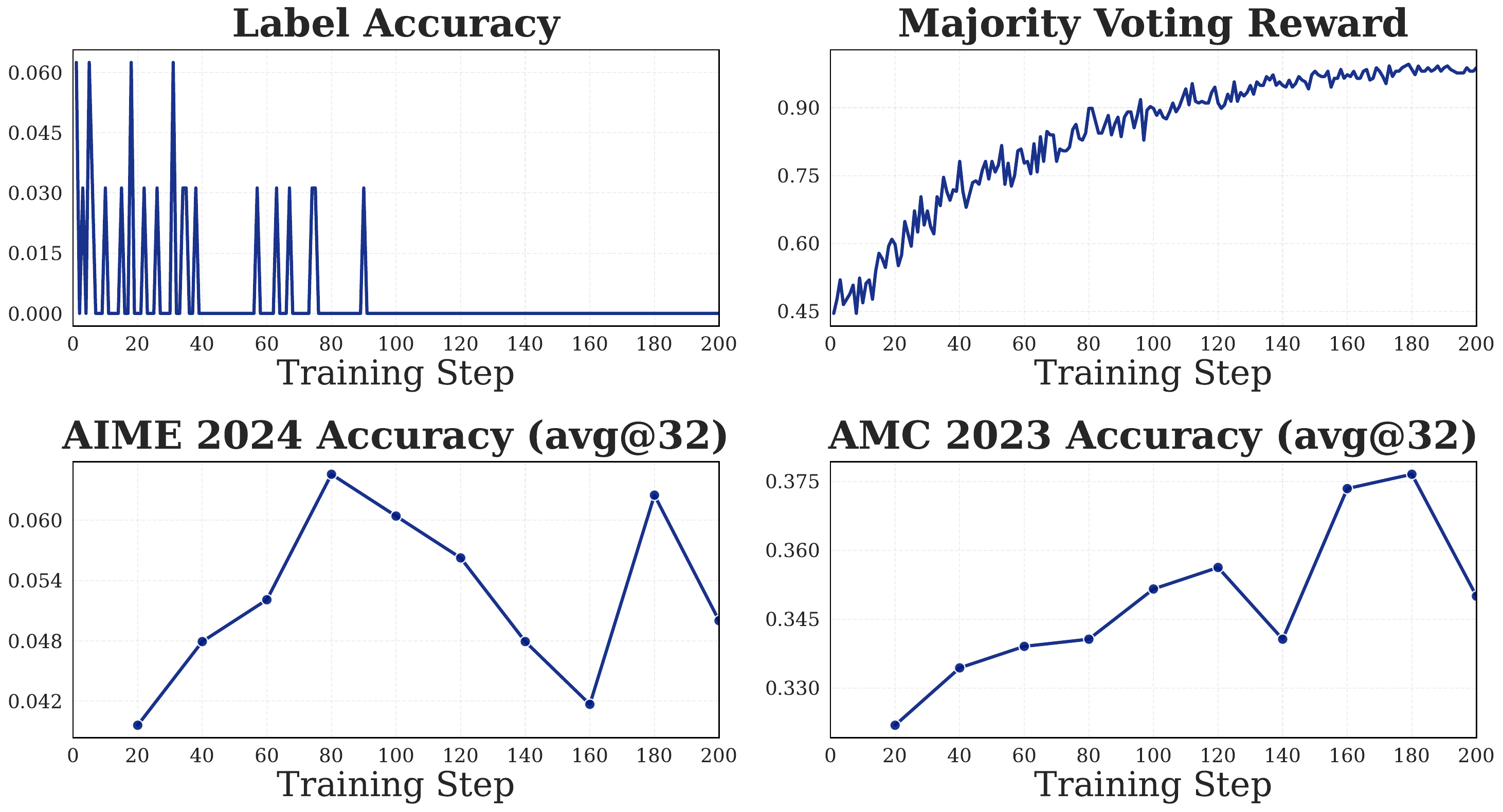}
\caption{Training dynamics of extreme DAPO-32 setting, where almost all initial majority votes are incorrect.}
\label{fig:5p2p1}
% \end{figure}
\end{wrapfigure}

\noindent\textbf{Setup.}
We first perform offline filtering on DAPO-17k by sampling 64 responses per prompt from the base model and computing maj@64. Note that during training, we actually use maj@8 with temperature 1.0, which introduces some randomness.
To ensure that the majority vote is incorrect in a non-negligible portion, we deliberately filter offline with maj@64 and control for higher majority ratios (>40\%). We then trained on 32 filtered samples using the same setting as DAPO-32.
We monitor \textit{Label Accuracy} (measure whether maj@8 during training match ground truth label), \textit{Majority Voting Reward} and the actual performance on validation benchmarks.

\noindent\textbf{Results.} As shown in \Cref{fig:5p2p1}, we observe non-zero accuracy at a few early steps, with only 1-2 maj@8 matching ground truth labels. After 100 steps, it consistently shows zero \textit{Label Accuracy} with a convergent trend in \textit{Majority Voting Reward}. Even when almost all 32 samples have incorrect initial majority votes, training still produces effective learning without catastrophic collapse.
The gains observed on AIME24 and AMC23 demonstrate that small-scale training operates under fundamentally different dynamics than large-scale training.
This validates that small subsets may avoid collapse through localized overfitting rather than systematic policy shift.

\section{How Can We Measure Model Prior?}
\label{sec:prior}

% \section{A Practical Tool as Model Prior Indicator: Predicting RL Trainability}
% \label{sec:prior}

\begin{tcolorbox}[takeawaysbox]
% We propose \textit{Model Collapse Step} as a novel model prior indicator, measuring when reward accuracy collapses during intrinsic URLVR.
% It strongly predicts standard RL trainability, matching or surpassing pass@k, while requiring no ground truth labels and remaining robust to multiple-choice gaming.
We propose the \textit{Model Collapse Step} as a novel indicator of model priors, which measures standard RL trainability by tracking reward accuracy collapses during intrinsic URLVR.
This indicator achieves accuracy in assessing trainability on par with running standard RL itself, but with higher efficiency; it outperforms pass@k, requires no ground-truth labels and remains robust to multiple-choice problems.
\end{tcolorbox}

Previous sections suggest that intrinsic URLVR is effective only when the model’s initial confidence is aligned with correctness.
This raises a practical question: \textbf{Can the strength of this alignment be used to estimate the model prior rather than running expensive RL training to select the base model?}
% 目前学界通常会通过直接在不同的base model上跑完整的rl run来选择更适合rl的ckpt，并通常认为这种模型模型先验更强。
The community typically selects the base model by running full RL training on multiple checkpoints and considers such models to have a stronger prior.
% 但是这种方法显而易见的costly和低效。
However, this method is obviously costly and inefficient.
Another standard approach is $pass@k$~\citep{wu2025miragemethodmodeltaskalignment}, measuring by sampling $k$ solutions without training.
This works but also has limitations.
% First, it needs ground truth for verification, making it unusable for subjective tasks.
% 其次，这种方法的准确率会显著更低。
First, this method has a significantly lower accuracy.
And it also struggles with multiple-choice questions, where $pass@k \rightarrow 1$ when $k$ is sufficiently large.
% by simply trying all options.

% We start by systematically conducting a broad observational study of training dynamics. In \Cref{app:backbone}, we include a comparison of 11 models, covering various series (Qwen2.5/Qwen3/OctoThinker/Llama3.1), various training stages (base/math base/sft/instruct), and various sizes (1.7B, 3B, 4B, 8B). This analysis observes that different models indeed have vastly different stability profiles

In this section, we first examine intrinsic URLVR on different models, showing they lead to different outcomes.
Then, we leverage the rise-then-fall pattern, proposing \textbf{Model Collapse Step} and measuring how long a model can sustain intrinsic URLVR before it starts to collapse.
Formally, we define it as the training step where \textit{Reward Accuracy} drops below 1\%.
Models with stronger priors remain stable for longer before collapsing, indicating that they are more suitable as base models for RL.

% We propose the \textbf{Reward Reliability Horizon (RRH)}, a dynamics-based diagnostic that observes when \textit{Reward Accuracy} falls below 1\% during brief intrinsic URLVR training. Models with stronger confidence-correctness alignment maintain high reward accuracy for more training steps before intrinsic rewards begin systematically favoring incorrect solutions. Unlike $pass@k$'s static sampling, RRH captures the temporal evolution of confidence-correctness correlation through training dynamics. We demonstrate that RRH is (1) \textbf{accurate}: strongly predicts RL performance gains, matching or exceeding $pass@k$ while being applicable to subjective tasks and robust to multiple-choice questions; and (2) \textbf{rapid}: can be computed within comparable time to $pass@k$ by using aggressive hyperparameters that accelerate convergence without changing relative model rankings.

\subsection{Pilot Study of Different Models}

We first observe that intrinsic URLVR behaves differently across model families and training stages.

\begin{figure*}[!h]
    % \vspace{-30pt}
    \centering
    \includegraphics[width=0.49\linewidth]{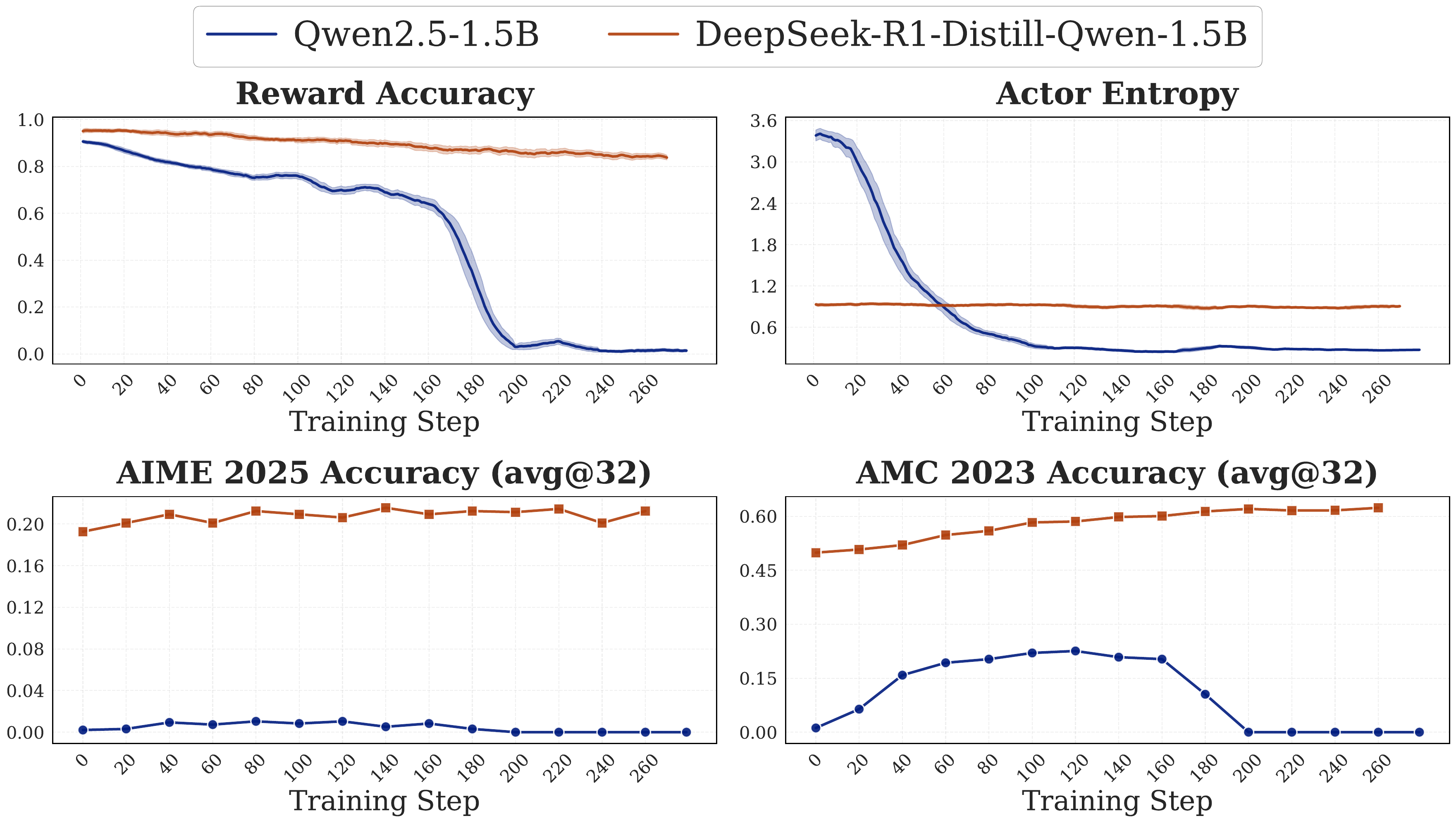}
    \hfill
    \vrule width 0.5pt
    \hfill
    \includegraphics[width=0.49\linewidth]{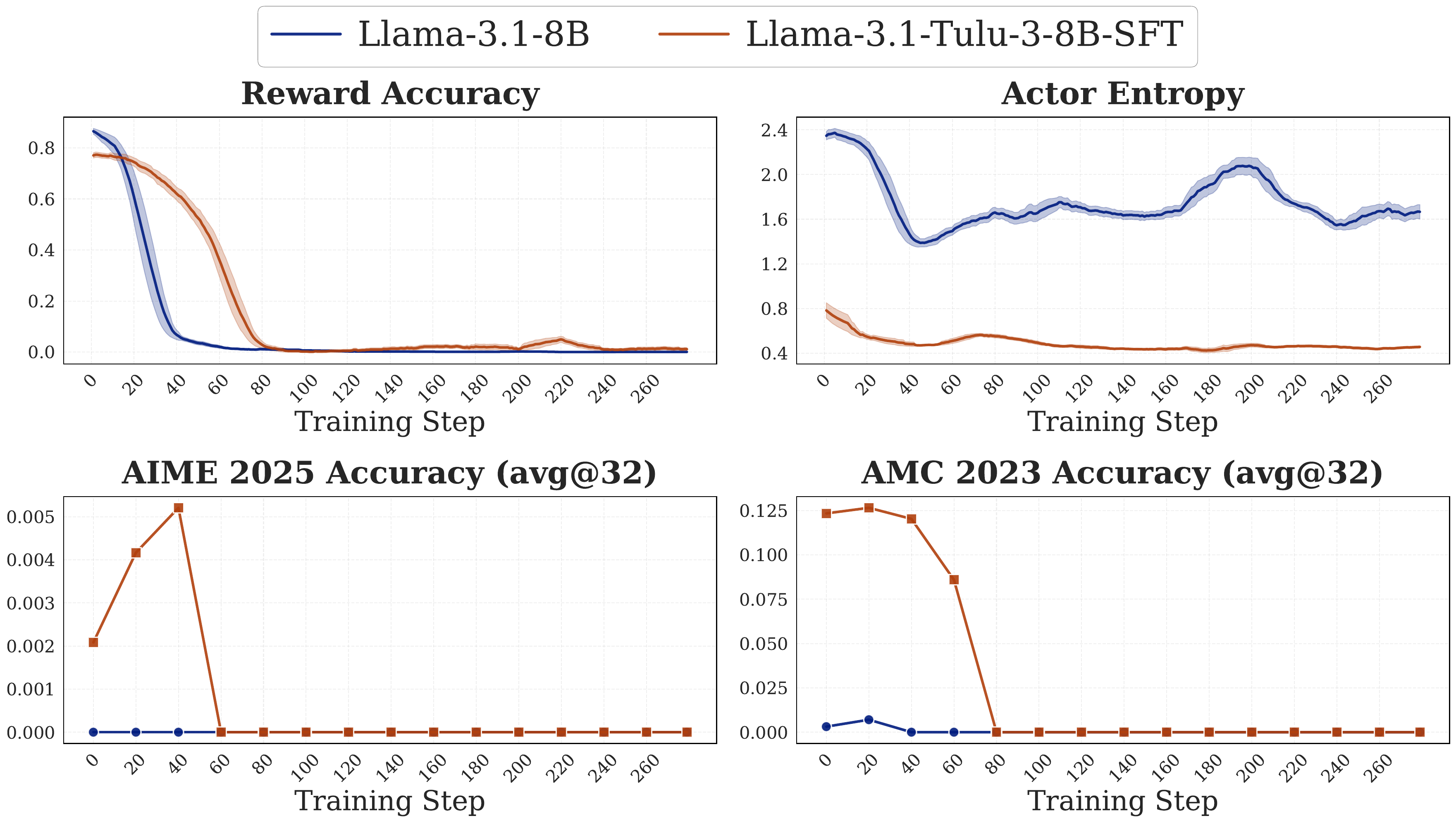}
    % \vspace{-10pt}
    \caption{RL Training dynamics after different training stages in Qwen (left) and LLaMA (right) family.
    In Qwen, SFT enables stable training while base collapses by step 200.
    In LLaMA, both eventually collapse but SFT delays failure.}
    \label{fig:4p1p3}
    \vspace{-10pt}
\end{figure*}

\noindent\textbf{Setup.} We compare four models: Qwen2.5-1.5B and DeepSeek-R1-Distill-Qwen-1.5B from the Qwen family; Llama-3.1-8B and Llama-3.1-Tulu-3-8B-SFT from the LLaMA family. This lets us separate architectural effects from training stage effects. All models are trained on DAPO-17k with majority voting reward. We track \textit{Reward Accuracy}, \textit{Actor Entropy}, and validation performance.

\noindent\textbf{Results.} The two families show strikingly different patterns (\Cref{fig:4p1p3}). In Qwen family (left), the SFT variant maintains \textit{Reward Accuracy} above 0.8 throughout training while the base model drops to near zero by step 200, despite starting with higher \textit{Actor Entropy}. In LLaMA family (right), both variants eventually collapse but at different rates. The base model fails by step 40, while the SFT variant shows an initial performance rise before collapsing later. This architectural difference highlights that Qwen models appear fundamentally more stable, aligned with \citet{shao2025spurious}.

% \noindent\textbf{Note.}
\citet{zhang2025no} explains these dynamics through entropy minimization that early gains come from learning correct output formats, while later degradation stems from suppressing high-entropy transitional words that support reasoning.
While this view explains how intrinsic rewards sharpen distributions, it cannot predict when sharpening helps versus hurts.
\Cref{fig:4p1p3} reveals the missing piece that base models from both families start with higher \textit{Actor Entropy} yet perform worse, with lower \textit{Reward Accuracy} and faster collapse.
If high entropy enables better reasoning, base models should outperform their SFT models. The opposite pattern suggests entropy is a consequence of the sharpening process, not its determinant and not enough to predict RL trainability and model prior.

\begin{figure*}[!t]
    \centering
    \includegraphics[width=1\linewidth]{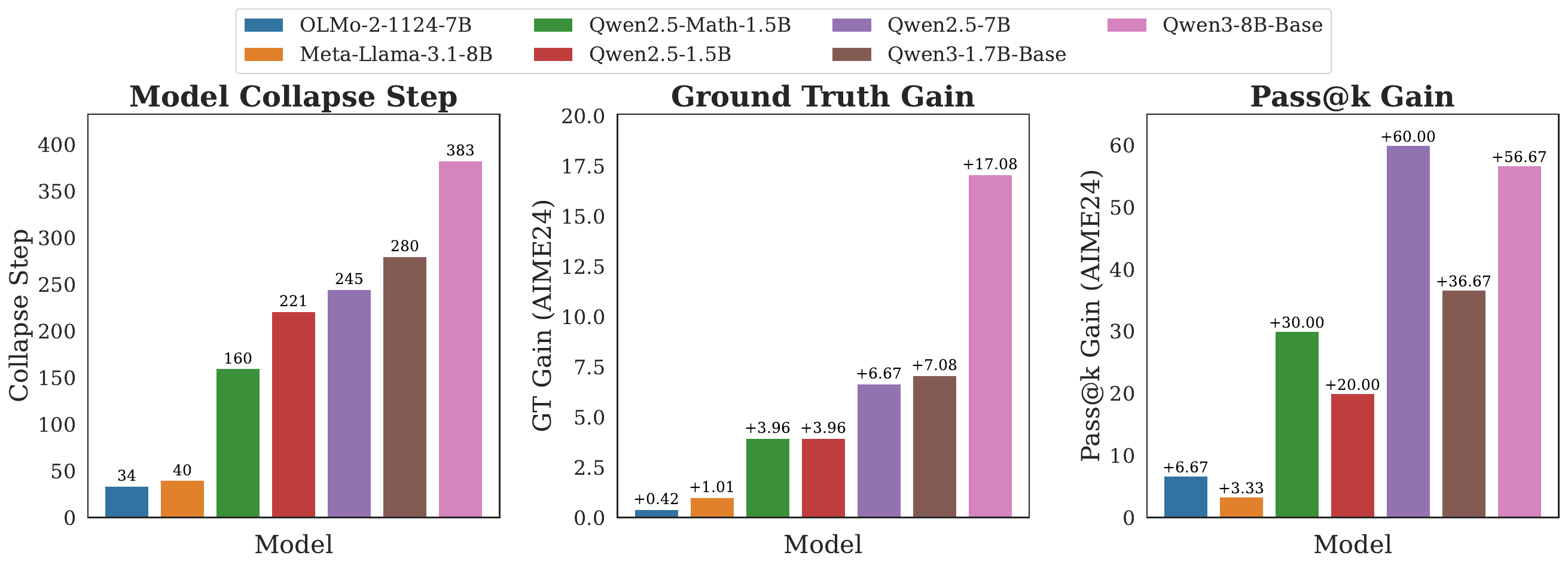}
    \caption{Comparison between Model Collapse Step (left) and $\mathbf{Pass@k}$ Gain (right) as predictors of RL trainability, with GT Gain (middle) as reference. Later collapse strongly predicts better RL performance.}
    \label{fig:5p5p1}
\end{figure*}

\subsection{Model Collapse Step Accurately Predicts RL Gains}

To validate Model Collapse Step as an accurate predictor of RL trainability, we compare it against the widely-used $pass@k$ metric and standard RLVR training.

\noindent\textbf{Setup.} We evaluate 7 models from 3 families (OLMo, LLaMA, Qwen) on AIME24, using three indicators:
(1) \textbf{Ground Truth (GT) Gain}, measuring performance improvement from standard supervised RLVR on 1 epoch DAPO-17k training with ground truth rewards;
(2) \textbf{$\mathbf{Pass@k}$ Gain}, calculated as the difference between $pass@256$ and $pass@1$ on AIME24;
and (3) \textbf{Model Collapse Step}, measured as the training step when \textit{Reward Accuracy} falls below 1\% during intrinsic URLVR training with majority voting reward with default hyperparameters.
We evaluate the correlation between these indicators and use them to measure the model prior.

\noindent\textbf{Results.} As shown in \Cref{fig:5p5p1}, Model Collapse Step correlates strongly with GT Gain. Models that survive longer with a larger Model Collapse Step consistently yield better results in standard supervised RL training. This metric matches and even surpasses $pass@k$'s predictive power and it cannot be gamed by random guessing on multiple-choice questions, demonstrating its reliability as an indicator.

\subsection{Model Collapse Step Rapidly Predicts RL Gains}

While Model Collapse Step provides an accurate prediction, can it be computed efficiently? Based on insights from the hyperparameter tuning (\Cref{app:hyperparameter}), we hypothesize that certain hyperparameters, such as mini-batch size and rollout number, can accelerate convergence without sacrificing the accuracy of the model prior indicator for RL trainability. This is valuable for practitioners evaluating multiple model candidates before committing to expensive RL training.

\noindent\textbf{Setup.} We test whether it remains consistent across different training hyperparameters by varying mini-batch size $\in \{1, 8, 64\}$ and rollout count $\in \{8, 16, 32\}$. For each setting, we measure for the same 7 models and assess whether we can preserve predictive power while reducing computation time.

\begin{figure}[!t]
\centering
\includegraphics[width=\linewidth]{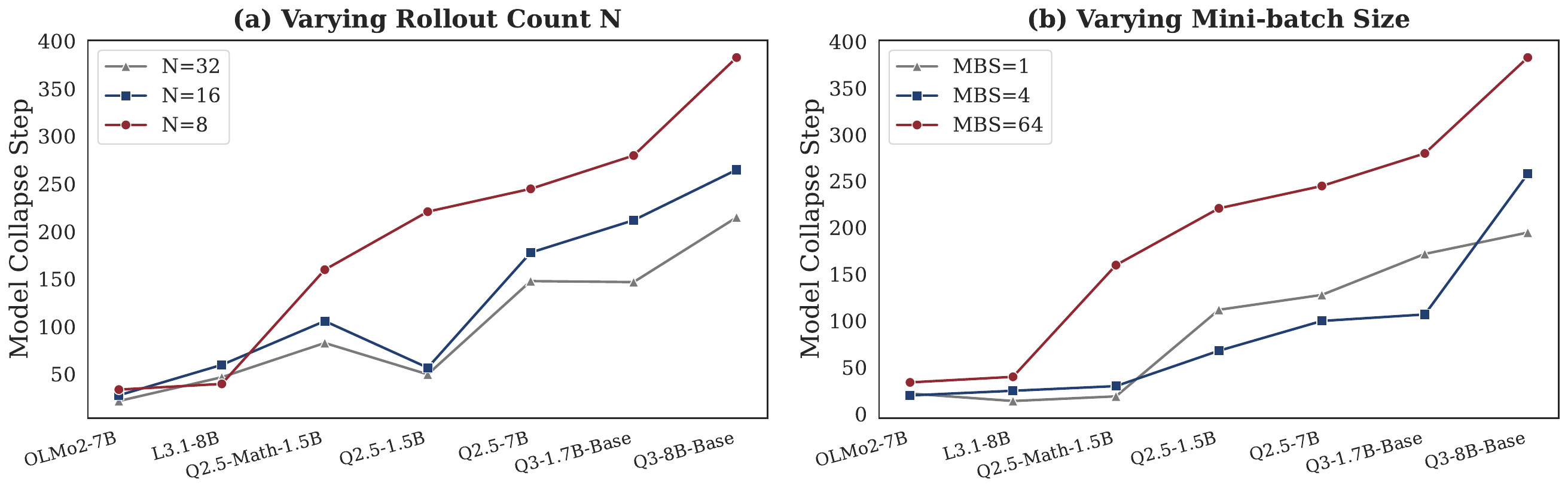}
\caption{Consistency across different rollout counts (left) and mini-batch sizes (right). Aggressive hyperparameters accelerate collapse but preserve relative model rankings and predictive power.}
\label{fig:5p5p2}
\end{figure}

\noindent\textbf{Results.} \Cref{fig:5p5p2} shows that while aggressive hyperparameters ($MBS=1$ or $N=32$) accelerate collapse in absolute steps, the ranking of models remains relatively stable. This enables more rapid assessment for models with strong prior ($\geq50$ steps less than before), effectively shortening the time required to measure model priors, making Model Collapse Step not only accurate but also rapid. 

\begin{table}[t]
\centering
\caption{Computation cost comparison between Model Collapse Step and the gold standard (GT Gain) for assessing RL trainability across 7 models.}
\label{tab:efficiency}
\scriptsize
\resizebox{\textwidth}{!}{
\begin{tabular}{lccc}
\toprule
\textbf{Indicator} & \textbf{Computation Cost} & \textbf{Total Tokens}  & \textbf{Requires GT} \\
\midrule
GT Gain & $7k × 8 × 17k × 7$ & 6.66B & Yes \\
 & (response × rollouts × problems × models) & (baseline) \\
% \midrule
% Pass@k Gain & $7k × 256 × 30 × 7$ & 376M & Yes \\
%  & (response × k × problems × models) & (17.7× faster) \\
% \midrule
\textbf{Model Collapse Step} & $7k × 8 × 662 × 32$ & \textbf{1.19B} & \textbf{No} \\
 & (response × rollouts × total steps × batch) & \textbf{(5.6× faster)} \\
\bottomrule
\end{tabular}}
\end{table}

To further quantify this, we compare the computation cost of Model Collapse Step against the gold standard GT Gain. For Model Collapse Step, we adopt the tuned aggressive hyperparameters ($MBS=1$, $N=8$) and record the collapse steps across seven models: [22, 14, 19, 112, 128, 172, 195], totally 662 steps. These values are also visualized in \Cref{fig:5p5p2} (right, grey line).

As shown in \Cref{tab:efficiency}, Model Collapse Step requires substantially less computation than GT Gain, using 5.6× fewer tokens than full RL training while preserving the relative ranking of models. Moreover, it requires no ground truth labels, making it applicable even when verification is unavailable.

While $pass@k$ offers an alternative static metric, our experiments show it correlates less reliably with actual RL gains (\Cref{fig:5p5p1}). By measuring training dynamics rather than pre-training performance, Model Collapse Step captures the interaction between model priors and the learning process itself. For practitioners evaluating multiple model candidates for RL, this provides a more reliable filter before committing to expensive RL training.

The predictive power of Model Collapse Step is grounded in \Cref{theorem:mv_convergence}. It captures the point at which geometric amplification shifts from beneficial (reinforcing correct) to harmful (amplifying errors). In practice, this enables efficient pre-RLVR assessment of trainability via short diagnostic runs with aggressive hyperparameters and calibrated early stopping, capturing gains while avoiding collapse.

\section{Discussion}
\label{sec:beyond_intrinsic}

\begin{tcolorbox}[takeawaysbox]
Intrinsic rewards are fundamentally bounded by what the model already knows.
External rewards grounded in unlabeled data or generation-verification asymmetry provide signals that scale with data and computation rather than saturating with model capacity, offering a more promising path towards scalable URLVR.
\end{tcolorbox}

Preceding sections reveal that intrinsic URLVR faces fundamental scalability limits rooted in confidence-correctness alignment. When this alignment is weak, intrinsic methods amplify existing biases rather than discovering new knowledge. This limitation is not merely a hyperparameter issue but stems from the nature of the reward signal itself, which derives entirely from the model's internal state and therefore cannot consistently push the model beyond what it already knows.

This motivates exploring alternative URLVR methods that can escape this ceiling. Recent works have explored external reward methods that generate verifiable signals through mechanisms independent of the model's internal state. As surveyed in \Cref{sec:external}, we identify two promising directions that leverage unlabeled data structures to derive rewards from the corpus and exploit generation-verification asymmetries.
Both paradigms offer rewards that scale with data or computation rather than saturating with model capacity, providing complementary paths toward scalable self-evolution.

% \noindent\textbf{Self-Verification as a Case Study.}
To understand how external rewards differ in practice, we examine self-verification as a concrete case study, which exploits a key asymmetry in data that generation is hard but verification is easy.
This asymmetry allows models to provide accurate rewards by themselves without ground truth labels, which is also validated in DeepSeekMath-V2~\citep{shao2025deepseekmath}~\footnote{Same method name but different implementation}.

\begin{figure*}[!t]
    \centering
    \includegraphics[width=.9\linewidth]{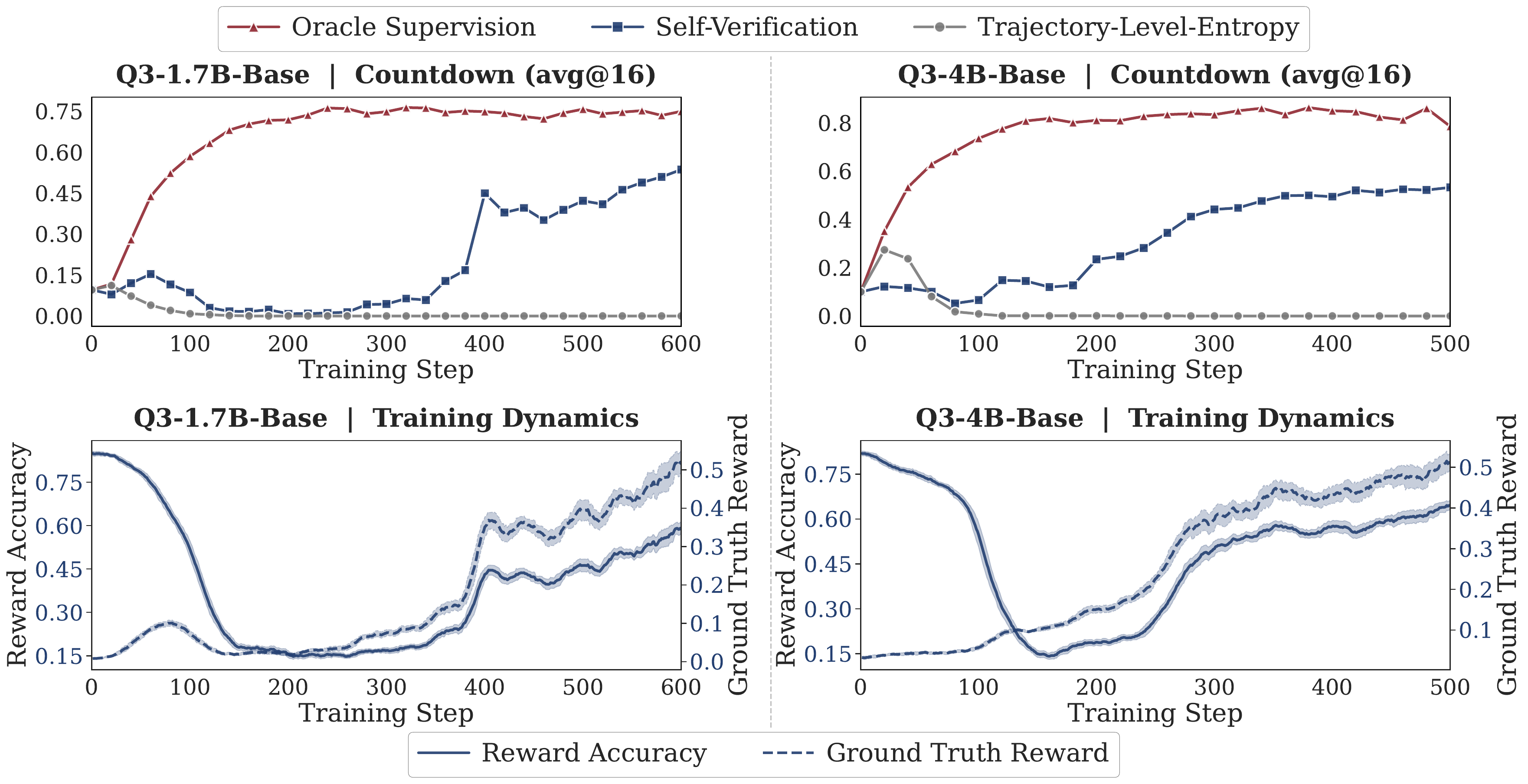}
    \vspace{-15pt}
    \caption{Top: Validation accuracy comparing Self-Verification with Trajectory-Level Entropy and Oracle Supervision. Bottom: Training dynamics for Self-Verification.}
    \label{fig:self_verify_comp}
\end{figure*}

% \noindent\textbf{Setup.}
Specifically, we train Qwen3-1.7B-Base and Qwen3-4B-Base on a subset of the Countdown task\footnote{\url{https://huggingface.co/datasets/Jiayi-Pan/Countdown-Tasks-3to4}}, where models must form arithmetic expressions that reach a target value.
For Countdown, generating correct expressions is challenging, but checking if an expression equals the target is trivial.
We sample 4k problems for training and 1k for validation.
For self-verification, the model generates solutions and evaluates them using a verification prompt that outputs binary correctness (see Prompt 2 in \Cref{app:verification_prompts}). For reward score computation, we use the ground truth scoring function\footnote{\url{https://github.com/Jiayi-Pan/TinyZero/blob/main/verl/utils/reward_score/countdown.py}} to determine whether expressions correctly evaluate to the target and compare:
(1) \textbf{Self-Verification} (our method),
(2) \textbf{Trajectory-Level Entropy} (reward from \Cref{tab:certainty_based_formulas}),
and (3) \textbf{Oracle Supervision} (training with ground truth reward).
We track validation accuracy, \textit{Ground Truth Reward} and \textit{Reward Accuracy}.

% \noindent\textbf{Results.}
Our experiments show that Self-Verification works much better than Trajectory-Level Entropy (\Cref{fig:self_verify_comp}, top).
Both models reach higher validation accuracy with Self-Verification. The bottom figure shows an interesting pattern. \textit{Reward Accuracy} initially drops around step 200 as the policy tries to exploit the verifier, then recovers and stabilizes above 0.5.
Meanwhile, \textit{Ground Truth Reward} keeps rising.
This recovery shows the model is genuinely learning to solve problems, validating that it supplies stronger signals than previous intrinsic rewards while resisting reward hacking, supporting generation-verification asymmetry as a promising direction beyond intrinsic methods.

% \subsection{Instruction Alignment Enhances Verifier Robustness}

We further find that the ability of instruction following is the key to the success of self-verification. We compare Qwen3-1.7B-Base against its instruction-aligned Qwen3-1.7B to see how model capability affects self-verification. We test both with two verification prompts (P1 and P2, see \Cref{app:verification_prompts}) to measure prompt sensitivity. We track the validation accuracy, \textit{Reward Accuracy} and \textit{Self-Verify Reward} (proxy reward used during training).

\begin{figure*}[!t]
    \centering
    \vspace{-10pt}
    \includegraphics[width=1\linewidth]{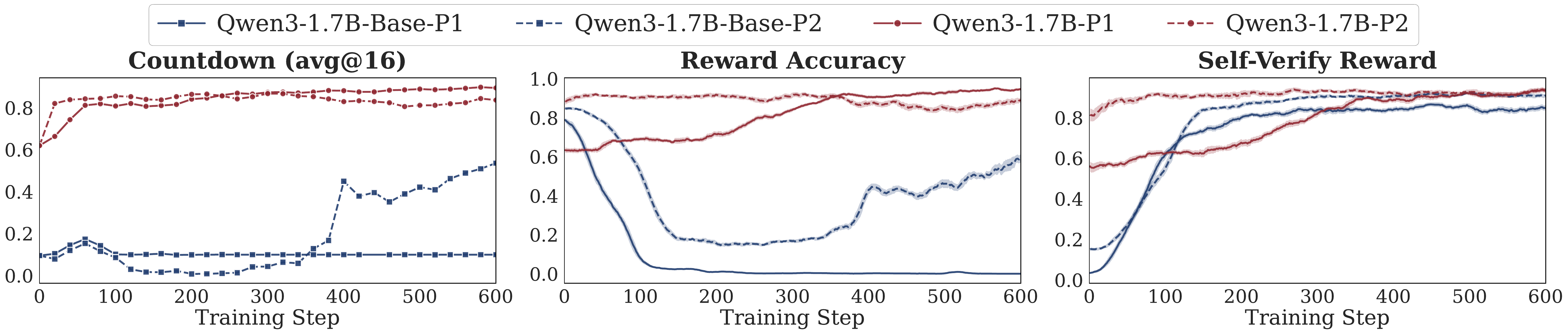}
    \caption{Prompt sensitivity across base and instruction-aligned models.}
    \label{fig:prompt_comp}
    \vspace{-10pt}
\end{figure*}

% \noindent\textbf{Setup.} We compare Qwen3-1.7B-Base against its instruction-aligned version (Qwen3-1.7B) to see how model capability affects self-verification. We test both with two verification prompts (P1 and P2, see \Cref{app:verification_prompts}) to measure prompt sensitivity. We track the validation accuracy, \textit{Reward Accuracy} and \textit{Self-Verify Reward} (proxy reward used during training).

% \noindent\textbf{Results.}
From \Cref{fig:prompt_comp}, we find that instruction alignment helps with higher starting performance and robustness to prompt choice. The instruction model starts above 60\% accuracy (surpassing the base model's final performance) and improves further to over 80\%. More importantly, it succeeds with both prompts, while the base model only works with P2. The middle figure shows instruction models maintain stable \textit{Reward Accuracy} with both prompts, while base models are highly sensitive. These results suggest self-verification can scale beyond test-time training (\Cref{sec:q3}) when combined with instruction alignment, offering a robust path for scaling RL.

These observations point to why external rewards offer a more promising path for scaling URLVR. \textbf{First, external rewards are grounded in procedures that do not degrade as the model improves.} A verifier that checks arithmetic expressions, executes code, or validates proofs against formal specifications remains equally reliable regardless of how sophisticated the model's outputs become. This contrasts with intrinsic rewards, which shift as the model's distribution changes and can amplify arbitrary patterns. \textbf{Second, external rewards can be generated at scale without human annotation.} Verification procedures are often cheap to run, and the unlabeled-data paradigm shows that even tasks without native verification asymmetries can be converted into reward-based learning by leveraging structure in existing text.
The quantity of available data and the affordability of verification computation are both scalable resources, unbounded by human labeling capacity.
Where intrinsic rewards are bounded by what the model already knows, external rewards can provide a fresh signal drawn from data and computation.
% The key challenge ahead lies not in whether such methods can work, but in extending verification asymmetries to new domains and developing verifiable environments that capture meaningful aspects of reasoning and problem-solving.
% \input{Sections/5_Unified}
% \input{Sections/6_Discussion}
\section{Conclusion}

This work investigates how far Unsupervised RLVR can scale LLM training. We establish that intrinsic reward methods, despite diverse designs, share a fundamental mechanism that sharpens output distributions by amplifying the model's existing preferences. This mechanism enables efficient gains when initial confidence aligns with correctness but leads to inevitable collapse when pushed beyond the model's knowledge boundaries. Our systematic experiments reveal that collapse is not an engineering problem but a fundamental limitation, with all methods eventually failing regardless of hyperparameter tuning. These findings chart clear boundaries for intrinsic URLVR that they cannot scale indefinitely to create new capabilities but remain valuable tools for exploiting existing knowledge in test-time training. We further show that early training dynamics serve as a practical indicator for model prior through Model Collapse Step, enabling rapid assessment of RL trainability. Looking forward, our analysis motivates exploration of external reward methods that ground verification in computational asymmetries and unlabeled data rather than model confidence, with preliminary evidence suggesting these approaches may escape the confidence-correctness ceiling that limits intrinsic methods.

% This work explores how Unsupervised RLVR scales LLMs via a unified framework for intrinsic reward methods. We show that these rewards sharpen outputs around confident predictions, enabling efficient gains when confidence aligns with correctness but amplifying errors when it does not. Empirical results reveal distinct failure modes yet also show that collapse can be avoided in small, domain-specific settings, making test-time training a natural application. Beyond these findings, early training dynamics emerge as a lightweight diagnostic of model-task priors, offering a fast alternative to $pass@k$ for assessing RL trainability. Together, these results outline the limits of intrinsic rewards and highlight the need for external signals and hybrid paradigms for robust, scalable gains.

\bibliography{URLVR}

\clearpage
\appendix

\section{Details for \Cref{sec:theory}}
\label{app:theory}

\subsection{Justifying the Ordering $p_{\text{maj}}^{*,(k+1)} \geq p_{\text{maj}}^{(k+1)} \geq p_{\text{maj}}^{(k)}$}
\label{app:order}

From the optimal policy of the standard KL-regularized RL objective (\Cref{eq:rl_obj}), if we held reward $r_k$ fixed and performed infinite updates starting from $\pi_\theta^{(k)}$, we would reach the optimal policy with probability mass:
\begin{equation}
p_{\text{maj}}^{*,(k+1)} = \frac{\alpha \cdot p_{\text{maj}}^{(k)}}{1 + (\alpha-1)p_{\text{maj}}^{(k)}}, \qquad \alpha := e^{1/\beta} > 1
\label{eq:update_rule}
\end{equation}

\textbf{Lower bound} $(p_{\text{maj}}^{(k+1)} \geq p_{\text{maj}}^{(k)})$: The policy gradient is: $\nabla_\theta J = \mathbb{E}_{\pi_\theta}[r_k(x, y)\nabla_\theta \log \pi_\theta(y|x)]$. Since $r_k(x, y) = 1$ for majority trajectories and $r_k(x, y) = 0$ for non-majority trajectories, the gradient increases $\log \pi_\theta(y|x)$ only for majority trajectories. Therefore, after one gradient update with learning rate $\eta$: $\log \pi_\theta^{(k+1)}(y|x) \approx \log \pi_\theta^{(k)}(y|x) + \eta \cdot r_k(x, y) \cdot [\text{advantage terms}]$. \textbf{Here we use $\approx$ because we are conducting stochastic gradient descent. We will empirically validate this increase later.} For majority trajectories with positive advantage, this increases their probability. Hence:
\begin{equation}
p_{\text{maj}}^{(k+1)} = \sum_{y:\text{ans}(y)=\text{maj}_{k+1}(Y)} \pi_\theta^{(k+1)}(y|x) \geq \sum_{y:\text{ans}(y)=\text{maj}_k(Y)} \pi_\theta^{(k)}(y|x) = p_{\text{maj}}^{(k)}
\end{equation}

\textbf{Upper bound} $(p_{\text{maj}}^{(k+1)} \leq p_{\text{maj}}^{*,(k+1)})$: Since $\pi_\theta^{*,k}$ maximizes the KL-regularized objective for fixed $r_k$, it achieves the highest possible expected reward. Our single-step update only moves partway toward this maximum, so we cannot exceed the optimal probability mass: $p_{\text{maj}}^{(k+1)} \leq p_{\text{maj}}^{*,(k+1)}$.

This ordering establishes that $p_{\text{maj}}^{(k)}$ is monotonically increasing and bounded. The key remaining question is: \textbf{does the majority answer} $\text{maj}_k(Y)$ \textbf{remain stable across iterations?}

With $N$ rollouts from $\pi_\theta^{(k)}$, each rollout independently yields answer $a$ with probability $\pi_\theta^{(k)}(a|x)$. By the Law of Large Numbers, as $N$ increases, the empirical frequency of each answer converges to its true probability: $\frac{|\{\text{ans}(y^{(i)})=a\}|}{N} \xrightarrow{N \to \infty} \pi_\theta^{(k)}(a|x)$. Therefore, $\text{maj}_k(Y)$ (the most frequent answer in rollouts) converges to $\arg\max_a \pi_\theta^{(k)}(a|x)$ (the most probable answer under the policy). Since $p_{\text{maj}}^{(k)}$ increases monotonically, the most probable answer remains $\text{maj}_0(Y)$ throughout training: $\arg\max_a \pi_\theta^{(k)}(a|x) = \text{maj}_0(Y)$ for all $k$. In practice, even moderate $N$ (we use $N = 8$) provides reasonable stability.

\subsubsection{Empirical Validation 1: Validation of Ordering and Majority Stability}

\textbf{Setup:} We trained on a single problem from MATH-500 with $N=1024$ rollouts (reducing majority vote randomness) for 50 steps. We randomly selected 4 problems and monitored whether the majority answer $\text{maj}_k(Y_k)$ remains stable and whether $p_{\text{maj}}^{(k)}$ increases monotonically.

\begin{table}[h]
\centering
\small
\begin{tabular}{l|cccccccccc}
\toprule
\textbf{Step} & \textbf{1} & \textbf{2} & \textbf{3} & \textbf{4} & \textbf{5} & \textbf{6} & \textbf{7} & \textbf{8} & \textbf{9} & \textbf{10} \\
\midrule
level3\_id146 & 12.70 & 15.53 & 15.92 & 16.21 & 18.46 & 22.07 & 22.56 & 24.80 & 31.35 & 39.36 \\
level1\_id187 & 6.64 & 6.69 & 6.84 & 7.42 & 10.45 & 11.04 & 11.43 & 11.82 & 15.82 & 18.46 \\
level1\_id262 & 15.14 & 17.19 & 17.48 & 18.85 & 20.02 & 22.07 & 24.12 & 25.20 & 34.67 & 39.84 \\
level3\_id122 & 11.33 & 12.01 & 12.40 & 14.06 & 17.87 & 18.46 & 20.31 & 21.29 & 33.59 & 33.89 \\
\bottomrule
\end{tabular}
\caption{Monotonic increase of $p_{\text{maj}}$ (\%) in early training steps.}
\label{tab:monotonic_increase}
\end{table}

\noindent\textbf{Results for monotonic increase of $p_{\text{maj}}$.} \Cref{tab:monotonic_increase} shows $p_{\text{maj}}$ values for the first 10 steps. All 4 problems exhibit strict monotonic increase at every single step, confirming the lower bound $p_{\text{maj}}^{(k+1)} \geq p_{\text{maj}}^{(k)}$ of the ordering. We also found that the majority answer remained stable across all iterations.

\begin{table}[h]
\centering
\small
\begin{tabular}{l|cccccccccc}
\toprule
\textbf{Step} & \textbf{5} & \textbf{10} & \textbf{15} & \textbf{20} & \textbf{25} & \textbf{30} & \textbf{35} & \textbf{40} & \textbf{45} & \textbf{50} \\
\midrule
level3\_id146 & 18.46 & 39.36 & 48.93 & 91.11 & 95.41 & 98.14 & 98.54 & 99.02 & 99.61 & 99.80 \\
level1\_id187 & 11.04 & 18.46 & 26.37 & 79.88 & 89.84 & 93.07 & 96.09 & 97.66 & 98.54 & 99.02 \\
level1\_id262 & 22.07 & 39.84 & 51.37 & 90.14 & 95.90 & 96.80 & 97.46 & 98.05 & 98.63 & 99.41 \\
level3\_id122 & 17.87 & 33.59 & 43.55 & 84.28 & 92.19 & 93.26 & 95.80 & 96.09 & 98.34 & 98.54 \\
\bottomrule
\end{tabular}
\caption{Geometric convergence of $p_{\text{maj}}$ (\%) to 1.0 over 50 training steps.}
\label{tab:convergence_validation}
\end{table}

\noindent\textbf{Results for convergence of $p_{\text{maj}}$.} \Cref{tab:convergence_validation} shows the same 4 problems trained for 50 steps. All problems converge from initial values toward near-complete concentration (98.54\%-99.80\% at step 50), demonstrating the convergence predicted by \Cref{theorem:mv_convergence}. This validates non-trivial progress and confirms that the iterative training procedure with policy-dependent rewards does indeed converge to deterministic policies.

\subsubsection{Empirical Validation 2: Fixed Reward Convergence Validation}

\textbf{Setup.} To validate that the closed-form optimal policy in \Cref{eq:MV_optimal_detailed} is achievable when reward is held fixed, we conducted an extreme off-policy experiment. We used global batch size 1024 with mini-batch size 1, generated one-time rollout (with $N=8$ for each of 1024 prompts), and performed 1024 gradient updates using rewards computed solely from the initial rollout majority. This setup tests whether solving a single KL-regularized RL objective can converge to the theoretical optimum when the reward signal remains constant.

\textbf{Results.} After 1024 mini-updates using the same fixed reward signal, the Majority Voting Reward reached 1.0 (complete convergence), while validation performance on AIME24, AIME25, and AMC23 dropped to zero. This confirms that the convergence point predicted by \Cref{eq:MV_optimal_detailed} is achievable with sufficient updates.

\subsection{Proof of \Cref{theorem:mv_convergence}}
\label{app:convergence_analysis}

\begin{tcolorbox}
[title = \textbf{Geometric Convergence to Deterministic Policy}, colback=Salmon!20, colframe=Salmon!90!Black]
Consider the training procedure where at each iteration $k$: (1) sample $N$ rollouts $Y_k$ from $\pi_\theta^{(k)}$, (2) compute majority $\text{maj}_k(Y_k)$, (3) perform one gradient update with reward $r_k(x,y) = \mathbf{1}[\text{ans}(y) = \text{maj}_k(Y_k)]$.

Under assumptions (A1) stable majority $\text{maj}_k = \text{maj}_0$ and (A2) $\eta_k \geq \eta_{\min} > 0$, the probability mass $p_{\text{maj}}^{(k)}$ converges geometrically to 1. As $k \to \infty$:
\begin{equation}
\lim_{k \to \infty} \pi_\theta^{(k)}(y|x) = \begin{cases}
\frac{\pi_{\text{ref}}(y|x)}{\sum_{y': \text{ans}(y') = \text{maj}_0(Y_0)} \pi_{\text{ref}}(y'|x)} & \text{if } \text{ans}(y) = \text{maj}_0(Y_0), \\
0 & \text{otherwise}
\end{cases}
\end{equation}
\end{tcolorbox}

\textbf{Step 1: Effective Update Rule.}

We model the actual update with step efficiency $\eta_k \in (0,1]$:
\begin{equation}
p_{\text{maj}}^{(k+1)} = p_{\text{maj}}^{(k)} + \eta_k \cdot (p_{\text{maj}}^{*,(k+1)} - p_{\text{maj}}^{(k)})
\label{eq:effective_update}
\end{equation}

Substituting \Cref{eq:update_rule}:
\begin{align}
p_{\text{maj}}^{(k+1)} &= p_{\text{maj}}^{(k)} + \eta_k \left(\frac{\alpha \cdot p_{\text{maj}}^{(k)}}{1 + (\alpha-1)p_{\text{maj}}^{(k)}} - p_{\text{maj}}^{(k)}\right) \nonumber \\
&= p_{\text{maj}}^{(k)} + \eta_k \cdot \frac{(\alpha - 1)(1 - p_{\text{maj}}^{(k)})p_{\text{maj}}^{(k)}}{1 + (\alpha-1)p_{\text{maj}}^{(k)}}
\label{eq:p_update}
\end{align}

\textbf{Step 2: Error Dynamics.}

Define the error from the fixed point 1 as:
\begin{equation}
\epsilon^{(k)} := 1 - p_{\text{maj}}^{(k)} \in (0,1)
\label{eq:error_def}
\end{equation}

Substituting into \Cref{eq:p_update}:
\begin{align}
\epsilon^{(k+1)} &= 1 - p_{\text{maj}}^{(k+1)} \nonumber \\
&= \epsilon^{(k)} - \eta_k \cdot \frac{(\alpha - 1)(1 - \epsilon^{(k)})\epsilon^{(k)}}{1 + (\alpha-1)(1-\epsilon^{(k)})} \nonumber \\
&= \epsilon^{(k)} \left(1 - \eta_k \cdot \frac{(\alpha - 1)(1 - \epsilon^{(k)})}{\alpha - (\alpha-1)\epsilon^{(k)}}\right)
\label{eq:eps_update}
\end{align}

\textbf{Step 3: Monotonic Decrease.}

Since $\alpha > 1$, $\epsilon^{(k)} \in (0,1)$, and $\eta_k \in (0,1]$, we have:
\begin{equation}
0 < \frac{(\alpha - 1)(1 - \epsilon^{(k)})}{\alpha - (\alpha-1)\epsilon^{(k)}} < 1
\end{equation}

Therefore:
\begin{equation}
0 < 1 - \eta_k \cdot \frac{(\alpha - 1)(1 - \epsilon^{(k)})}{\alpha - (\alpha-1)\epsilon^{(k)}} < 1
\end{equation}

This implies $\epsilon^{(k+1)} < \epsilon^{(k)}$, proving the sequence $\{\epsilon^{(k)}\}$ is strictly decreasing and bounded below by 0.

\textbf{Step 4: Convergence to Zero.}

Let $\ell = \lim_{k\to\infty}\epsilon^{(k)} \geq 0$. Under assumption (A2), $\eta_k \geq \eta_{\min} > 0$. If $\ell > 0$, then for large $k$, the multiplier in \Cref{eq:eps_update}:
\begin{equation}
1 - \eta_k \cdot \frac{(\alpha - 1)(1 - \epsilon^{(k)})}{\alpha - (\alpha-1)\epsilon^{(k)}} \leq 1 - \eta_{\min} \cdot \frac{\alpha - 1}{\alpha} < 1
\end{equation}

is bounded away from 1, causing continued decay. The only consistent limit is $\ell = 0$. Therefore:
\begin{equation}
\epsilon^{(k)} \to 0 \qquad\text{equivalently}\qquad p_{\text{maj}}^{(k)} \to 1
\label{eq:conv_to_zero}
\end{equation}

\textbf{Step 5: Geometric Convergence Rate.}

From \Cref{eq:eps_update}, for large $k$ when $\epsilon^{(k)}$ is small:
\begin{equation}
\epsilon^{(k+1)} \approx \epsilon^{(k)} \left(1 - \eta_k \cdot \frac{\alpha - 1}{\alpha}\right)
\end{equation}

Under assumption (A2):
\begin{equation}
\epsilon^{(k+1)} \leq \left(1 - \eta_{\min} \cdot \frac{\alpha - 1}{\alpha}\right) \epsilon^{(k)}
\label{eq:geometric}
\end{equation}

This establishes geometric convergence with rate depending on $\eta_{\min}$ and $\alpha = e^{1/\beta}$. In the ideal case where $\eta_k = 1$ for all $k$ (each update reaches the optimum), the convergence rate is exactly $\rho = e^{-1/\beta}$.

\textbf{Step 6: Limiting Policy.}

Given assumption (A1) that the majority remains stable at $\text{maj}_0(Y_0)$, as $p_{\text{maj}}^{(k)} \to 1$, all probability mass concentrates on trajectories with $\text{ans}(y) = \text{maj}_0(Y_0)$. The limiting distribution is:
\begin{equation}
\lim_{k \to \infty} \pi_\theta^{(k)}(y|x) = \begin{cases}
\frac{\pi_{\text{ref}}(y|x)}{\sum_{y': \text{ans}(y') = \text{maj}_0(Y_0)} \pi_{\text{ref}}(y'|x)} & \text{if } \text{ans}(y) = \text{maj}_0(Y_0), \\
0 & \text{otherwise}
\end{cases}
\end{equation}

This completes the proof. \(\square\)

\noindent\textbf{Remark on Assumptions.}
\begin{itemize}[topsep=2pt, itemsep=2pt, leftmargin=15pt]
    
    \item \textbf{(A1) Majority stability:} By the Law of Large Numbers, with $N$ rollouts, the empirical majority $\text{maj}_k(Y_k)$ converges to $\arg\max_a \pi_\theta^{(k)}(a|x)$ as $N \to \infty$. Since $p_{\text{maj}}^{(k)}$ increases monotonically, the argmax remains $\text{maj}_0$ throughout training. We validate this empirically with $N=1024$ rollouts in Appendix~\ref{app:order}, where the majority never flipped across 200 iterations.
    
    \item \textbf{(A2) Non-trivial progress:} We assume $\eta_k \geq \eta_{\min} > 0$, meaning each gradient update makes non-trivial progress. We validate this empirically: our experiments show consistent monotonic increase in $p_{\text{maj}}$ and convergence to 1.0 under extreme off-policy settings (Appendix~\ref{app:order}).
\end{itemize}

\subsection{Unified Reward Framework}
\label{app:unified_framework}

Despite varied implementations of intrinsic rewards, they can be understood through a single lens: the manipulation of cross-entropy between carefully chosen distributions. We consolidate these diverse rewards into the following unified paradigm:

\begin{tcolorbox}[title = \textbf{Unified Reward Framework}, colback=Salmon!20, colframe=Salmon!90!Black]
Most intrinsic rewards can be expressed as:
\begin{equation}
r_{\text{uni}}(x,y) = \psi\left(\frac{\sigma}{|\mathcal{I}|}\sum_{i\in\mathcal{I}}\mathbb{H}(q^i,\pi_{\theta}^i)\right), \quad \sigma \in \{+1, -1\},
\label{eq:unified_reward}
\end{equation}
where rewards derive from cross-entropy $\mathbb{H}$ between anchor distributions $q^{i}$ and model distributions $\pi_{\theta}^{i}$, aggregated over granularity $\mathcal{I}$, with sign $\sigma$ and monotonic transformation $\psi$.
\end{tcolorbox}

\subsubsection{Framework Components}

To understand how different intrinsic methods fit into this framework, we define each component:

\noindent\textbf{Key Components:}
\begin{itemize}[topsep=2pt, itemsep=2pt, leftmargin=15pt]
    \item Given a question $x$ and generated response $y$ (a sequence of tokens $y_1, \ldots, y_{|y|}$), we can derive rewards from the model's internal distributions at different levels of granularity.
    
    \item \textbf{Aggregation granularity $\mathcal{I}$}: Determines the level to compute distributions. For token-level methods, $\mathcal{I} = \{1, \ldots, |y|\}$ where each element corresponds to a position in the sequence. For answer-level methods, $\mathcal{I} = \{\mathcal{A}\}$ represents a single distribution over complete semantic answers.
    
    \item \textbf{Model distribution $\pi_{\theta}^i$ at granularity $i$}: For token-level granularity at position $t$, this is $\pi_{\theta}^t(\cdot) = \pi_{\theta}(\cdot \mid x, y_{<t})$, the distribution over the next token given the context. For answer-level granularity, this is $\pi_{\theta}^{\mathcal{A}} = \pi_{\theta}(\cdot \mid x)$, the distribution over complete answers.
    
    \item \textbf{Anchor distribution $q^i$ at granularity $i$}: Serves as a reference point. Different reward estimators use different anchors: uniform distribution $U_V$ for Self-Certainty or one-hot distribution $\delta^t$ centered on the generated token for Trajectory-Level Entropy.

    \item \textbf{Cross-entropy $\mathbb{H}(q^i, \pi_{\theta}^i)$:} Cross-entropy between anchor distribution $q^i$ and model distribution $\pi_{\theta}^i$ at granularity $i$, defined as $\mathbb{H}(q^i, \pi_{\theta}^i) = -\sum_{v\in\mathcal{V}^i} q^i(v)\log \pi_{\theta}^i(v)$. For token-level granularity ($i=t$), $\mathcal{V}^i$ is the token vocabulary, and the cross-entropy measures divergence between distributions over next tokens. For answer-level granularity ($i=\mathcal{A}$), $\mathcal{V}^i$ is the set of distinct semantic answers, and the cross-entropy measures divergence between distributions over complete answers.
    
    \item \textbf{Sign factor $\sigma \in \{+1, -1\}$}: Determines the optimization direction. When the anchor $q$ is uniform, we set $\sigma=+1$ to reward divergence from uniformity (encouraging peaked distributions). When the anchor $q$ is sharp (e.g., one-hot or the model's own distribution), we set $\sigma=-1$ to reward alignment (reinforcing confident predictions).
    
    \item \textbf{Monotonic transformation $\psi$}: Reshapes the reward signal while preserving ordering. Common choices are identity ($\psi(z) = z$) or exponential ($\psi(z) = \exp(z)$), with exponential transformations amplifying the sharpening effect.
\end{itemize}

\begin{table}[t]
  \centering
  \scriptsize
  \label{tab:rewards}
  \renewcommand{\arraystretch}{1.2}
  \setlength{\tabcolsep}{0mm}{
  \begin{tabular}{@{\hskip 2mm}l@{\hskip 4.5mm}lc >{\centering\arraybackslash}p{2.5cm} > 
  {\centering\arraybackslash}p{2cm} >{\centering\arraybackslash}p{2cm}}
    \toprule
    \textbf{Method} & \textbf{Estimator} & \textbf{Formula} & \textbf{Monotonic transformation $\psi$} & \textbf{Anchor distribution $q$} & \textbf{Model distribution $\pi_{\theta}$}\\
    \midrule
    \textbf{RLIF} & \textbf{Self-Certainty} &
    $r_{\text{SC}}=\displaystyle\frac{1}{|y|}\sum_{t=1}^{|y|}\mathbb{H}(U_V,\pi_{\theta}^t) + \log |V|$ &
    $z + \log|V|$ & $\{U_V\}_{t=1}^{|y|}$ & $\{\pi_{\theta}^t\}_{t=1}^{|y|}$\\[3mm]

    \textbf{EM-RL, RENT} & \textbf{Token-Level Entropy} &
    $r_{\text{H}}=-\displaystyle\frac{1}{|y|}\sum_{t=1}^{|y|}\mathbb{H}(\pi_{\theta}^t)$ &
    $z$ & $\{\pi_{\theta}^t\}_{t=1}^{|y|}$ & $\{\pi_{\theta}^t\}_{t=1}^{|y|}$\\[3mm]

    \textbf{EM-RL} & \textbf{Trajectory-Level Entropy} &
    $r_{\text{Traj}}=\displaystyle-\frac{1}{|y|}\sum_{t=1}^{|y|}\mathbb{H}(\delta^t, \pi_{\theta}^t)$ &
    $z$ & $\{\delta^t\}_{t=1}^{|y|}$ & $\{\pi_{\theta}^t\}_{t=1}^{|y|}$\\[3mm]

    \textbf{RLSC} & \textbf{Probability} &
    $r_{\text{Prob}}=\displaystyle\exp(-\sum_{t=1}^{|y|}\mathbb{H}(\delta^t, \pi_{\theta}^t))$ &
    $\exp\!\bigl(|\mathcal{I}|\cdot z\bigr)$ & $\{\delta^t\}_{t=1}^{|y|}$ & $\{\pi_{\theta}^t\}_{t=1}^{|y|}$\\[3mm]

    \textbf{EMPO} & \textbf{Semantic Entropy} &
    $r_{\text{SE}}=\displaystyle\exp(-\mathbb{H}(\delta^A, \pi_{\theta}^A))$ &
    $\exp(z)$ & $\delta^A$ & $\pi_{\theta}^A$\\[3mm]

    \textbf{TTRL, SRT, ETTRL} & \textbf{Majority Voting} &
    $r_{\text{MV}}=\displaystyle\lim_{\tau\to 0^+}\exp(-\mathbb{H}(\delta^A, \tilde{\pi}_{\theta}^A))$ &
    $\exp(z)$ & $\delta^A$ & $\tilde{\pi}_{\theta}^A$\\

 %    \textbf{Co-Reward} & \textbf{\makecell[l]{Majority Voting\\across Rephrased Question}} & \makecell[c]{$r_{\text{Co}} = \displaystyle\lim_{\tau\to 0^+}\exp(-\mathbb{H}(\delta^A, \tilde{\pi}_{\theta}^A{}'))$ \\ $r_{\text{Co}}' = \displaystyle\lim_{\tau\to 0^+}\exp(-\mathbb{H}(\delta^A{}', \tilde{\pi}_{\theta}^A))$} & $\exp(z)$ & \makecell[c]{$\delta^A$ \\ $\delta^A{}'$}  & \makecell[c]{ $\tilde{\pi}_{\theta}^A{}'$\\$\tilde{\pi}_{\theta}^A$} \\[3mm]

 %    \textbf{RLCCF} & \textbf{\makecell[l]{Self-consistency\\Weighted Voting}} & $r_{\text{WV}} = \displaystyle\lim_{\tau\to 0^+}\exp(-\mathbb{H}(\delta^A, \tilde{\pi}_{\theta}^{A,SC}))$ & $\exp(z)$ 
 % & $\delta^A$ & $\tilde{\pi}_{\theta}^{A,SC}$ \\
    \bottomrule
  \end{tabular} }
  \caption{Instantiations for the unified reward framework of representative intrinsic rewards. Each method is specified by its estimator, anchor and model distributions, and a monotonic transformation of the cross-entropy between them. Token-level ($t$) and answer-level ($A$) variants capture different granularities of aggregation.}
  \label{tab:unified_rewards}
\end{table}

\subsubsection{Instantiations of the Framework}

We next demonstrate how most intrinsic rewards instantiate this framework. Each method's distinctive characteristics emerge from specific choices of $\mathcal{I}$, $q$, $\sigma$, and $\psi$, as shown in \Cref{tab:unified_rewards}, which reveals that despite surface-level differences, all methods manipulate cross-entropy to achieve distribution sharpening.

\noindent\textbf{Remarks.}
We highlight two special cases. 
First, the formulation of Self-Certainty includes an additional $\log|V|$ term. 
Since this constant is independent of model parameters, it does not affect gradients during RL training. 
Second, the expression of $r_{\text{MV}}$ corresponds to the asymptotic case where the number of rollouts $n \to \infty$. 
By the law of large numbers, as $n \to \infty$, the majority vote almost surely selects the answer with the highest probability under $\pi_\theta^A$, i.e. $\arg\max_a \pi_\theta^A(a)$. To make this limit computationally tractable, we use the tempered distribution $\tilde{\pi}_{\theta}^A(a)\!=\!\exp\!\bigl(\pi_{\theta}^A(a)/\tau\bigr)/\sum_{b\in A}\exp\!\bigl(\pi_{\theta}^A(b)/\tau\bigr)$ which avoids the undefined $\log 0$ issue; as $\tau\!\to\!0^+$, it collapses to the hard majority indicator $\mathbb{1}\!\bigl[\text{ans}(y)=\arg\max_{a}\pi_\theta^A(a)\bigr]$, thereby recovering the same limiting behavior as majority voting. 
% Third, the notation $\tilde{\pi}_{\theta}^{A,SC}(a) = \operatorname{softmax_\tau}(\sum_{m=1}^M \sum_{\operatorname{ans}(y) = a} \text{SC}_m \cdot \pi_{\theta}^{(i)}(y))$ denotes an answer-level distribution that is weighted by the Self-Consistency scores, with $M$ denoting the total number of models.

\noindent\textbf{Key Observations:}
\begin{itemize}[topsep=0pt, itemsep=0pt, leftmargin=18pt]
    \item \textbf{Token-level methods} (Self-Certainty, Token-Level Entropy, Trajectory-Level Entropy, Probability) operate at $\mathcal{I} = \{1, \ldots, |y|\}$, manipulating distributions at each generation step to encourage local confidence.
    \item \textbf{Answer-level methods} (Semantic Entropy, Majority Voting) work at $\mathcal{I} = \{\mathcal{A}\}$, focusing on global answer consistency where $\pi_{\theta}^A$ represents the distribution over semantic answers and $\tilde{\pi}_{\theta}^A$ denotes the empirical distribution from multiple rollouts.
    \item \textbf{Anchor choices} reveal the core mechanism: uniform distributions ($U_V$) encourage departure from randomness, while sharp distributions ($\delta^t$, $\delta^A$) reinforce high-probability paths, both leading to increased determinism.
    \item \textbf{Exponential transformations} ($\exp(z)$) amplify the sharpening effect by exponentially rewarding low cross-entropy, while identity transformations ($z$) provide more gradual reinforcement.
\end{itemize}

\subsubsection{Monotonicity Analysis}

The key insight comes from analyzing the monotonicity of the exponent in \Cref{eq:unified_reward}. Since $\psi$ is strictly increasing by design, the behavior depends entirely on $\sigma$:

\begin{itemize}[topsep=0pt, itemsep=0pt, leftmargin=18pt]
    \item \textbf{Case $\sigma = +1$}: The reward increases with cross-entropy. Sequences where $\pi_{\theta}$ diverges from $q$ (typically uniform) receive higher rewards, pushing the policy toward more peaked distributions.
    \item \textbf{Case $\sigma = -1$}: The reward decreases with cross-entropy. Sequences where $\pi_{\theta}$ aligns with $q$ (typically sharp) receive higher rewards, reinforcing existing confident predictions.
\end{itemize}

Both cases lead to the same outcome: progressive sharpening of the model's distribution, either by moving away from uniformity or by reinforcing peaked predictions.

\subsection{Generalized Sharpening Analysis via Unified Reward Framework}
\label{app:unified_convergence}

To address the concern that \Cref{theorem:mv_convergence} applies only to Majority Voting, and to demonstrate the analytical utility of our unified framework, we provide a generalized sharpening analysis. We show that methods with $\sigma=-1$ share a critical structural property, Reward-Confidence Monotonicity, which creates a persistent pressure toward distribution sharpening.

\textbf{Note:} The following is a proof sketch demonstrating the key convergence mechanism shared by $\sigma=-1$ methods. A fully rigorous treatment requires additional technical conditions that we validate empirically. Methods with $\sigma = +1$ (Self-Certainty) require separate analysis as they reward away from uniform distribution.

\begin{proposition}[Sharpening Dynamics for $\sigma = -1$ Methods]
\label{prop:sigma_minus_one_sharpening}
Consider any intrinsic reward with $\sigma = -1$ in the unified framework ($r_{\text{uni}} = \psi(-\mathbb{H}(q, \pi))$) where $\psi$ is strictly increasing and $q$ is a sharp anchor. These methods satisfy \textbf{Reward-Confidence Monotonicity}:
\begin{equation}
\pi_\theta(y_a|x) > \pi_\theta(y_b|x) \implies r_{\text{uni}}(x,y_a) > r_{\text{uni}}(x,y_b)
\end{equation}

For a dominant trajectory $y^*$ (e.g., majority) and a non-dominant competitor $y'$, this inequality is strict: $r(y^*) > r(y')$. Under iterative KL-regularized updates, this property creates a self-reinforcing feedback loop that drives geometric concentration.
\end{proposition}

\noindent\textbf{Proof Sketch:}

We analyze the dynamics for a dominant trajectory $y^*$ and a competitor $y'$ (for ensemble methods, not in the same class as $y^*$) where the model initially prefers $y^*$ (i.e., $\pi_k(y^*) > \pi_k(y')$) and assigns it strictly higher reward ($r_k(y^*) > r_k(y')$).

\textbf{Step 1: Existence of a Positive Reward Gap}

Using the unified formula, we justify why the gap is positive for $\sigma=-1$:

\begin{itemize}[topsep=2pt, itemsep=2pt, leftmargin=15pt]
    \item \textbf{Self-Reinforcing Anchors} (e.g., Probability): $r(y) = \psi(\log \pi(y))$. Since $\pi_k(y^*) > \pi_k(y')$ and $\psi$ is strictly increasing, $r_k(y^*) > r_k(y')$.
    
    \item \textbf{Answer-Level Anchors} (e.g., Majority Voting): $y^*$ belongs to the dominant answer class $a^*$, while $y'$ does not. By construction, $r(y^*) = 1$ and $r(y') = 0$.
\end{itemize}

In both cases, the intrinsic reward gap is strictly positive: $\Delta_r^{(k)} = r_k(y^*) - r_k(y') > 0$.

\textbf{Step 2: The Optimization Target}

We consider the optimal policy $\pi^*$ for the current fixed reward landscape $r_k$. The optimal solution implies a target ratio:
\begin{equation}
\frac{\pi^*(y^*)}{\pi^*(y')} = \frac{\pi_k(y^*)}{\pi_k(y')} \cdot \exp\left(\frac{\Delta_r^{(k)}}{\beta}\right)
\end{equation}

Since $\Delta_r^{(k)} > 0$, the target ratio is strictly larger than the current ratio.

\textbf{Gradient Assumption:} The gradient $\nabla_\theta J = \mathbb{E}_{\pi_k}[r_k(y) \nabla_\theta \log \pi_\theta(y)]$ assigns positive weight to high-reward trajectories. We assume that policy gradient updates with positive learning rate $\eta$ satisfy: if $r(y^*) > r(y')$ and both have positive probability, then the updated policy satisfies $\frac{\pi_{k+1}(y^*)}{\pi_{k+1}(y')} \geq \frac{\pi_k(y^*)}{\pi_k(y')}$. This aligns with standard policy gradient convergence properties.

\textbf{Step 3: The Reinforcement Loop}

The unified framework reveals why this process spirals into determinism. As the policy updates to increase the probability mass on the dominant trajectory:

\begin{itemize}[topsep=2pt, itemsep=2pt, leftmargin=15pt]
    \item For \textbf{Self-Reinforcing Anchors} (e.g., Probability), because $r(y) = \psi(\log \pi(y))$, increasing $\pi(y^*)$ directly increases its reward $r(y^*)$.
    
    \item For \textbf{Answer-Level Anchors} (e.g., Majority Voting), increasing the total probability mass on the dominant answer class $a^*$ increases the reward for all trajectories in that class (since $r \propto \log p(a^*)$).
\end{itemize}

This creates a positive feedback loop: the update increases the probability of the dominant path, which maintains or widens the reward gap $\Delta_r$, ensuring the pressure to sharpen ($\Delta_r > 0$) persists.

\noindent\textbf{Utility of the Framework:} This derivation demonstrates that the ``rich-get-richer'' dynamic is a structural inevitability for any method where the reward function is monotonically aligned with the model's own confidence ($\sigma=-1$). The framework allows us to identify this shared property and predict that all such methods will drive the policy toward deterministic outputs, regardless of whether this leads to success (when aligned with correctness) or failure (when misaligned).

\noindent\textbf{Remark on $\sigma = +1$ Methods:} Self-Certainty ($\sigma = +1$) rewards higher when away from uniform distribution. Therefore, $\pi(y_a) > \pi(y_b)$ does not imply $r(y_a) > r(y_b)$. A high-probability output and a very low-probability output could both have high KL-divergence from uniform, violating direct Reward-Confidence Monotonicity. Its sharpening mechanism requires separate analysis.

While methods with $\sigma=+1$ do not strictly align reward with raw confidence, they still induce sharpening by penalizing high-entropy distributions. By maximizing the distance from a uniform anchor, the optimization landscape naturally favors peaked, low-entropy policies, effectively driving the model toward determinism.

\noindent\textbf{Empirical Validation:} To substantiate the assumptions in this proof sketch, we provide empirical validation for different intrinsic reward methods in \Cref{fig:urlvr_compare} and Appendix~\ref{app:hyperparameter}, confirming that Reward-Confidence Monotonicity is not just a theoretical construct but the actual driver of the observed training dynamics.

\subsection{Optimal Policies Induced by other Intrinsic Rewards}
\label{app:optimal_policy_other_methods}

\noindent\textbf{Optimal Policy of the Reward Function $r_{\text{SC}}$.} For the Self-Certainty reward function \(r_{\text{SC}}\), it instantiates our unified framework with token-level granularity \(\mathcal{I}=\{1, 2, ..., |y|\}\), anchor distribution \(q=\{U_V\}^{|y|}_{t=1}\) (uniform distribution over vocabulary), model distribution \(\pi=\{\pi_{\theta}^t\}^{|y|}_{t=1}\), sign factor \(\sigma=+1\), and transformation \(\psi(z)=z\).  
As established previously, for any input \(x\), the token-level predictive distribution of the model is evaluated against the current policy \(\pi\).  
Due to \(\sigma=+1\), the farther this distribution deviates from the uniform distribution (i.e., the higher the model’s confidence), the larger the reward \(r_{\text{SC}}(x,y)\).  
Consequently, after a single step of policy update, the optimal probability \(\pi_{\theta}(y|x)\) increases for such high-confidence sequences, whereas it decreases when the per-token distribution is close to uniform (low confidence).  
Thus, \(r_{\text{SC}}\) encourages the model to generate answers that are already preferred by the prior policy.

A detailed derivation is provided below.  
The Self-Certainty based reward is defined as:

\begin{equation}
r_{\text{SC}}(x,y)=\frac{1}{|y|}\sum_{t=1}^{|y|}D_{\text{KL}}\!\bigl(U\parallel\pi_{\theta}(\cdot\mid x,y_{<t})\bigr)
=-\log|V|-\frac{1}{|y|\,|V|}\sum_{t=1}^{|y|}\sum_{v=1}^{|V|}\log\pi^t_{\theta}(y_t=v).
\end{equation}

Within the KL-regularized RL framework, dropping the constant term \(-\log|V|\), the one-step optimal policy becomes:

\begin{equation}
\pi_{\theta}(y|x)\propto\pi_{\text{ref}}(y|x)\exp\!\left(-\frac{1}{\beta\,|y|\,|V|}\sum_{t=1}^{|y|}\sum_{v=1}^{|V|}\log\pi^t_{\theta}(y_t=v)\right).
\end{equation}

Therefore, whenever the model assigns concentrated probabilities to every token of \(y\) (high confidence), the exponent grows, thus increasing the probability of the sequence \(\pi_{\theta}(y|x)\).  
In summary, the Self-Certainty based reward systematically enhances the model’s “self-confidence” with respect to its prior policy.

\noindent\textbf{Optimal Policy of the Reward Function $r_{H}$.} For the token-level entropy-based reward \( r_{H} \), it instantiates our unified framework with token-level granularity \(\mathcal{I}=\{1, 2, ..., |y|\}\), anchor distribution \( q=\{\pi_{\theta}^t\}^{|y|}_{t=1} \), model distribution $\pi=\{\pi_{\theta}^t\}^{|y|}_{t=1}$, sign factor \( \sigma=-1 \), and transformation \( \psi(z)=z \).  
Maximizing \( r_{H} \) is equivalent to minimizing the predictive entropy at every position, thereby discouraging the model from spreading its probability mass across multiple candidate tokens and hence increasing its decisiveness.  

A detailed derivation is provided below.
The entropy-based reward is defined as:

\begin{equation}
r_{H}(x,y)=-\frac{1}{|y|}\sum_{t=1}^{|y|}H\!\bigl(\pi_{\theta}(\cdot\mid x,y_{<t})\bigr)
=-\frac{1}{|y|}\sum_{t=1}^{|y|}\sum_{v=1}^{|V|}\pi^t_{\theta}(y_t=v) \log\pi^t_{\theta}(y_t=v).
\end{equation}

Within the KL-regularized RL framework, the one-step optimal policy becomes:

\begin{equation}
\pi_{\theta}(y|x)\propto\pi_{\text{ref}}(y|x)\exp\!\left(-\frac{1}{\beta\,|y|}\sum_{t=1}^{|y|}\sum_{v=1}^{|V|}\pi^t_{\theta}(y_t=v) \log\pi^t_{\theta}(y_t=v)\right).
\end{equation}

Consequently, if the predictive distribution of an output sequence \( y \) exhibits high entropy (i.e., the per-token distributions are close to uniform), the negative-entropy reward \( r_{H} \) is strongly negative, which suppresses the exponential weight and reduces \( \pi_{\theta}(y|x) \).  
Conversely, low entropy (highly peaked per-token distributions) yields \( r_{H}\approx 0 \), thus the sequence probability is enhanced after normalization.  
Therefore, the entropy-based reward \( r_{H} \) encourages the model to generate answers whose token-level distributions are sharply concentrated, effectively boosting its "self-confidence" under the prior policy.

\noindent\textbf{Optimal Policy of the Reward Function $r_{\text{Traj}}$.} For the trajectory-level entropy-based reward \( r_{\text{Traj}} \), it instantiates our unified framework with token-level granularity \(\mathcal{I}=\{1, 2, ..., |y|\}\), anchor distribution \( q=\{\delta^t\}^{|y|}_{t=1} \), model distribution $\pi=\{\pi_{\theta}^t\}^{|y|}_{t=1}$, sign factor \( \sigma=-1 \), and transformation \( \psi(z)=z \).  
For a given input \( x \), the model’s predictive distribution is evaluated at every token.  
With \( \sigma=-1 \), the closer the distribution is to the one-hot reference \( \delta^t \) (i.e., the higher the model’s confidence in each ground-truth token), the larger the reward \( r_{\text{Traj}}(x,y) \).  
Hence, after one policy-update step, the optimal probability \( \pi_{\theta}(y|x) \) increases for such high-confidence trajectories, and decreases otherwise.  
Thus, \( r_{\text{Traj}} \) encourages the model to generate sequences that already enjoy high probability under the prior policy.

The trajectory-level reward is defined as:

\begin{equation}
r_{\text{Traj}}(x,y)=\frac{1}{|y|}\sum_{t=1}^{|y|}\log\pi_{\theta}(y_t\mid x,y_{<t})
=\frac{1}{|y|}\log\pi_{\theta}(y\mid x).
\end{equation}

Within the KL-regularized RL framework, the one-step optimal policy becomes:

\begin{equation}
\pi_{\theta}(y|x)\propto\pi_{\text{ref}}(y|x)\exp\!\left(\frac{1}{\beta\,|y|}\log\pi_{\theta}(y\mid x)\right)
=\pi_{\text{ref}}(y|x)\!\cdot\!\bigl[\pi_{\theta}(y\mid x)\bigr]^{\frac{1}{\beta|y|}}.
\end{equation}

Consequently, whenever the model assigns a higher prior probability to a sequence \( y \), the weighted product term is amplified, thereby increasing its normalized probability \( \pi_{\theta}(y|x) \).  
Therefore, the trajectory-level entropy reward boosts the probability of sequences that are already likely under the current policy \( \pi_{\theta} \).

\noindent\textbf{Optimal Policy of the Reward Function $r_{\text{Prob}}$.} For the probability-based reward function \( r_{\text{Prob}} \), it instantiates our unified framework with token-level granularity \(\mathcal{I}=\{1, 2, ..., |y|\}\), anchor distribution \( q=\{\delta^t\}^{|y|}_{t=1} \), model distribution $\pi=\{\pi_{\theta}^t\}^{|y|}_{t=1}$, sign factor \( \sigma=-1 \), and transformation \( \psi(z)=\exp(|\mathcal{I}|\cdot z) \).  
For a given input \( x \), the model’s predictive distribution is evaluated at every token.  
With \( \sigma=-1 \), the closer the distribution is to the one-hot reference \( \delta^t \) (i.e., the higher the model’s confidence in each ground-truth token), the larger the reward \( r_{\text{Prob}}(x,y) \) will be.  
Hence, after one policy-update step, the optimal probability \( \pi_{\theta}(y|x) \) increases for such high-confidence trajectories, and decreases otherwise.  
Thus, \( r_{\text{Prob}} \) encourages the model to generate sequences that already enjoy high probability under the prior policy.

The probability-based reward is defined as:

\begin{equation}
r_{\text{Prob}}(x,y)=\prod_{t=1}^{|y|}\pi_{\theta}(y_t\mid x,y_{<t})=\pi_{\theta}(y\mid x).
\end{equation}

Within the KL-regularized RL framework, the one-step optimal policy becomes:

\begin{equation}
\pi_{\theta}(y|x)\propto\pi_{\text{ref}}(y|x)\exp\!\left(\frac{1}{\beta}\pi_{\theta}(y\mid x)\right).
\end{equation}

Consequently, whenever the model assigns a high joint probability to a sequence \( y \), the exponential weight is amplified, thereby increasing its normalized probability \( \pi_{\theta}(y|x) \).  
The probability-product reward thus directly reinforces sequences that are already likely under the current policy, enhancing the model’s preference for "high-likelihood" trajectories.

\noindent\textbf{Optimal Policy of the Reward Function $r_{\text{EMPO}}$.} For the answer-space probability-distribution reward \( r_{\text{EMPO}} \) employed by the EMPO algorithm, it instantiates our unified framework with answer-level granularity $\mathcal{I} = \{\mathcal{A}\}$, anchor distribution \( q=\delta^A \), model distribution \( \pi=\pi^A_{\theta} \), \( \sigma=-1 \), and transformation \( \psi(z)=\exp(z) \).  
For a given input \( x \), multiple roll-outs are used to estimate the current policy’s distribution over the answer space.  
With \( \sigma=-1 \), the closer this distribution is to the one-hot reference \( \delta^A \) (i.e., the more probability mass is assigned to the extracted answer), the larger the reward \( r_{\text{EMPO}}(x,y) \) will be.  
Hence, after one policy-update step, the optimal probability \( \pi_{\theta}(y|x) \) increases for sequences that endorse the high-probability answer, while it decreases for all others.  
Maximizing \( r_{\text{EMPO}} \) is therefore equivalent to driving the model to become more decisive at the answer level, thereby improving the consistency and determinism of the generated outputs.

Formally, the reward is defined as:

\begin{equation}
r_{\text{EMPO}}(x,y)=\pi_{\theta}(\text{ans}(y)\mid x),
\quad\text{where}\ \pi_{\theta}(\text{ans}(y)\mid x)=\sum_{\text{ans}(y')=\text{ans}(y)}\pi_{\theta}(y'\mid x).
\end{equation}

Within the KL-regularised RL framework, the one-step optimal policy is:

\begin{equation}
\pi_{\theta}(y\mid x)\propto\pi_{\text{ref}}(y\mid x)\,\exp\!\left(\dfrac{\pi_{\theta}(\text{ans}(y)\mid x)}{\beta}\right).
\label{eq:EMPO_optimal_policy}
\end{equation}

As evidenced by \Cref{eq:EMPO_optimal_policy}, a single EMPO update re-weights each sequence by a factor of \( \exp\!\bigl(\pi_{\theta}(\text{ans}(y)\mid x)/\beta\bigr) \) that depends on the current answer-level probability.  
After normalization, answers that already enjoy high probability under the prior policy gain additional mass, whereas low-probability answers suffer a decrease.  
Consequently, the optimal policy at each step systematically shifts the overall probability mass toward the high-probability region of the prior policy.

\section{Details for \Cref{sec:rise}}

\begin{figure*}[!t]
    % % \vspace{-30pt}
    \centering
    \includegraphics[width=.6\linewidth]{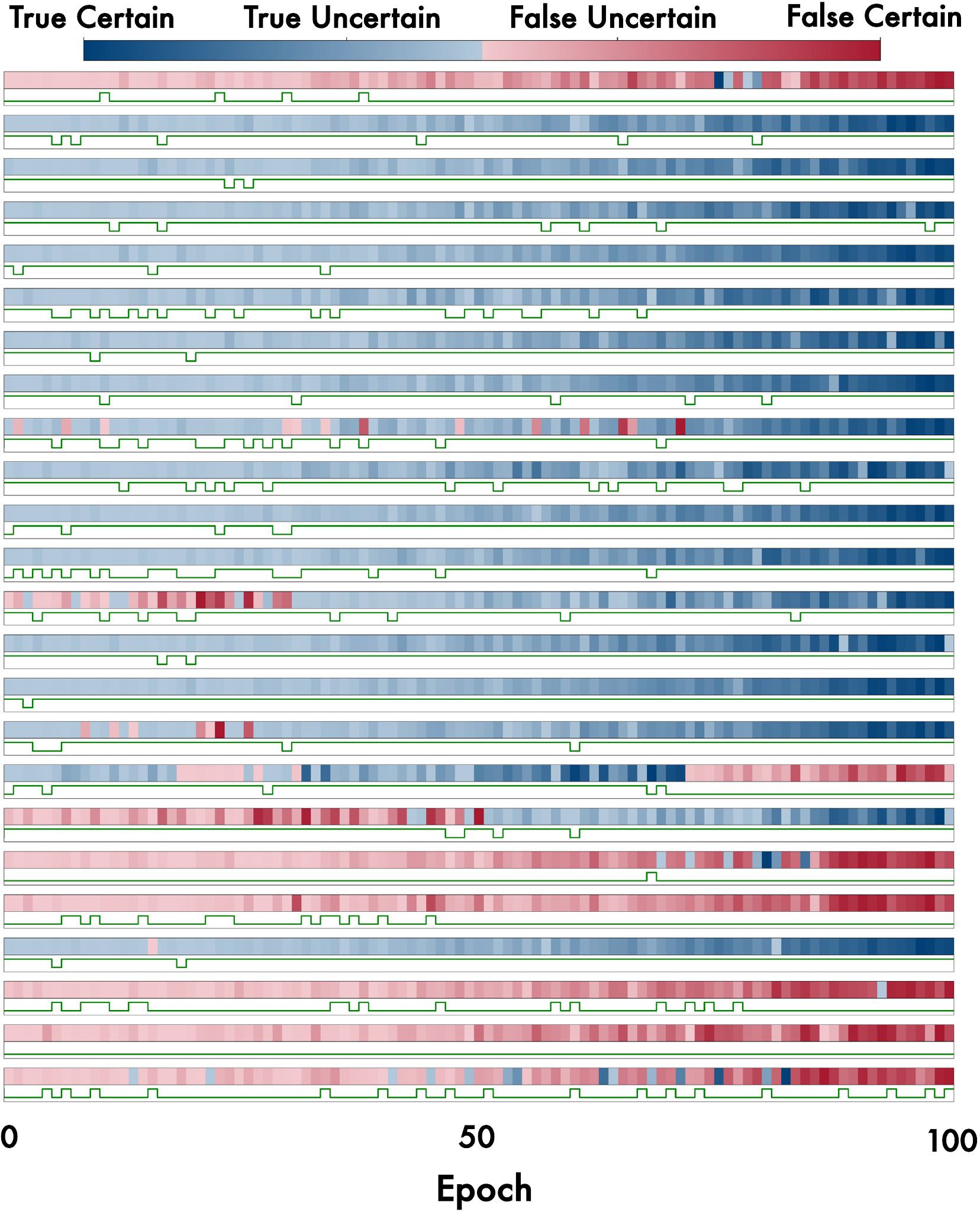}
    % % \vspace{-15pt}
    \caption{Examples of per-problem training dynamics from MATH-500.}
    \label{fig:heatmap_step_all_label}
    % % \vspace{-15pt}
\end{figure*}

\subsection{Experimental Setup}
\label{app:experimental_setup}

\begin{table}[!t]
    % % % \vspace{-30pt}
    \centering
    \renewcommand{\arraystretch}{1.3}
    \caption{Default hyperparameters for training.}
    \resizebox{\textwidth}{!}{
    \begin{tabular}{cccccccccc}
    % \begin{tabular}{@{}ccccccc@{}}
    \toprule
    \textbf{\makecell[c]{Advantage\\ Estimator}} &
    \textbf{\makecell[c]{Training\\ Temperature}} & \textbf{\makecell[c]{Global\\ Batch Size}} & \textbf{\makecell[c]{Mini\\ Batch Size}} & \textbf{\makecell[c]{Rollout\\ Number}} & \textbf{Regularization} & \textbf{\makecell[c]{Max Prompt\\ Length}} & \textbf{\makecell[c]{Max Response\\ Length}} & \textbf{\makecell[c]{Learning\\ Rate}} & \textbf{Epoch} \\ \midrule
    GRPO & 1.0 & 64 & 64 & 8 & \makecell[c]{w/o KL/Entropy} & 1024 & 7168 & 1e-6 & 1  \\
    \toprule
    
    \end{tabular}
    }
    % % \vspace{-10pt}
    \label{tab:default_hyper}
    % % \vspace{-10pt}
\end{table}

\noindent\textbf{Implementation Details.}
All experiments are conducted using the veRL framework~\citep{sheng2025hybridflow} with the GRPO algorithm. Unless stated otherwise, we utilize the default configuration outlined in \Cref{tab:default_hyper}. We implement five representative intrinsic rewards by customizing the \texttt{RewardManager} module of VeRL, following the reward formulations in \Cref{tab:certainty_based_formulas} and \Cref{tab:ensemble_based_formulas}:
\begin{itemize}[topsep=0pt, itemsep=0pt, leftmargin=18pt]
    \item \textbf{Ensemble-Based Reward Estimators:} Majority Voting
    \item \textbf{Certainty-Based Reward Estimators:} Self-Certainty, Token-Level Entropy, Trajectory-Level Entropy, and Probability
\end{itemize}
% \xiusi{We may want to briefly explain why we did not include heuristic-based approaches?}

% \noindent\textbf{Evaluation Protocol.} 
% We evaluate on three challenging mathematics benchmarks: AIME 2024~\citep{li2024numinamath}, AIME 2025~\citep{balunovic2025matharena}, and AMC 2023~\citep{li2024numinamath}.
% % MATH-500~\citep{hendrycks2021measuring}, Minerva Math~\citep{lewkowycz2022solving}, and OlympiadBench~\citep{he-etal-2024-olympiadbench}.
% Following standard practice, we generate 32 solutions per problem using a temperature of 0.6 and a top-p value of 0.95, and report the mean accuracy at 32 solutions (mean@32).

\noindent\textbf{Training Dynamics Monitoring.} 
To monitor reward hacking and validate our theoretical predictions from \Cref{sec:theory}, we implement specialized metrics to track the evolution of pseudo-rewards and their alignment with ground truth. These metrics help identify when and how these intrinsic methods transition from beneficial sharpening to pathological collapse.

\begin{itemize}[topsep=2pt, partopsep=0pt, leftmargin=12pt, itemsep=2pt]
\item \textbf{Ensemble-Based Metrics:} For methods using majority voting, we separately track the accuracy of the chosen label and the accuracy of the rewards it generates.
\begin{itemize}
    \item \textit{Label Accuracy}: Prompt-level accuracy of majority-voted answers against ground truth, measuring ensemble quality
    \item \textit{Reward Accuracy}: Sample-level agreement between pseudo rewards and oracle rewards, capturing ``lucky hits''~\citep{zuo2025ttrl} where individual rewards align despite incorrect majority votes
    \item \textit{Ground Truth Reward}: Average oracle reward (supervised baseline), computed using actual correctness
    \item \textit{Majority Voting Reward}: Average pseudo reward from majority voting, the divergence from \textit{Ground Truth Reward} indicates reward hacking
\end{itemize}
\item \textbf{Certainty-Based Metrics:} For certainty-based methods, we measure the correlation between this proxy reward and the actual correctness.
\begin{itemize}
\item \textit{Label Accuracy}: Ground-truth accuracy of the highest-confidence response per prompt, testing whether maximum certainty implies correctness
% \item \textit{Point-Biserial Correlation}: Point-biserial correlation between pseudo reward and binary correctness, quantifying the fundamental assumption that confidence predicts accuracy
\end{itemize}
\end{itemize}

These metrics collectively diagnose (1) pseudo-label quality degradation via \textit{Label Accuracy} and (2) reward signal corruption via the gap between \textit{Majority Voting Reward} and \textit{Ground Truth Reward}. Mathematical definitions and implementation details are provided in \Cref{app:metrics}.

% These metrics collectively diagnose three critical phenomena: (1) pseudo-label quality degradation via \textit{Label Accuracy}, (2) reward signal corruption via the gap between \textit{Majority Voting Reward} and \textit{Ground Truth Reward}, and (3) confidence miscalibration via \textit{Point-Biserial Correlation}. Mathematical definitions and implementation details are provided in \Cref{app:metrics}.

\subsection{Calculation of Training Dynamics}
\label{app:metrics}

We provide mathematical definitions for the metrics used to monitor training dynamics. These metrics diagnose reward hacking and validate theoretical predictions about distribution sharpening.

\subsubsection{Notation}

Let $\mathcal{D} = \{(x_i, a_i^*)\}_{i=1}^{M}$ denote the training dataset with $M$ prompts, where $x_i$ is the $i$-th prompt and $a_i^*$ is its ground-truth answer. For each prompt $x_i$, we generate $N$ rollout responses $\{y_{i,j}\}_{j=1}^{N}$ from the current policy $\pi_\theta$, where each response $y_{i,j}$ contains a trajectory and an extracted answer $\text{ans}(y_{i,j})$.

Define the following:
\begin{itemize}[topsep=0pt, itemsep=0pt, leftmargin=18pt]
    \item $\mathbf{1}[\cdot]$: Indicator function returning 1 if the condition is true, 0 otherwise
    \item $\text{maj}(x_i)$: Majority-voted answer for prompt $x_i$, computed as $\arg\max_{a} \sum_{j=1}^{N} \mathbf{1}[\text{ans}(y_{i,j}) = a]$
    \item $r_{\text{gt}}(y_{i,j})$: Ground-truth reward for response $y_{i,j}$, equals $\mathbf{1}[\text{ans}(y_{i,j}) = a_i^*]$
    \item $r_{\text{mv}}(y_{i,j})$: Majority-voting pseudo-reward, equals $\mathbf{1}[\text{ans}(y_{i,j}) = \text{maj}(x_i)]$
    \item $r_{\text{cert}}(y_{i,j})$: Certainty-based reward (e.g., self-certainty, entropy) for response $y_{i,j}$
\end{itemize}

\subsubsection{Ensemble-Based Metrics}

\paragraph{Label Accuracy}
Measures the prompt-level accuracy of majority-voted answers:
\begin{equation}
\text{Label Accuracy} = \frac{1}{M} \sum_{i=1}^{M} \mathbf{1}[\text{maj}(x_i) = a_i^*].
\end{equation}
This metric ranges from 0 to 1, where 1 indicates perfect pseudo-label generation.

\paragraph{Reward Accuracy}
Quantifies sample-level agreement between pseudo-rewards and oracle rewards:
\begin{equation}
\text{Reward Accuracy} = \frac{1}{M \cdot N} \sum_{i=1}^{M} \sum_{j=1}^{N} \mathbf{1}[r_{\text{mv}}(y_{i,j}) = r_{\text{gt}}(y_{i,j})].
\end{equation}
This captures ``lucky hits'' where individual rewards are correct even when the majority vote is wrong. For example, if the majority vote is incorrect but a minority response is correct, that response still receives the appropriate (zero) pseudo-reward.

\paragraph{Ground Truth Reward}
Average oracle reward across all generated responses:
\begin{equation}
\text{Ground Truth Reward} = \frac{1}{M \cdot N} \sum_{i=1}^{M} \sum_{j=1}^{N} r_{\text{gt}}(y_{i,j}).
\end{equation}
This represents the true quality of generated responses and serves as the supervised baseline.

\paragraph{Majority Voting Reward}
Average pseudo-reward from majority voting:
\begin{equation}
\text{Majority Voting Reward} = \frac{1}{M \cdot N} \sum_{i=1}^{M} \sum_{j=1}^{N} r_{\text{mv}}(y_{i,j}).
\end{equation}
The divergence between this metric and Ground Truth Reward indicates reward hacking: when the model learns to maximize pseudo-rewards at the expense of actual correctness.

\subsubsection{Certainty-Based Metrics}

\paragraph{Label Accuracy}
For certainty-based methods, we identify the highest-confidence response per prompt and measure its accuracy:
\begin{equation}
\text{Label Accuracy} = \frac{1}{M} \sum_{i=1}^{M} \mathbf{1}[\text{ans}(y_{i,j^*_i}) = a_i^*],
\end{equation}
where $j^*_i = \arg\max_{j \in \{1,\ldots,N\}} r_{\text{cert}}(y_{i,j})$ is the index of the highest-confidence response for prompt $x_i$.

% \paragraph{Point-Biserial Correlation}
% Measures the correlation between continuous certainty scores and binary correctness:
% \begin{equation}
% \rho_{pb} = \frac{\bar{r}_1 - \bar{r}_0}{s_r} \cdot \sqrt{\frac{n_1 n_0}{n^2}},
% \end{equation}
% where:
% \begin{itemize}[topsep=0pt, itemsep=0pt, leftmargin=18pt]
%     \item $n = M \cdot N$ is the total number of responses
%     \item $n_1 = \sum_{i,j} r_{\text{gt}}(y_{i,j})$ is the number of correct responses
%     \item $n_0 = n - n_1$ is the number of incorrect responses
%     \item $\bar{r}_1 = \frac{1}{n_1} \sum_{i,j: r_{\text{gt}}(y_{i,j})=1} r_{\text{cert}}(y_{i,j})$ is the mean certainty for correct responses
%     \item $\bar{r}_0 = \frac{1}{n_0} \sum_{i,j: r_{\text{gt}}(y_{i,j})=0} r_{\text{cert}}(y_{i,j})$ is the mean certainty for incorrect responses
%     \item $s_r = \sqrt{\frac{1}{n-1} \sum_{i,j} (r_{\text{cert}}(y_{i,j}) - \bar{r})^2}$ is the standard deviation of all certainty scores
%     \item $\bar{r} = \frac{1}{n} \sum_{i,j} r_{\text{cert}}(y_{i,j})$ is the mean of all certainty scores
% \end{itemize}

% The correlation $\rho_{pb} \in [-1, 1]$ quantifies the relationship between confidence and correctness. Positive values indicate that higher certainty correlates with correctness (desired behavior), while values near zero suggest certainty is uninformative, and negative values indicate miscalibration.

\subsection{Hyperparameter Tuning}
\label{app:hyperparameter}

\noindent\textbf{Setup.} We study four hyperparameters, including training temperature, mini-batch size, KL divergence regularization, and rollout count, that directly influence performance of intrinsic reward training. We vary one parameter at a time while keeping others fixed at baseline values (see \Cref{tab:default_hyper}).

\subsubsection{Majority Voting}

\begin{figure*}[!t]
    % % \vspace{-30pt}
    \centering
    \includegraphics[width=1\linewidth]{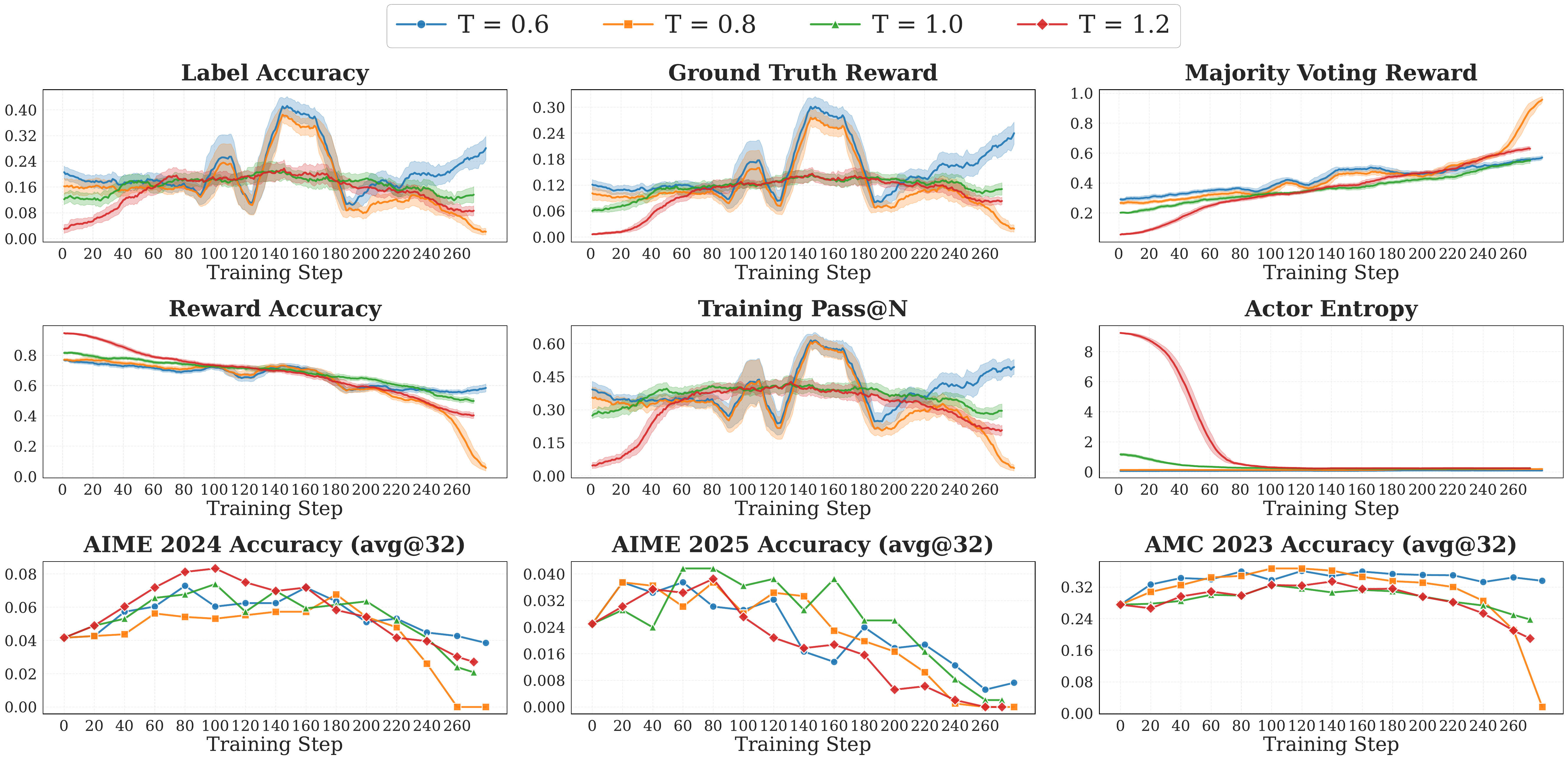}
    % \vspace{-20pt}
    \caption{Effect of training temperature for Majority Voting method.}
    % Higher temperatures increase exploration but may compromise stability, with $T=1.0$ achieving optimal balance across validation benchmarks.
    \label{fig:tune_temperature}
    % \vspace{-5pt}
\end{figure*}

\noindent\textbf{Training Temperature.} Temperature directly controls exploration during rollout generation and affects the quality of pseudo-labels via voting diversity. From our convergence analysis in \Cref{theorem:mv_convergence}, lower temperature reduces the effective $\beta$ in the KL regularization term, accelerating convergence. As shown in \Cref{fig:tune_temperature}, low $T\!\in\!\{0.6,0.8\}$ quickly sharpens logits, causing unstable \textit{Label Accuracy}, consistent with premature convergence to an early majority that may be incorrect. Higher temperature ($T=1.2$) maintains stability longer by preserving exploration, but the increased noise reduces peak performance. We find $T=1.0$ provides optimal balance, showing steady early gains with delayed degradation.

\begin{figure*}[!t]
    % % \vspace{-30pt}
    \centering
    \includegraphics[width=1\linewidth]{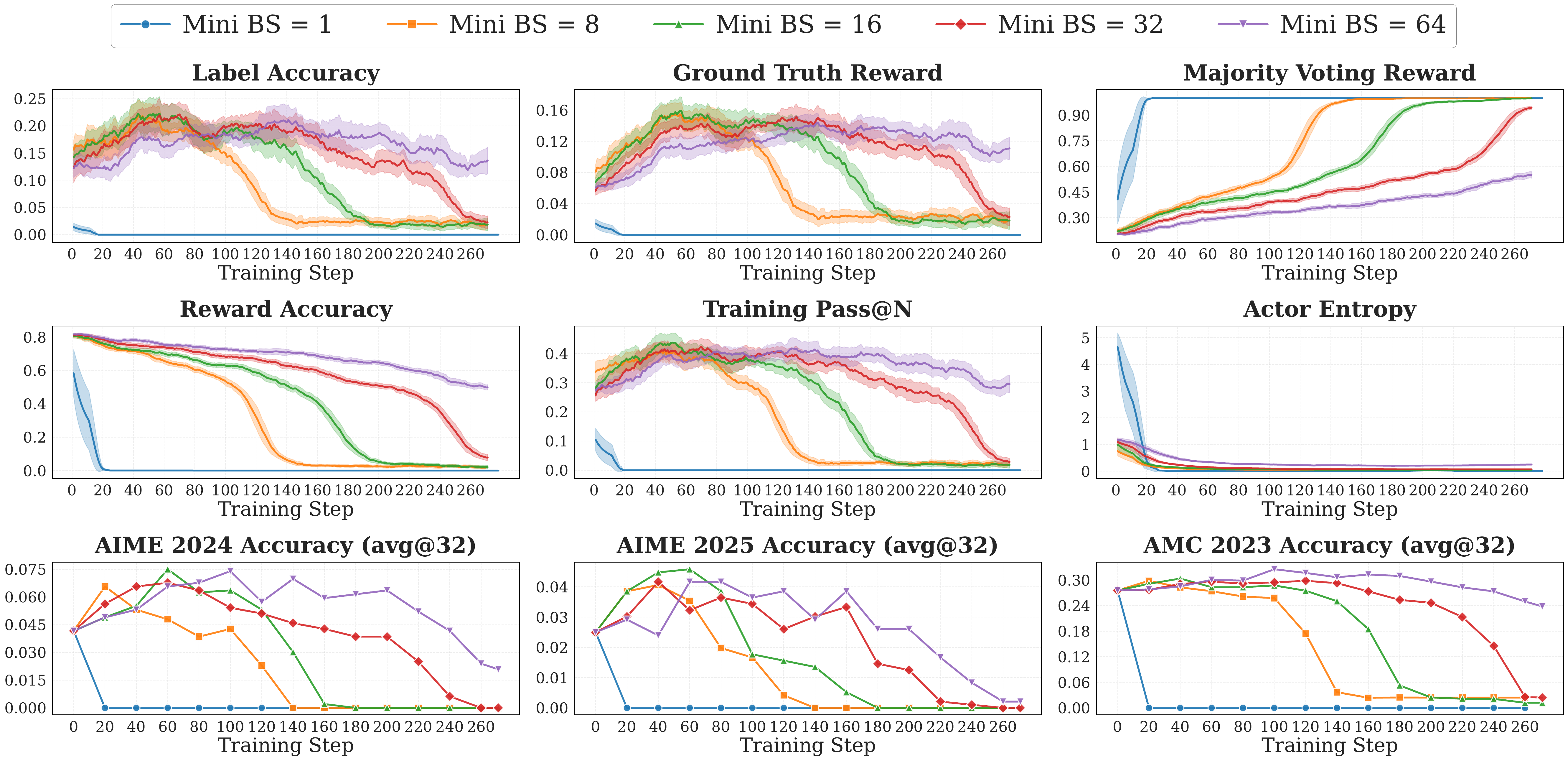}
    % \vspace{-20pt}
    \caption{Effect of mini-batch size for Majority Voting method.}
    % Larger mini-batch sizes provide more stable training and slower reward hacking, with pure on-policy training (mini-batch size = 64) achieving the best stability-performance trade-off.}
    \label{fig:tune_mbs}
    % \vspace{-5pt}
\end{figure*}

\noindent\textbf{Mini-batch Size.} This parameter controls the on-policy nature of updates, directly affecting the validity of our optimal policy assumptions. Our theoretical derivation in \Cref{eq:MV_optimal_detailed} assumes rewards are computed under the current policy $\pi_\theta$. Small mini-batches violate this assumption through reward staleness: pseudo-rewards computed under $\pi_\theta$ become misaligned when applied to samples from $\pi_{\theta_{\text{old}}}$. As shown in \Cref{fig:tune_mbs}, mini-batch size $1$ drives rapid collapse within $20$ steps, while pure on-policy training (mini-batch $= 64$, matching global batch size) provides maximum stability. The intermediate sizes ($16$–$32$) show gradual improvement, confirming that maintaining policy-reward alignment is crucial for stable convergence.

\begin{figure*}[!t]
    % % \vspace{-30pt}
    \centering
    \includegraphics[width=1\linewidth]{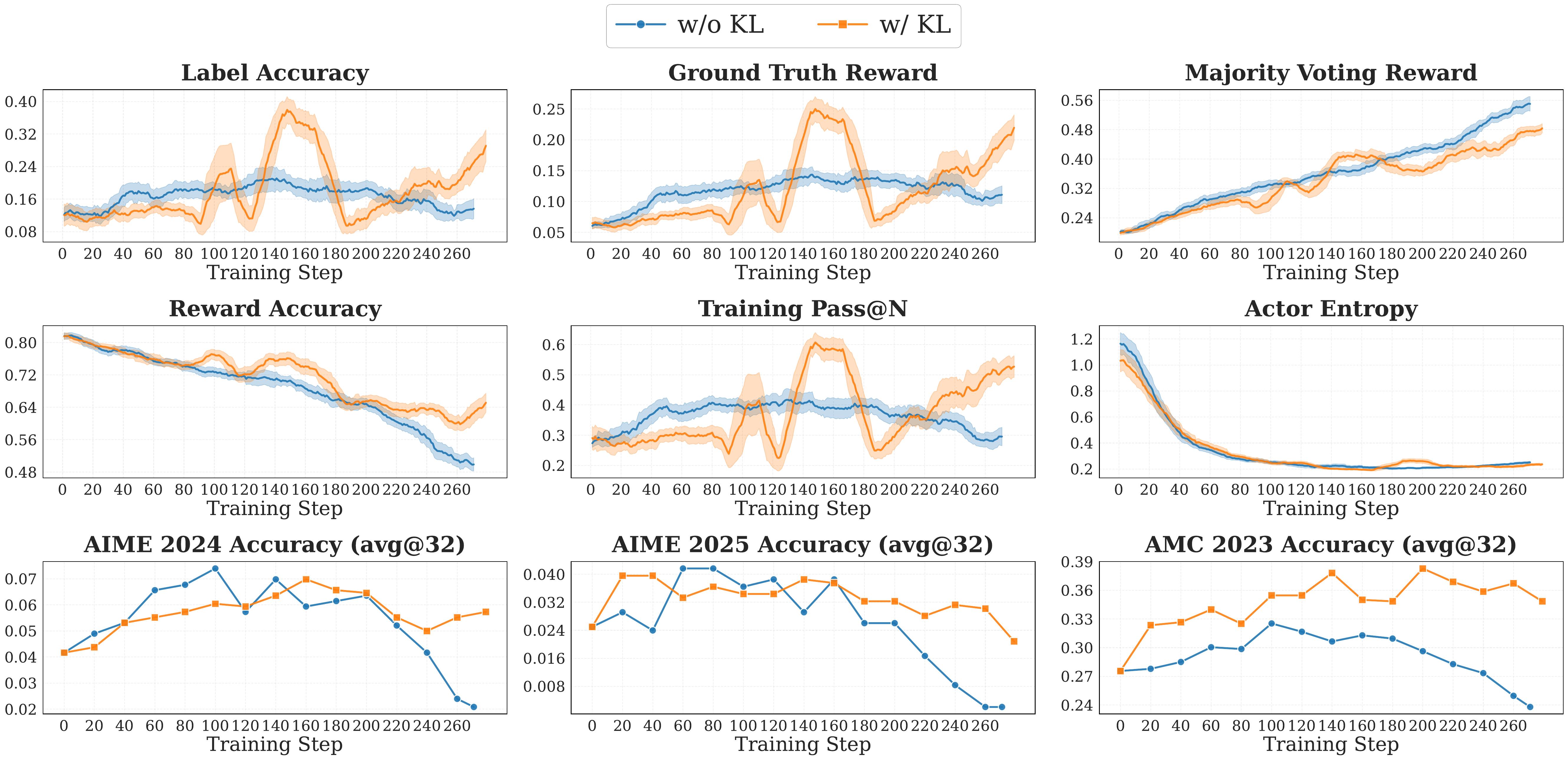}
    % \vspace{-20pt}
    \caption{Effect of KL divergence regularization for Majority Voting method.}
    % While KL regularization shows marginally higher validation scores, it introduces training instability without significant performance gains.}
    \label{fig:tune_kl}
    % \vspace{-5pt}
\end{figure*}

\noindent\textbf{KL Regularization.} Our theoretical analysis suggests that KL regularization should slow convergence by increasing the effective $\beta$ parameter in \Cref{eq:optimal_policy_general}. However, empirical results in \Cref{fig:tune_kl} show that adding KL regularization ($\beta=0.005$) yields only marginal benefits: small early gains but increased training variance and minimal delay in collapse ($\sim40$ steps). This discrepancy arises because intrinsic rewards create competing optimization pressures, where the intrinsic signal drives sharpening while KL pulls toward the reference policy. Rather than smoothly balancing these forces, the optimization oscillates between them, increasing variance without providing durable stability. The marginal gains do not justify the additional memory overhead and training instability.

\begin{figure*}[!t]
    % % \vspace{-30pt}
    \centering
    \includegraphics[width=1\linewidth]{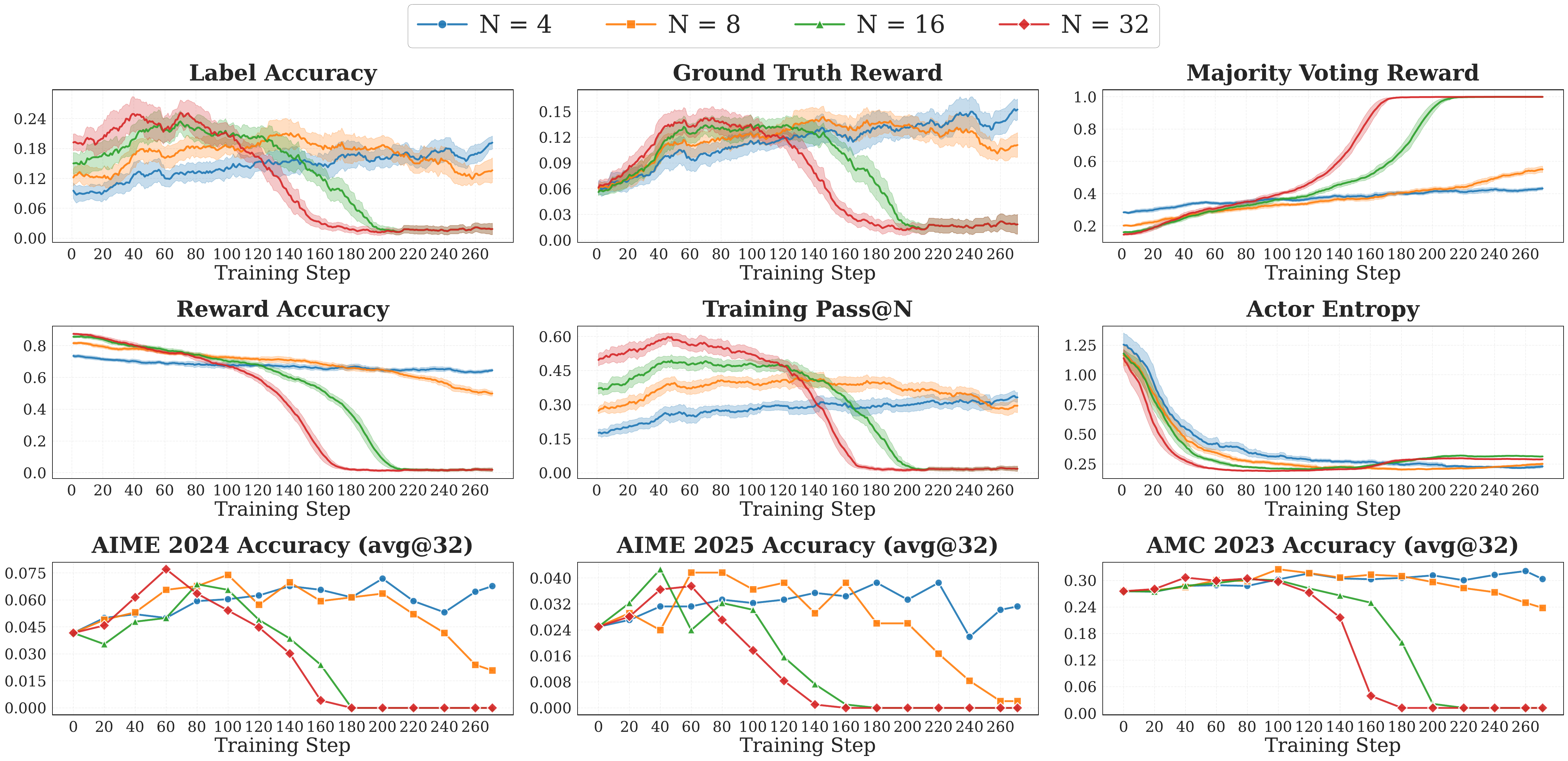}
    % \vspace{-20pt}
    \caption{Effect of rollout number for Majority Voting method.}
    % Higher rollout counts ($N \geq 16$) accelerate reward hacking and training collapse, while $N=8$ balances sample diversity with training stability.}
    \label{fig:tune_rollout_n}
    % \vspace{-5pt}
\end{figure*}

\noindent\textbf{Number of Rollouts.} The rollout count $N$ affects both vote reliability and signal strength. While more rollouts improve statistical reliability of the majority vote, they also amplify the majority signal strength. From \Cref{eq:MV_optimal_detailed}, each update amplifies majority probability by factor $e^{1/\beta}$. With more rollouts, this majority becomes more confident, accelerating convergence. \Cref{fig:tune_rollout_n} shows this effect: $N=32$ collapses within $180$ steps, $N=16$ within $220$ steps, while $N\leq8$ remains stable over the full epoch. Although $N=4$ shows competitive performance in some metrics, we recommend $N=8$ as it provides better statistical reliability for the voting mechanism while maintaining reasonable convergence control. The slight performance difference suggests that for this specific experimental setup, the trade-off between reliability and stability favors slightly smaller $N$, but $N=8$ offers more robust behavior across diverse problem types.

\subsubsection{Certainty-Based Methods}

% =================================================================

\begin{figure*}[!t]
    % % \vspace{-30pt}
    \centering
    \includegraphics[width=1\linewidth]{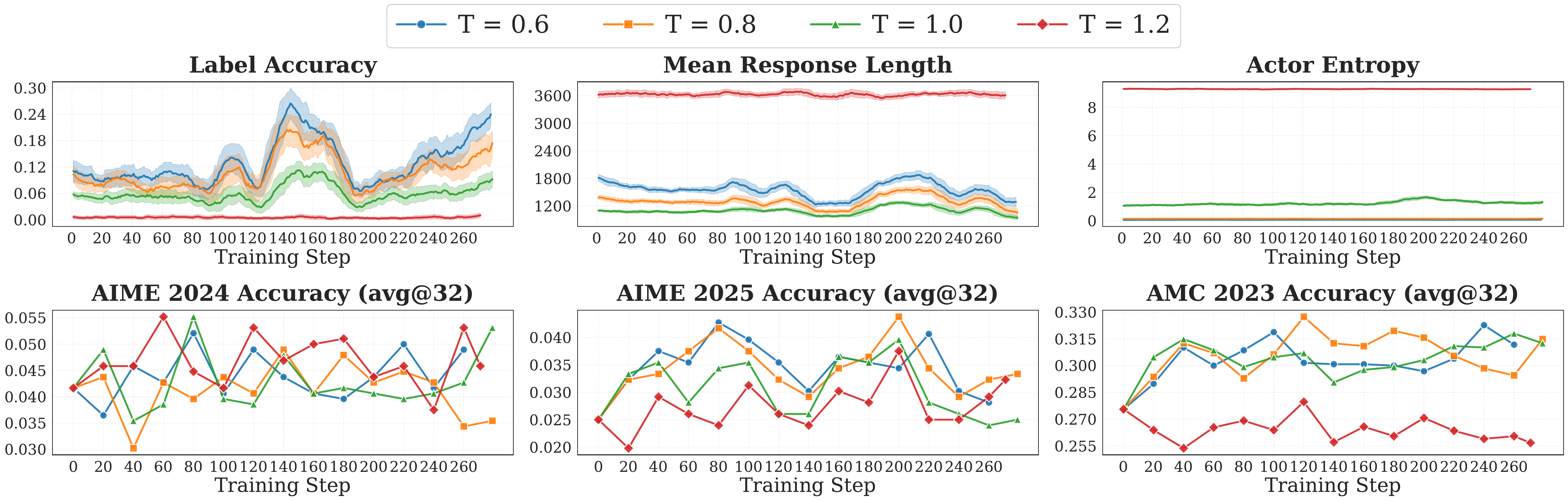}
    % \vspace{-20pt}
    \caption{Effect of training temperature on Self-Certainty performance. Note that Point-Biserial Correlation is replaced with Mean Response Length due to Self-Certainty's scoring characteristics.}
    \label{fig:tune_temperature_sc}
    % \vspace{-5pt}
\end{figure*}

\begin{figure*}[!t]
    % % \vspace{-30pt}
    \centering
    \includegraphics[width=1\linewidth]{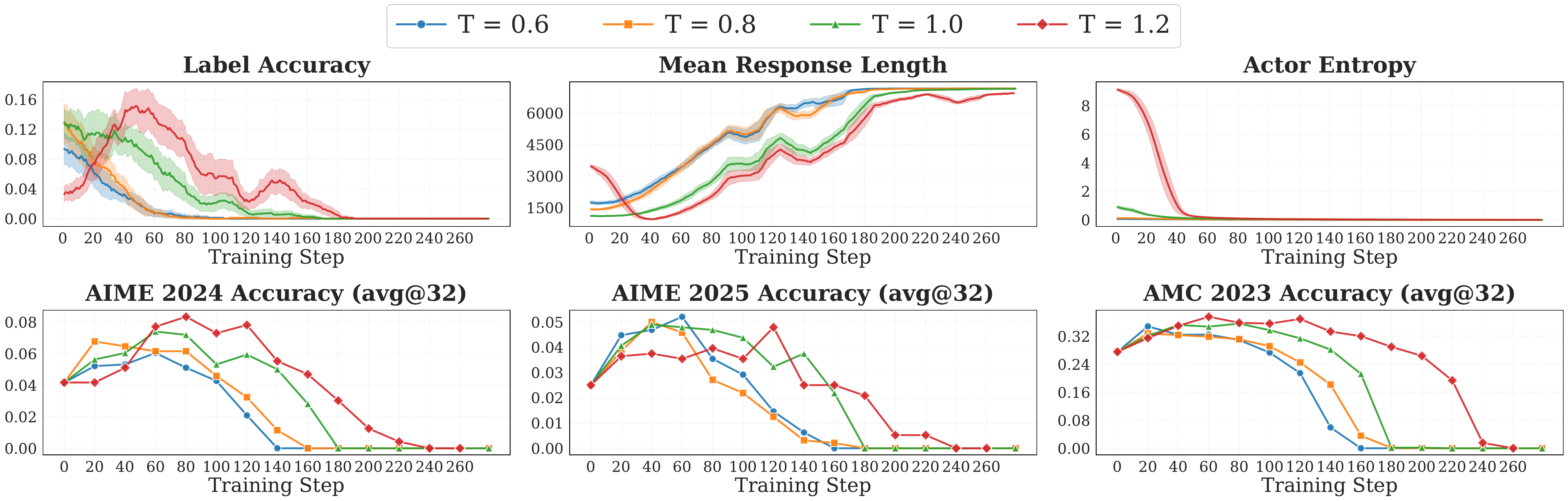}
    % \vspace{-20pt}
    \caption{Effect of training temperature on Token-Level Entropy performance.}
    \label{fig:tune_temperature_token_ent}
    % \vspace{-5pt}
\end{figure*}

\begin{figure*}[!t]
    % % \vspace{-30pt}
    \centering
    \includegraphics[width=1\linewidth]{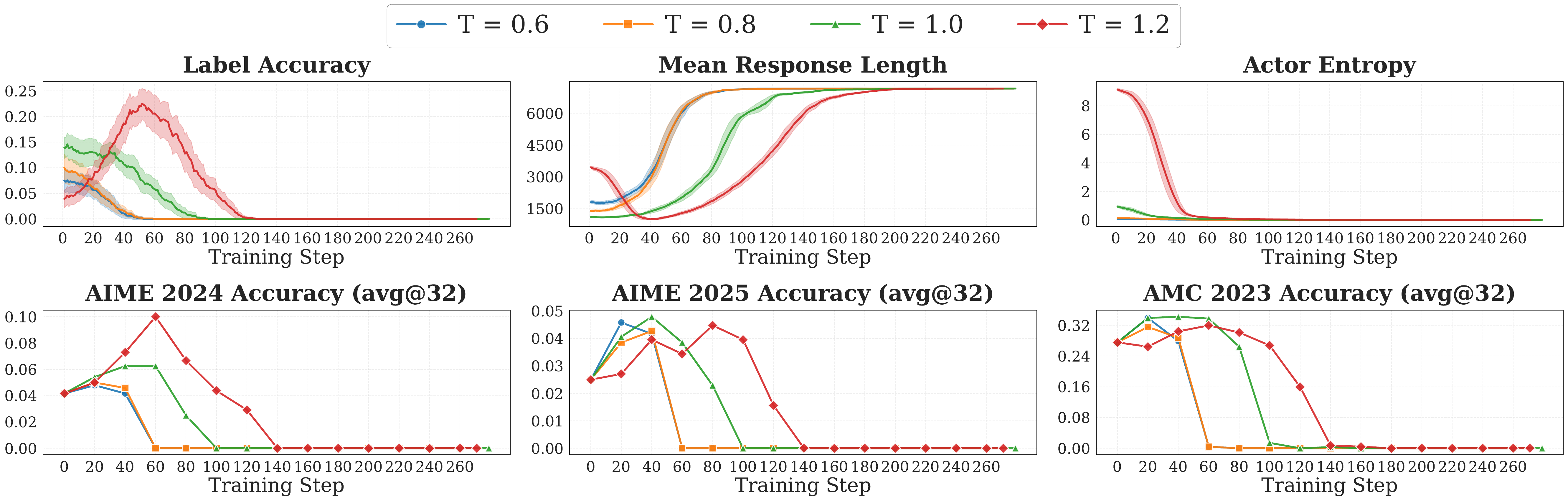}
    % \vspace{-20pt}
    \caption{Effect of training temperature on Trajectory-Level Entropy performance.}
    \label{fig:tune_temperature_traj_ent}
    % \vspace{-5pt}
\end{figure*}

\begin{figure*}[!t]
    % % \vspace{-30pt}
    \centering
    \includegraphics[width=1\linewidth]{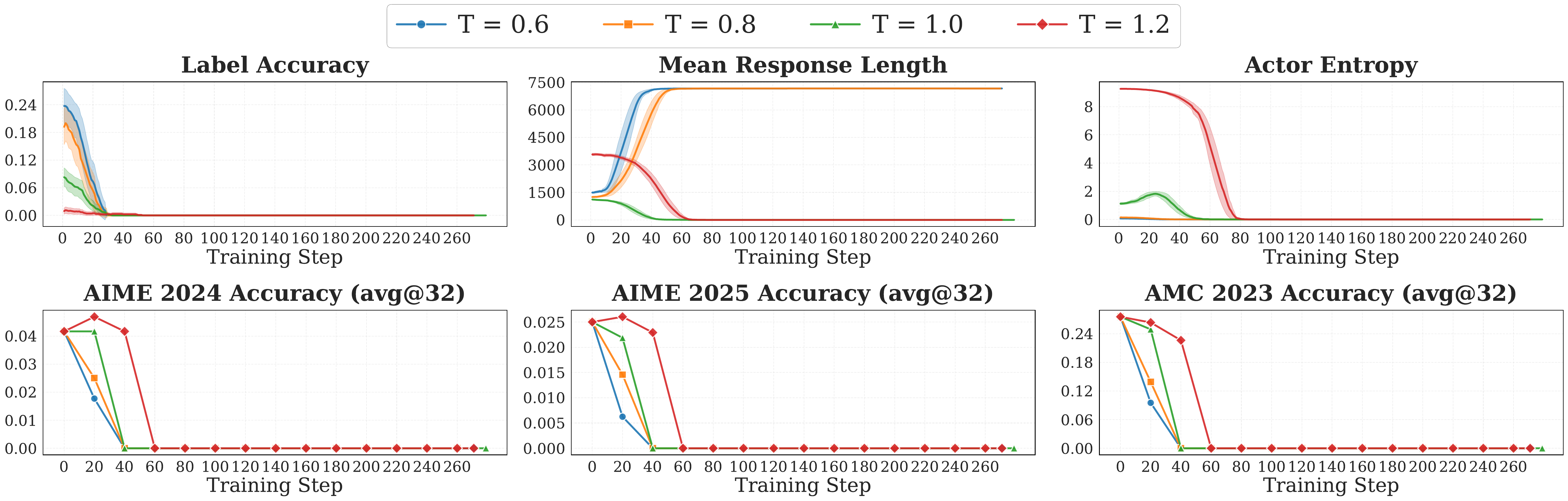}
    % \vspace{-20pt}
    \caption{Effect of training temperature on Probability-based certainty performance.}
    \label{fig:tune_temperature_prob}
    % \vspace{-5pt}
\end{figure*}

\noindent\textbf{Training Temperature.} Temperature effects on certainty-based methods reveal distinct behavioral patterns compared to ensemble-based approaches. Unlike Majority Voting, certainty-based methods generally benefit from higher exploration temperatures, with notable method-specific variations in optimal configurations and convergence characteristics.

Results in \Cref{fig:tune_temperature_token_ent,fig:tune_temperature_traj_ent,fig:tune_temperature_prob} demonstrate that higher temperature ($T=1.2$) significantly delays model collapse across Token-Level Entropy, Trajectory-Level Entropy, and Probability methods. Higher temperatures initially maintain elevated \textbf{Actor Entropy}, facilitating extended exploration phases with gradual improvements across validation benchmarks. Importantly, these methods also exhibit relatively higher \textbf{Point-Biserial Correlation} values at $T=1.2$, indicating stronger alignment between certainty estimates and actual correctness—a crucial property for effective uncertainty-based reward assignment.

However, \Cref{fig:tune_temperature_sc} reveals that Self-Certainty exhibits contrasting behavior. Higher temperature ($T=1.2$) leads to excessive exploration without convergence, maintaining persistently high \textbf{Actor Entropy} while achieving lower validation scores and \textbf{Label Accuracy}. The moderate temperature $T=1.0$ provides more stable and superior performance for Self-Certainty. This divergence suggests that while different certainty-based methods converge toward similar sharp distributions, they exhibit distinct convergence rates requiring method-specific temperature tuning. Among all certainty-based approaches, Token-Level and Trajectory-Level Entropy methods demonstrate the greatest benefits from higher temperature exploration, likely due to their more robust entropy-based uncertainty estimation mechanisms.

% =================================================================
\begin{figure*}[!t]
    % % \vspace{-30pt}
    \centering
    \includegraphics[width=1\linewidth]{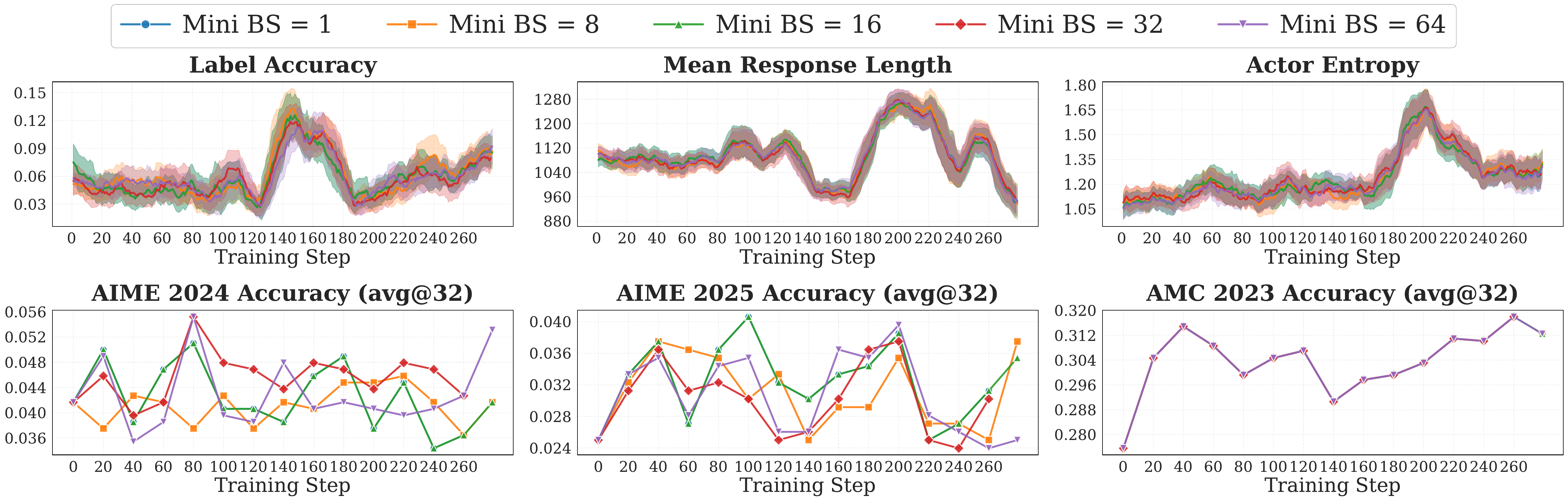}
    % \vspace{-20pt}
    \caption{Effect of mini-batch size on Self-Certainty performance.}
    \label{fig:tune_mbs_sc}
    % \vspace{-5pt}
\end{figure*}

\begin{figure*}[!t]
    % % \vspace{-30pt}
    \centering
    \includegraphics[width=1\linewidth]{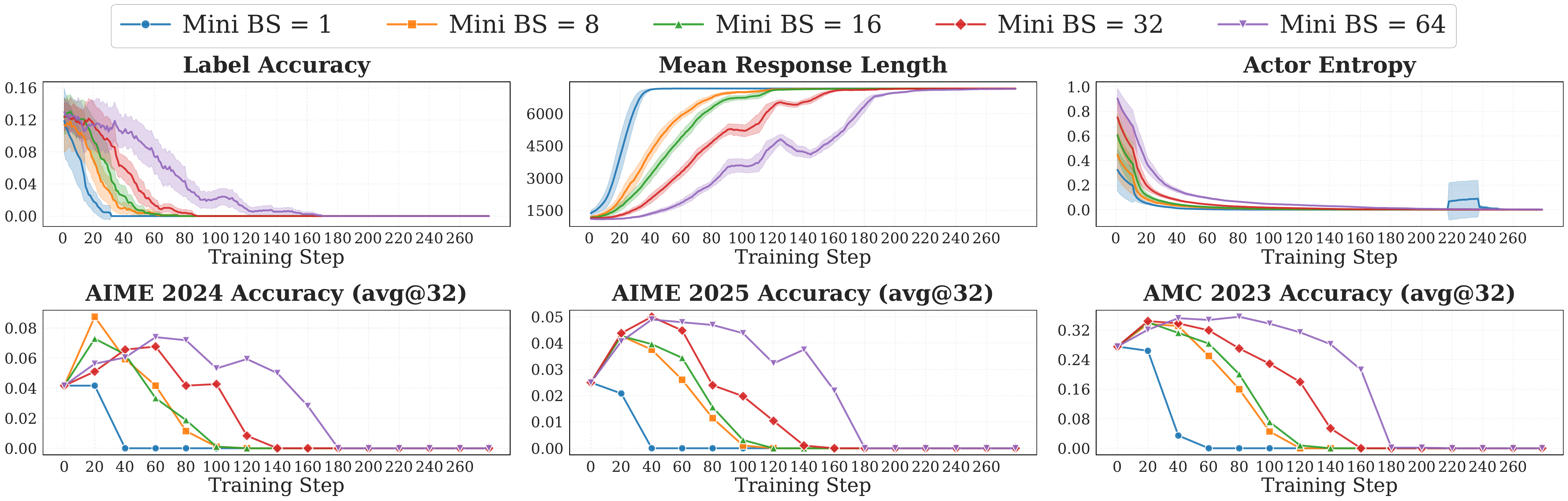}
    % \vspace{-20pt}
    \caption{Effect of mini-batch size on Token-Level Entropy performance.}
    \label{fig:tune_mbs_token_ent}
    % \vspace{-5pt}
\end{figure*}

\begin{figure*}[!t]
    % % \vspace{-30pt}
    \centering
    \includegraphics[width=1\linewidth]{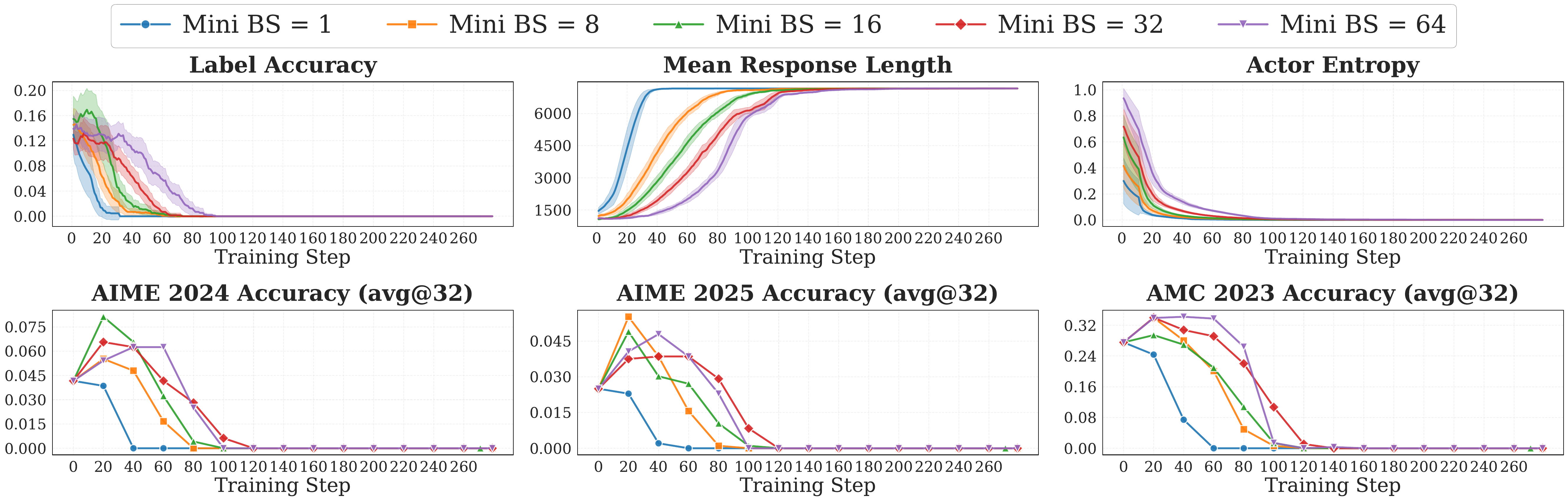}
    % \vspace{-20pt}
    \caption{Effect of mini-batch size on Trajectory-Level Entropy performance.}
    \label{fig:tune_mbs_traj_ent}
    % \vspace{-5pt}
\end{figure*}

\begin{figure*}[!t]
    % % \vspace{-30pt}
    \centering
    \includegraphics[width=1\linewidth]{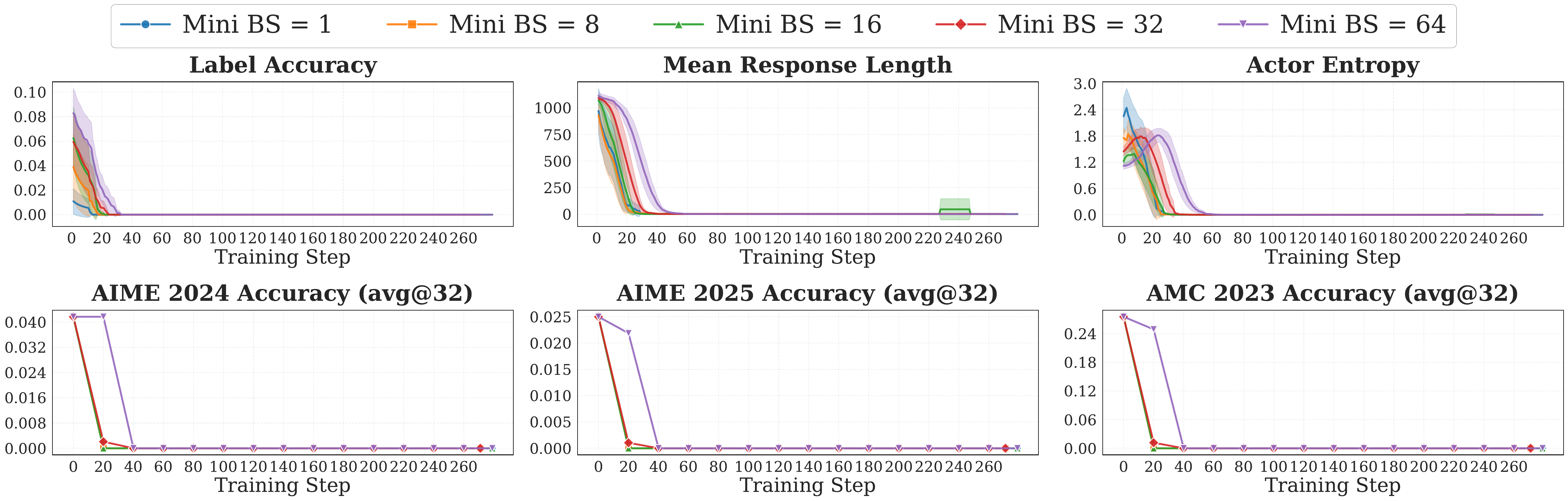}
    % \vspace{-20pt}
    \caption{Effect of mini-batch size on Probability performance.}
    \label{fig:tune_mbs_prob}
    % \vspace{-5pt}
\end{figure*}

\noindent\textbf{Mini-Batch Size.} Mini-batch size effects on certainty-based methods largely parallel those observed in Majority Voting, confirming that on-policy ratio critically affects training stability regardless of the underlying reward computation mechanism. However, method-specific sensitivities reveal important distinctions in robustness to off-policy learning.

\Cref{fig:tune_mbs_token_ent,fig:tune_mbs_traj_ent,fig:tune_mbs_prob} consistently demonstrate that larger mini-batch sizes prevent premature model collapse across Token-Level Entropy, Trajectory-Level Entropy, and Probability methods. This pattern mirrors Majority Voting behavior, where pure on-policy training (mini-batch size = 64) maintains optimal coupling between samples and their corresponding certainty-based rewards. The underlying mechanism remains consistent: certainty estimates computed from current policy states become unreliable when applied to samples generated from earlier policy iterations.

Notably, Self-Certainty exhibits exceptional robustness to mini-batch size variations, as shown in \Cref{fig:tune_mbs_sc}. This method demonstrates minimal sensitivity to on-policy ratio changes, suggesting that KL divergence-based certainty computation may be inherently more stable across different temporal policy alignments. This robustness likely stems from Self-Certainty's reliance on logit distribution comparisons rather than explicit probability estimates, making it less susceptible to the temporal inconsistencies that destabilize other certainty-based approaches. Among the certainty-based methods, Self-Certainty thus offers superior stability but at the cost of lower overall performance improvements.

% =================================================================

\begin{figure*}[!t]
    % % \vspace{-30pt}
    \centering
    \includegraphics[width=1\linewidth]{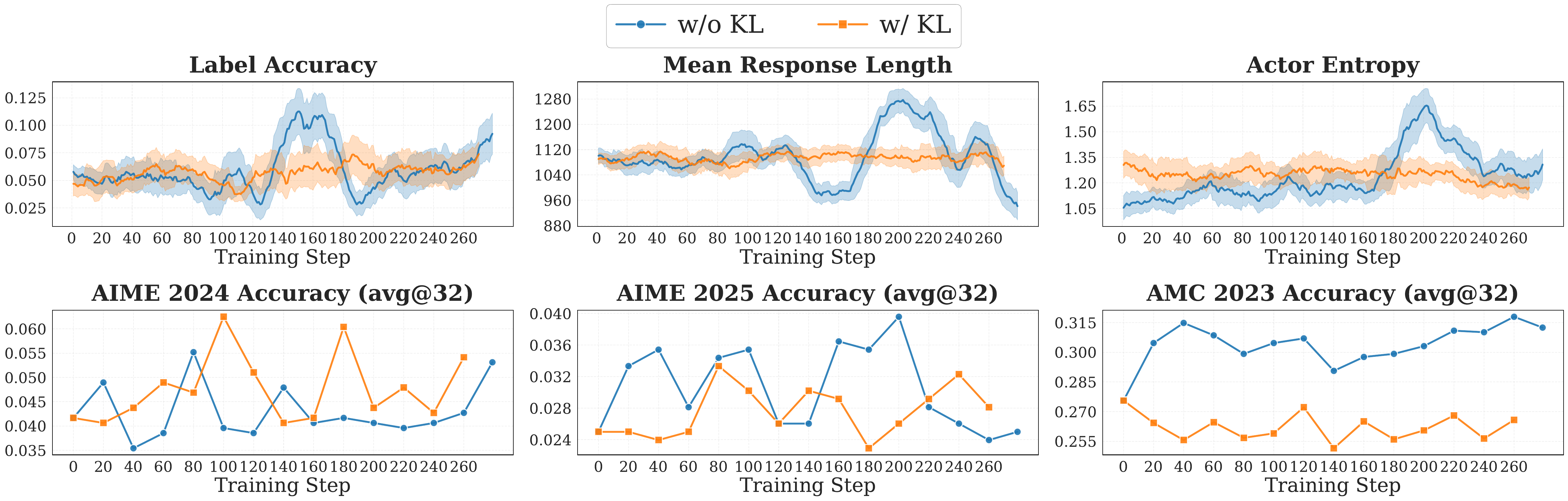}
    % \vspace{-20pt}
    \caption{Effect of KL divergence regularization on Self-Certainty performance.}
    \label{fig:tune_kl_sc}
    % \vspace{-5pt}
\end{figure*}

\begin{figure*}[!t]
    % % \vspace{-30pt}
    \centering
    \includegraphics[width=1\linewidth]{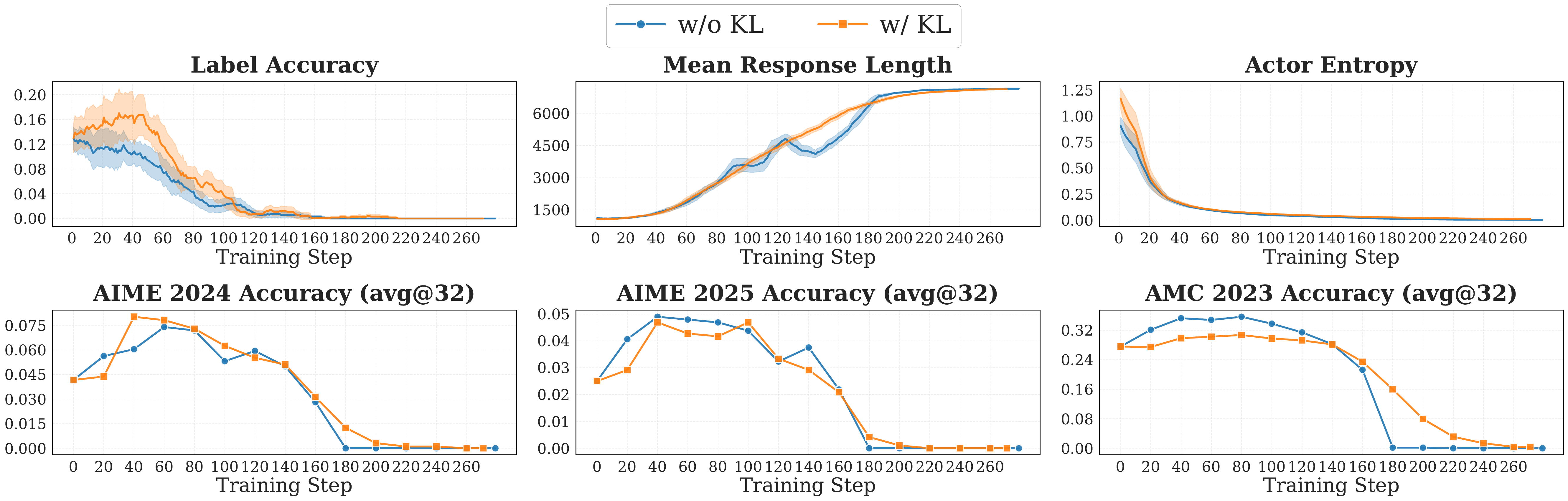}
    % \vspace{-20pt}
    \caption{Effect of KL divergence regularization on Token-Level Entropy performance.}
    \label{fig:tune_kl_token_ent}
    % \vspace{-5pt}
\end{figure*}

\begin{figure*}[!t]
    % % \vspace{-30pt}
    \centering
    \includegraphics[width=1\linewidth]{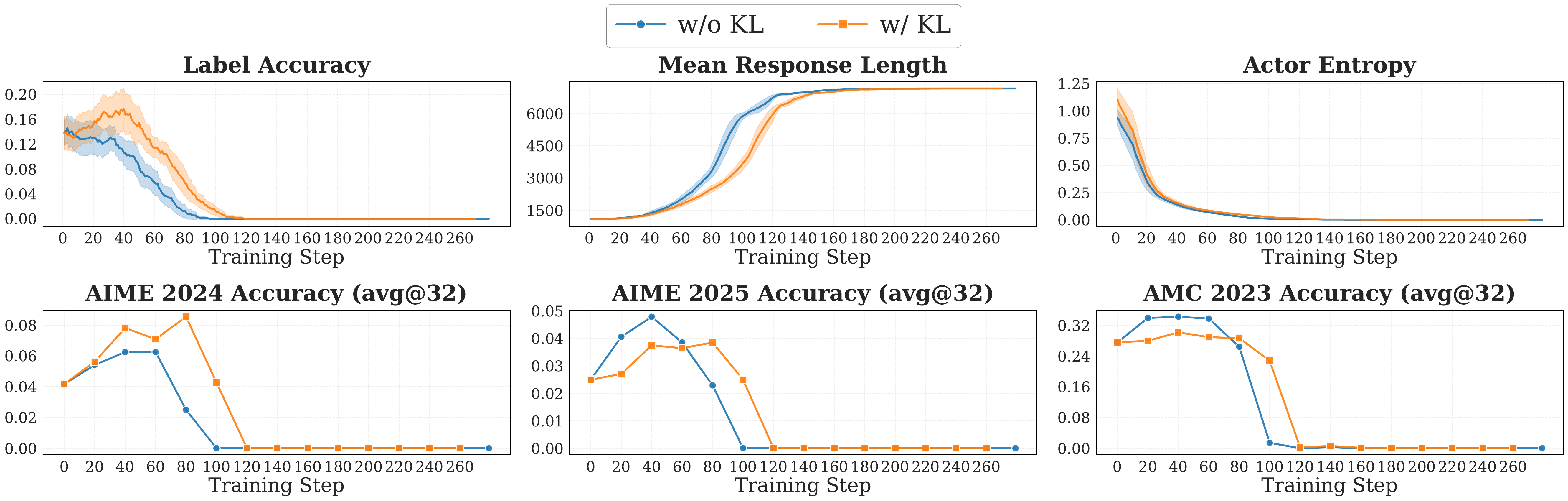}
    % \vspace{-20pt}
    \caption{Effect of KL divergence regularization on Trajectory-Level Entropy performance.}
    \label{fig:tune_kl_traj_ent}
    % \vspace{-5pt}
\end{figure*}

\begin{figure*}[!t]
    % % \vspace{-30pt}
    \centering
    \includegraphics[width=1\linewidth]{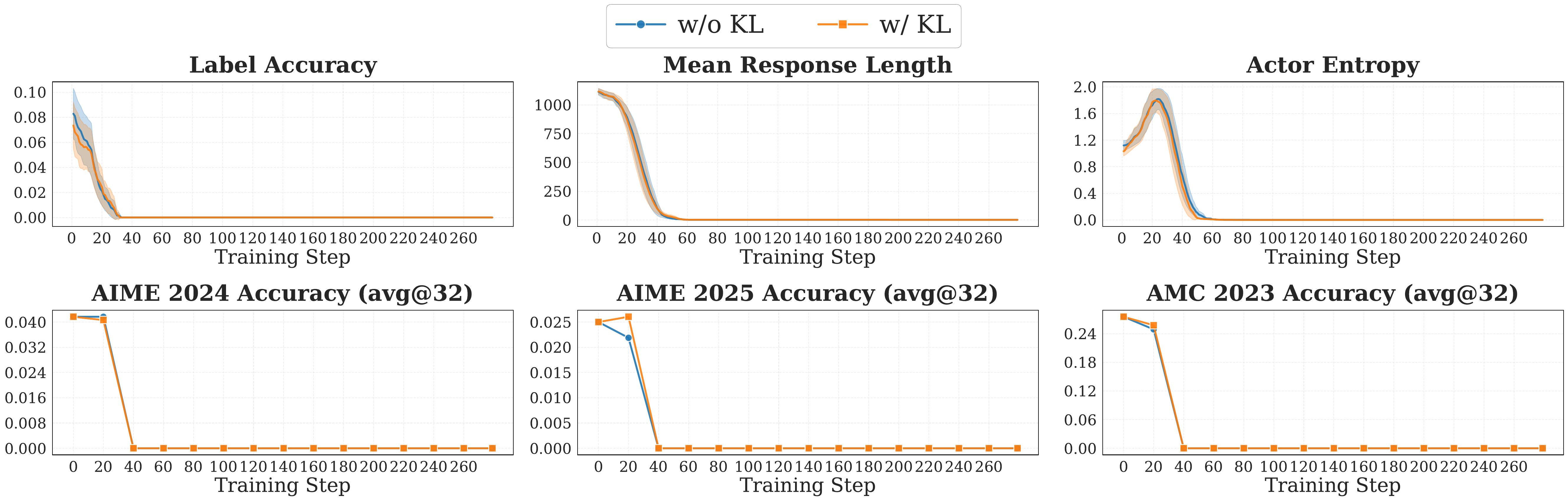}
    % \vspace{-20pt}
    \caption{Effect of KL divergence regularization on Probability performance.}
    \label{fig:tune_kl_prob}
    % \vspace{-5pt}
\end{figure*}

\noindent\textbf{KL Divergence Regularization.} KL regularization effects on certainty-based methods mirror the limited impact observed in Majority Voting, confirming that this regularization technique fails to address the fundamental instabilities inherent in training. However, subtle differences in method responses provide insights into the interaction between regularization and different uncertainty estimation approaches.

Results across all certainty-based methods (\Cref{fig:tune_kl_sc,fig:tune_kl_token_ent,fig:tune_kl_traj_ent,fig:tune_kl_prob}) show minimal impact on both training dynamics and downstream performance. KL regularization neither prevents eventual model collapse (except for Self-Certainty) nor significantly improves validation scores, consistent with our findings for Majority Voting. The underlying issue persists: regularization techniques designed for fixed reward signals cannot effectively stabilize systems where rewards themselves evolve with policy changes.

Interestingly, Token-Level and Trajectory-Level Entropy methods exhibit slightly more pronounced benefits from KL regularization, as evidenced by modest improvements in \textbf{Label Accuracy} curves. While these improvements remain insufficient to prevent collapse, they suggest that entropy-based certainty estimation may have marginally better compatibility with KL-based stabilization approaches. This observation aligns with the superior temperature robustness of these methods, indicating that entropy-based uncertainty measures may be inherently more amenable to regularization techniques than probability-based or KL-based certainty estimates.

% =================================================================

\begin{figure*}[!t]
    % % \vspace{-30pt}
    \centering
    \includegraphics[width=1\linewidth]{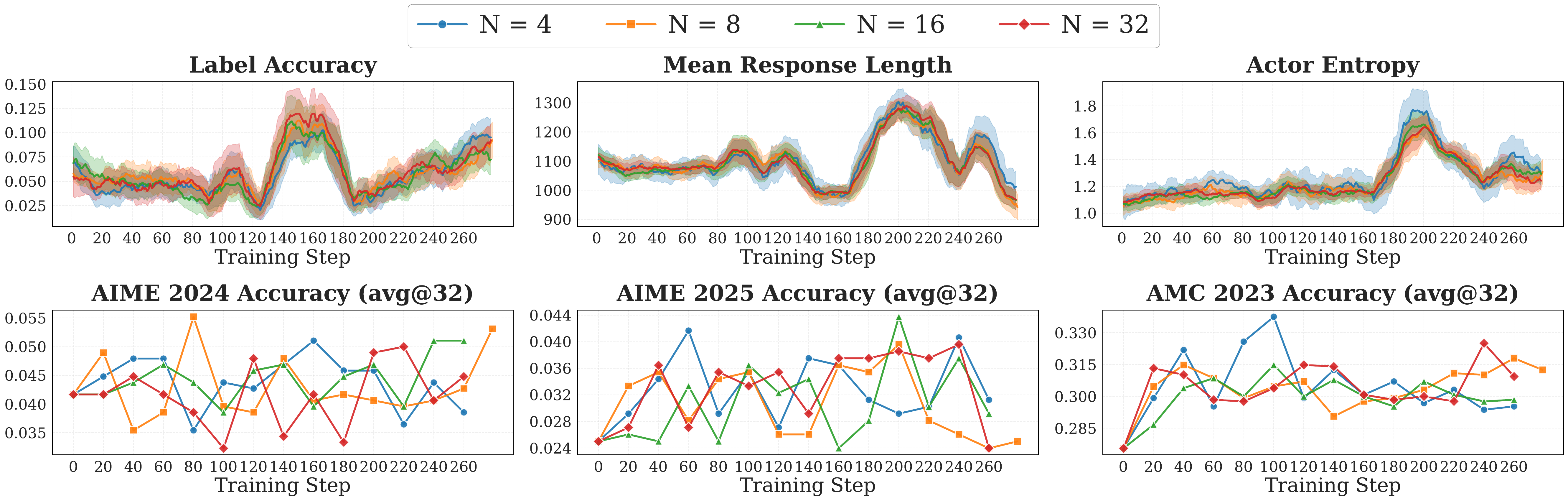}
    % \vspace{-20pt}
    \caption{Effect of rollout number on Self-Certainty performance.}
    \label{fig:tune_rollout_n_sc}
    % \vspace{-5pt}
\end{figure*}

\begin{figure*}[!t]
    % % \vspace{-30pt}
    \centering
    \includegraphics[width=1\linewidth]{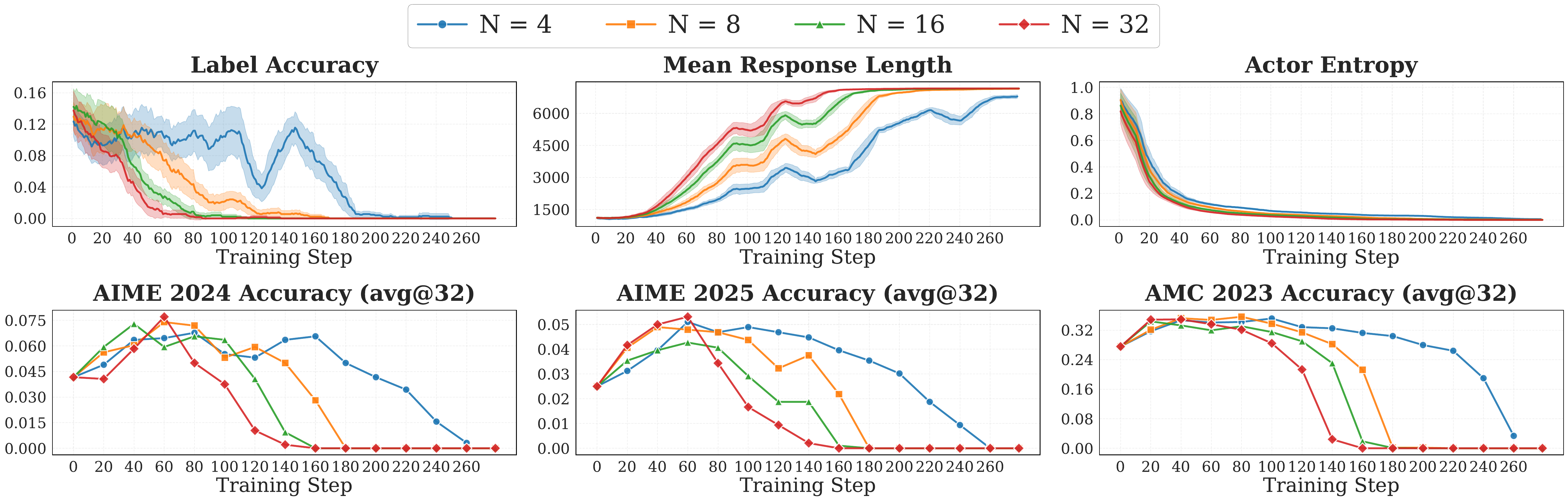}
    % \vspace{-20pt}
    \caption{Effect of rollout number on Token-Level Entropy performance.}
    \label{fig:tune_rollout_n_token_ent}
    % \vspace{-5pt}
\end{figure*}

\begin{figure*}[!t]
    % % \vspace{-30pt}
    \centering
    \includegraphics[width=1\linewidth]{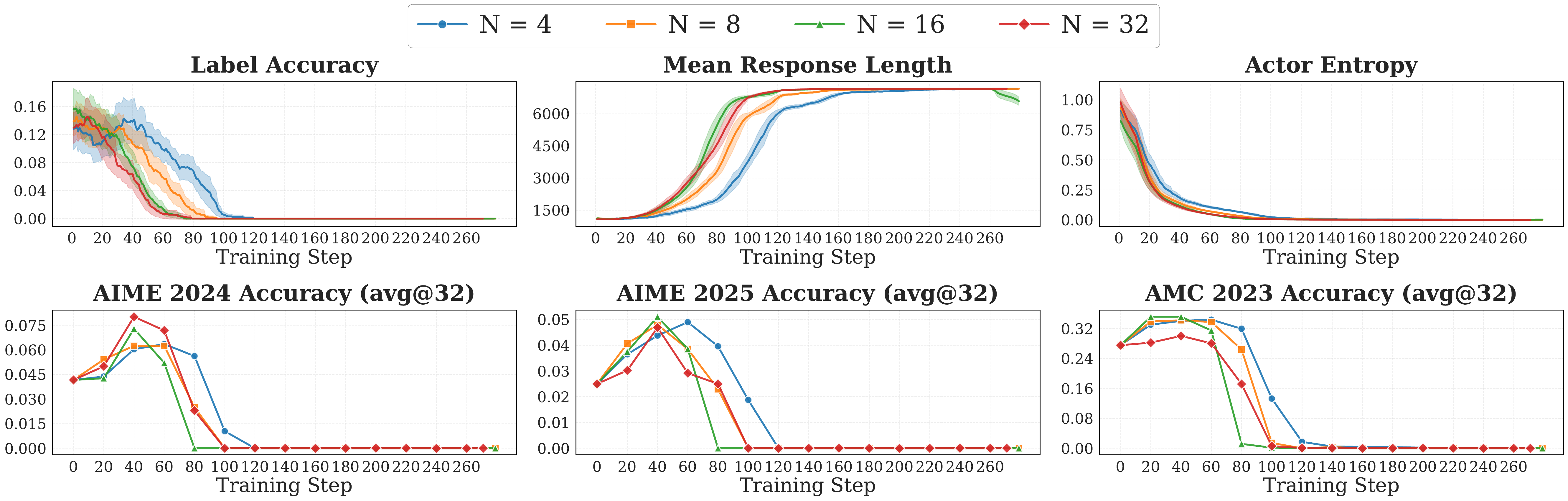}
    % \vspace{-20pt}
    \caption{Effect of rollout number on Trajectory-Level Entropy performance.}
    \label{fig:tune_rollout_n_traj_ent}
    % \vspace{-5pt}
\end{figure*}

\begin{figure*}[!t]
    % % \vspace{-30pt}
    \centering
    \includegraphics[width=1\linewidth]{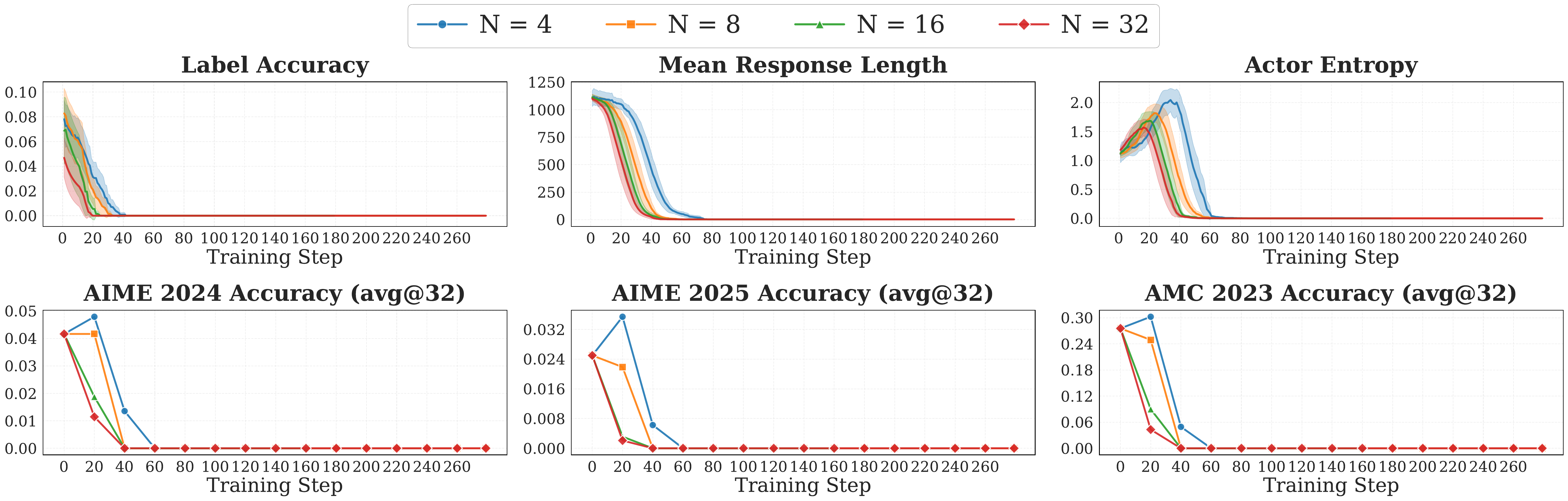}
    % \vspace{-20pt}
    \caption{Effect of rollout number on Probability performance.}
    \label{fig:tune_rollout_n_prob}
    % \vspace{-5pt}
\end{figure*}

\noindent\textbf{Number of Rollouts.} Rollout count effects reveal consistent patterns across most certainty-based methods, with one notable exception that highlights fundamental differences in underlying reward computation mechanisms. These findings provide crucial insights into the sample size requirements for reliable uncertainty estimation.

\Cref{fig:tune_rollout_n_token_ent,fig:tune_rollout_n_traj_ent,fig:tune_rollout_n_prob} demonstrate behavior parallel to Majority Voting: larger rollout counts ($N \geq 16$) accelerate model convergence and premature collapse, as evidenced by rapid degradation in validation benchmarks and \textbf{Label Accuracy}. This pattern suggests that the self-reinforcing dynamics observed in ensemble voting also manifest in certainty-based reward assignment, where higher sample sizes amplify confidence in potentially incorrect assessments, leading to faster convergence toward suboptimal solutions.

However, Self-Certainty exhibits markedly different behavior, as shown in \Cref{fig:tune_rollout_n_sc}. This method demonstrates remarkable stability across all rollout configurations, maintaining consistent performance without collapse or significant improvement. This unique characteristic stems from Self-Certainty's reliance on KL divergence between uniform and logit distribution. This fundamental difference in reward computation makes Self-Certainty inherently more robust to sample size variations, though at the cost of limited performance improvements throughout training.

\section{Other Experimental Details}

\subsection{Prompts for Self-Verification}
\label{app:verification_prompts}

\begin{tcolorbox}
[title=\textbf{Prompt 1: Adapted from RLSR~\citep{simonds2025rlsrreinforcementlearningself})}, colback=Salmon!20, colframe=Salmon!90!Black]
Verify: \\
1. "\textcolor{blue}{\{expr\}}" only contains numbers from \textcolor{blue}{\{nums\}}.\\
2. Every number from \textcolor{blue}{\{nums\}} is used in "\textcolor{blue}{\{expr\}}" and is used only once.\\
3. "\textcolor{blue}{\{expr\}}" is a valid arithmetic expression and not an equation.\\
4. "\textcolor{blue}{\{expr\}}" equals \textcolor{blue}{\{target\}}.\\
If \textbf{all} checks pass, return \texttt{\textbackslash boxed\{\textcolor{blue}{True}\}}; otherwise, return \texttt{\textbackslash boxed\{\textcolor{blue}{False}\}}.
\end{tcolorbox}

\begin{tcolorbox}[title=\textbf{Prompt 2}, colback=Salmon!20, colframe=Salmon!90!Black]
You are a strict mathematical verifier. Your task is to check whether the given expression correctly solves the arithmetic puzzle.\\[4pt]
Do NOT attempt to find or generate a new expression yourself. You must only analyze and evaluate the provided expression "\textcolor{blue}{\{expr\}}".\\[4pt]
Verification steps:\\
1. If "\textcolor{blue}{\{expr\}}" is missing, empty, or not a valid arithmetic expression (for example, if it contains words instead of numbers and operators), immediately output \texttt{\textbackslash boxed\{\textcolor{blue}{False}\}} and your task is over.\\
2. Check that "\textcolor{blue}{\{expr\}}" only uses numbers from \textcolor{blue}{\{nums\}}.\\
3. Each number from \textcolor{blue}{\{nums\}} must appear exactly once in "\textcolor{blue}{\{expr\}}".\\
4. The expression must contain only valid arithmetic operators: +, -, *, /, and parentheses.\\
5. Evaluate "\textcolor{blue}{\{expr\}}" numerically. If the computed result equals \textcolor{blue}{\{target\}} (within a tolerance of 1e-6), it passes this check.\\[4pt]
If and only if all checks pass, output \texttt{\textbackslash boxed\{\textcolor{blue}{True}\}}. Otherwise, output \texttt{\textbackslash boxed\{\textcolor{blue}{False}\}}.
\end{tcolorbox}

\begin{figure*}[!t]
    \centering
    \includegraphics[width=1\linewidth]{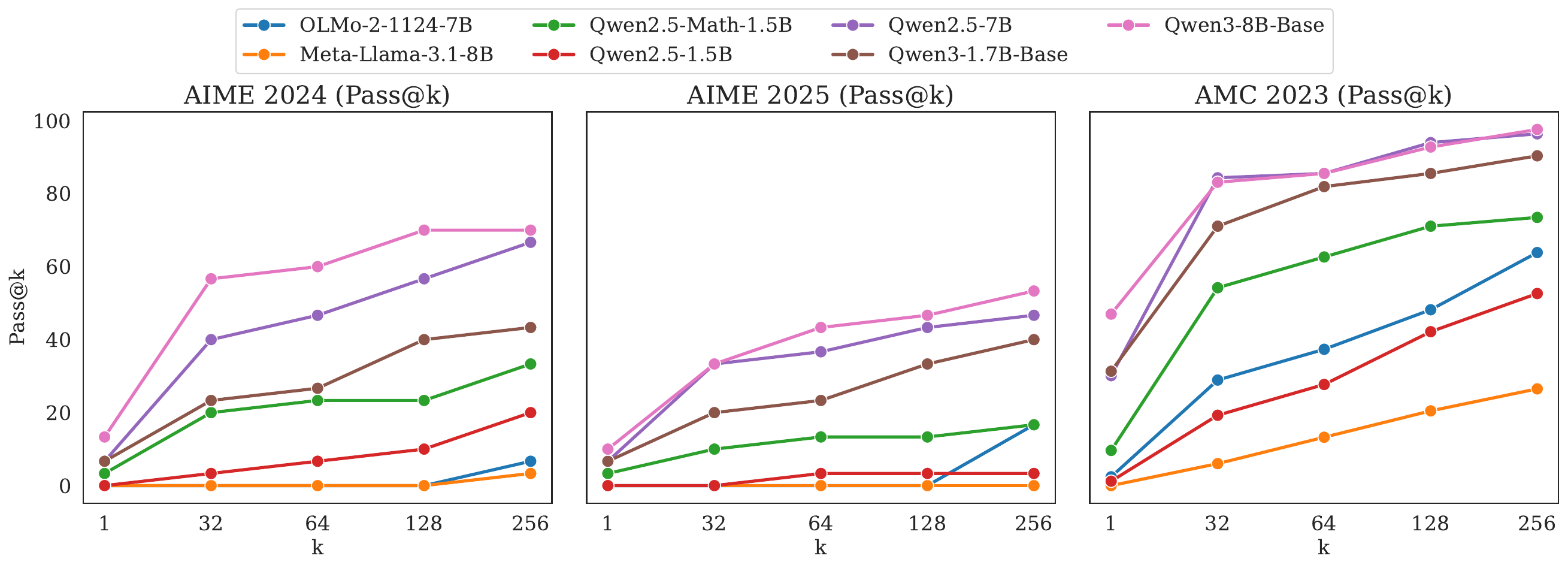}
    \caption{$Pass@k$ performance across different validation benchmarks and models, showing how trends vary as $k$ increases from 1 to 256.}
    \label{fig:5p5p1_apdx}
\end{figure*}

\subsection{Impact of Backbone Model}
\label{app:backbone}

\begin{table}[t]
\centering
\caption{Model configurations for backbone experiments. Models are categorized by family, training stage, and size.}
\resizebox{.8\textwidth}{!}{
\begin{tabular}{llccc}
\toprule
\textbf{Family} & \textbf{Model} & \textbf{Abbrev.} & \textbf{Stage} & \textbf{Size} \\
\midrule
\multirow{6}{*}{Qwen} 
 & Qwen2.5-1.5B & Q2.5-1.5B & Base & 1.5B \\
 & Qwen2.5-Math-1.5B & Q2.5-Math-1.5B & Math Base & 1.5B \\
 & DeepSeek-R1-Distill-Qwen-1.5B & DS-R1-1.5B & SFT & 1.5B \\
 & Qwen2.5-1.5B-Instruct & Q2.5-1.5B-Inst & Instruct & 1.5B \\
 & Qwen3-1.7B-Base & Q3-1.7B & Base & 1.7B \\
 & Qwen3-4B-Base & Q3-4B & Base & 4B \\
\midrule
\multirow{5}{*}{Llama} 
 & Meta-Llama-3.1-8B & L3.1-8B & Base & 8B \\
 & OctoThinker-8B-Short-Base & Octo-8B & Math Base & 8B \\
 & OctoThinker-3B-Short-Base & Octo-3B & Math Base & 3B \\
 & Llama-3.1-Tulu-3-8B-SFT & L3.1-8B-Tulu-SFT & SFT & 8B \\
 & Meta-Llama-3.1-8B-Instruct & L3.1-8B-Inst & Instruct & 8B \\
\bottomrule
\end{tabular}}
\label{tab:model_configs}
\end{table}

We investigate how backbone models influence training stability and performance across three key dimensions: training stage, model size, and architectural generation. Our analysis employs 11 models from Qwen and Llama families (detailed configurations in \Cref{tab:model_configs}), selected to provide systematic coverage of these factors. This selection is motivated by recent findings showing distinct architectural behaviors~\citep{gandhi2025cognitive} and potential data contamination concerns~\citep{wu2025reasoning}, making cross-architecture comparison essential. All models are trained on DAPO-17k using optimal hyperparameters from \Cref{app:hyperparameter} with Majority Voting as the representative intrinsic reward.

\subsubsection{Horizontal Analysis: Training Stage Impact}

\begin{figure*}[!t]
    % % \vspace{-30pt}
    \centering
    \includegraphics[width=1\linewidth]{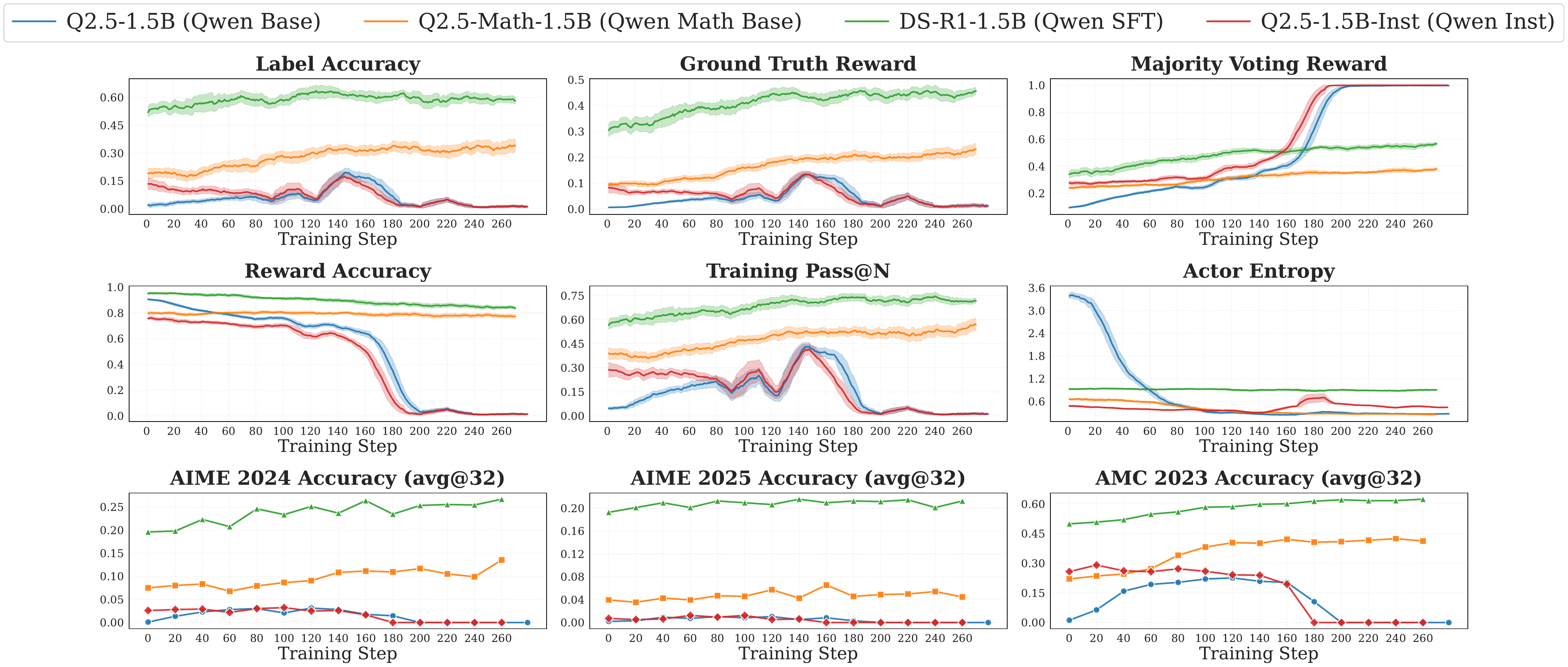}
    % \vspace{-20pt}
    \caption{Training dynamics across different training stages in Qwen family models.}
    \label{fig:qwen_stages}
    % \vspace{-10pt}
\end{figure*}

\begin{figure*}[!t]
    % % \vspace{-30pt}
    \centering
    \includegraphics[width=1\linewidth]{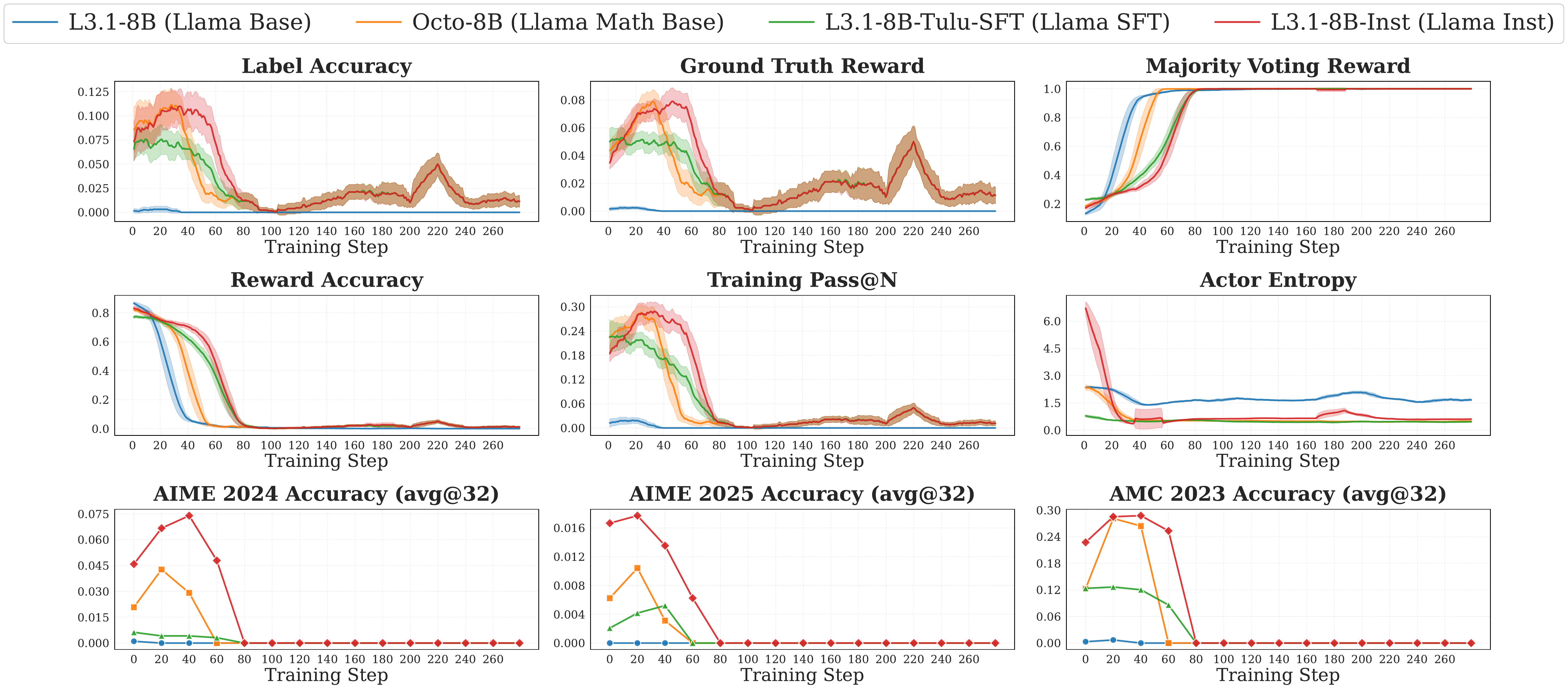}
    % \vspace{-20pt}
    \caption{Training dynamics across different training stages in Llama family models.}
    \label{fig:llama_stages}
    % \vspace{-5pt}
\end{figure*}

Training stage progression reveals distinct stability patterns between architectures. For the \textbf{Qwen family} (\Cref{fig:qwen_stages}), math-specialized and SFT models demonstrate superior stability, maintaining \textbf{Majority Voting Reward} within 0.3-0.6 while base and instruct variants reach saturation (1.0) by step 180. Math specialization and strong supervised fine-tuning (DS-R1-1.5B) create robust foundations for optimization compared to raw base models or non-math aligned instruct variants.

The \textbf{Llama family} exhibits contrasting behavior: all variants eventually succumb to reward hacking with different collapse timing, where base models fail earliest (step 40), followed by math-specialized, SFT, then instruct versions (detailed analysis in \Cref{fig:llama_stages}). This architectural difference highlights Qwen's fundamental advantage in providing genuine stability.

\subsubsection{Vertical Analysis: Scale and Generation Effects}

\begin{figure*}[!t]
    % % \vspace{-35pt}
    \centering
    \includegraphics[width=1\linewidth]{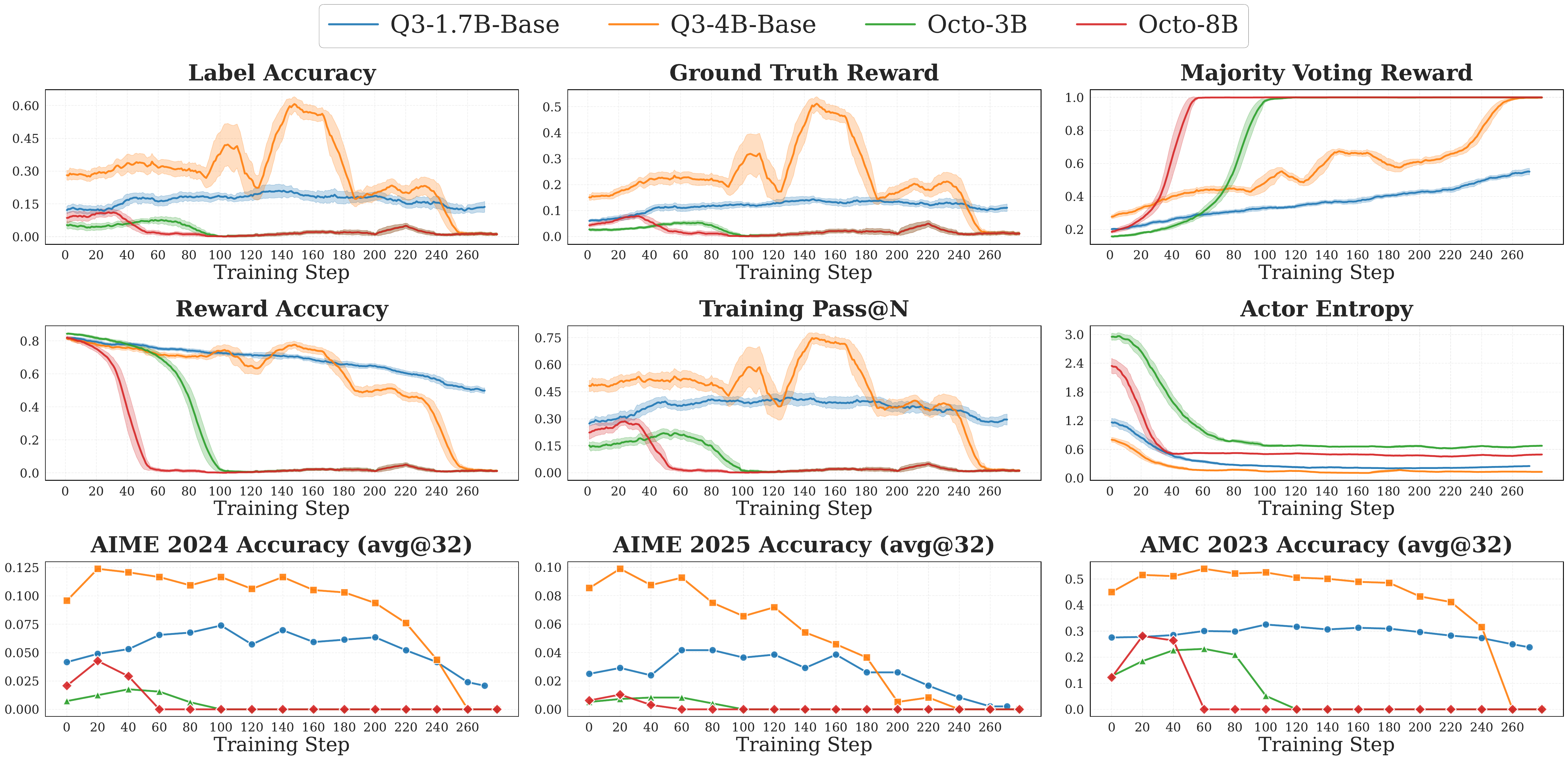}
    % \vspace{-20pt}
    \caption{Effect of model size on stability across both Qwen and Llama families.}
    \label{fig:model_size}
    % \vspace{-10pt}
\end{figure*}

\begin{figure*}[!t]
    % % \vspace{-30pt}
    \centering
    \includegraphics[width=1\linewidth]{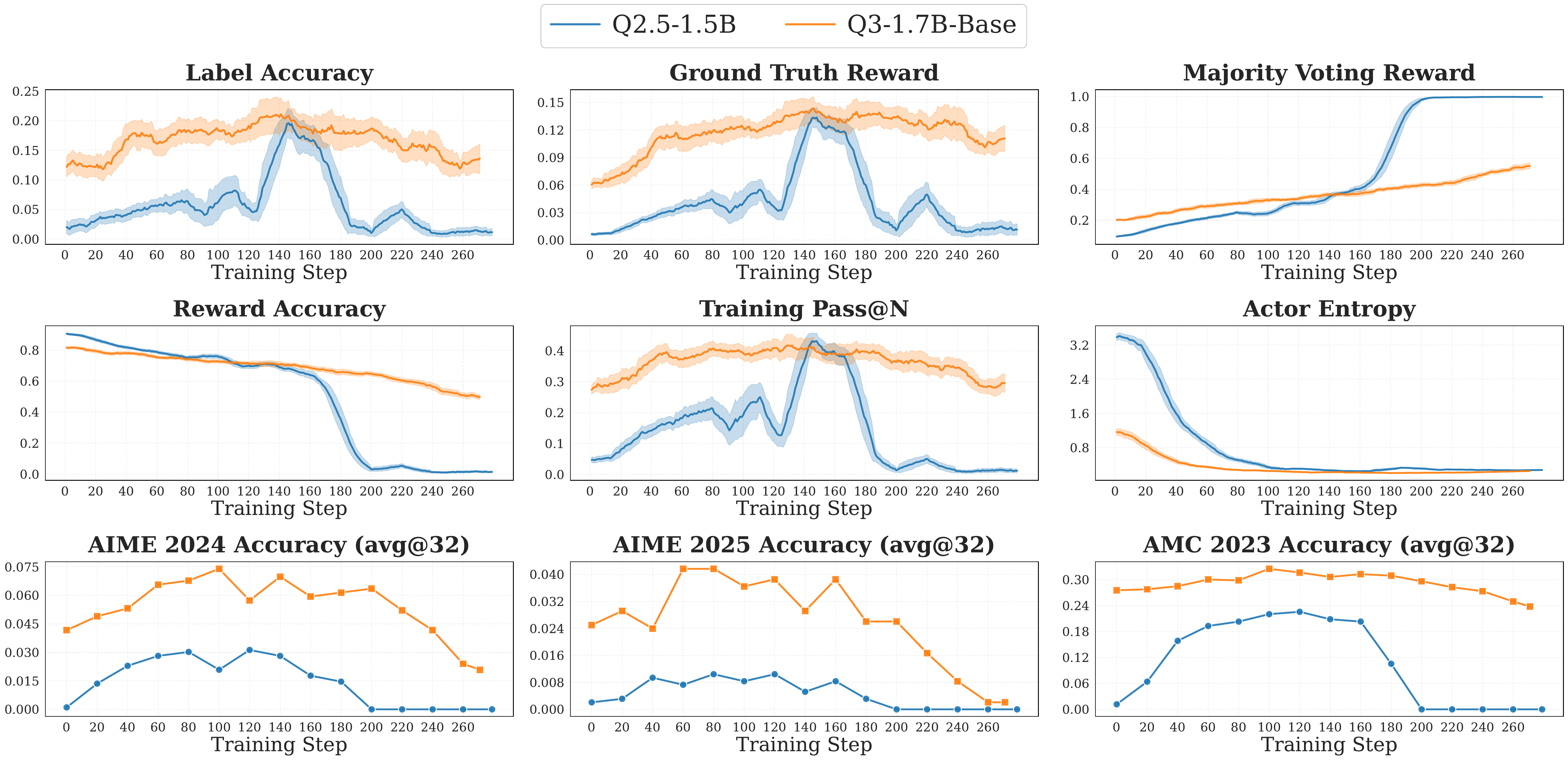}
    % \vspace{-20pt}
    \caption{Comparison of Qwen 2.5 and Qwen 3 generations across comprehensive training metrics. Results reveal improved stability in the newer generation, with Qwen3 models demonstrating more gradual and controlled training dynamics compared to Qwen2.5 counterparts.}
    \label{fig:qwen_series}
    % \vspace{-5pt}
\end{figure*}

\textbf{Model size analysis} (\Cref{fig:model_size}) reveals counterintuitive scaling effects: smaller models consistently outperform larger variants. Q3-1.7B maintains stability significantly longer than Q3-4B, while Octo-3B outlasts Octo-8B by about 40 steps. This suggests larger models' increased capacity amplifies sensitivity to noisy pseudo-rewards, accelerating convergence toward degenerate solutions and challenging conventional scaling assumptions.

\textbf{Architectural generation comparison} shows clear improvements in newer versions. Qwen3 models exhibit superior stability compared to Qwen2.5 counterparts, with Q3-1.7B-Base demonstrating more controlled \textbf{Majority Voting Reward} progression (comprehensive comparison in \Cref{fig:qwen_series}). These improvements likely stem from better-calibrated uncertainty estimates and enhanced representation learning supporting more reliable pseudo-reward computation.

\subsection{Impact of Training Dataset}
\label{app:train_dataset}

\begin{figure*}[!t]
    % \vspace{-30pt}
    \centering
    \includegraphics[width=1\linewidth]{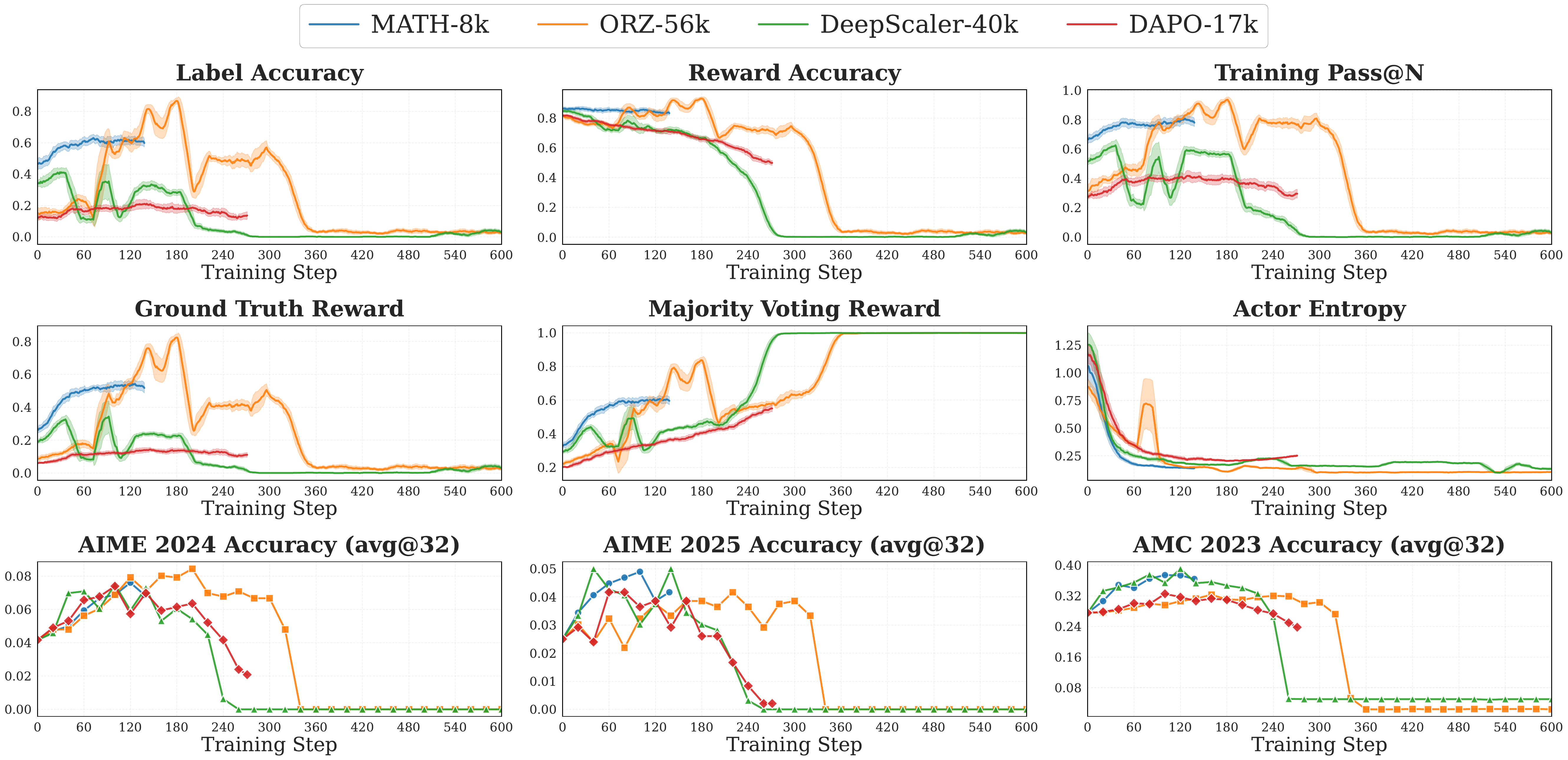}
    % \vspace{-20pt}
    \caption{Comparison of different training data sources.}
    \label{fig:datasets}
    % \vspace{-5pt}
\end{figure*}

\noindent\textbf{Setup.} We investigate how different training dataset influence training stability and performance, focusing on math reasoning, utilizing MATH-8k \citep{hendrycks2021measuring}, DeepScaleR-40k \citep{luo2025deepscaler}, DAPO-17k \citep{yu2025dapo} and ORZ-56k \citep{hu2025open}, all settings are trained on Qwen3-1.7B-Base with 1 epoch using optimal hyperparameters from \Cref{app:hyperparameter}, and also evaluated on three validation benchmarks.

\noindent\textbf{Results.} We can see from \Cref{fig:datasets}, much larger datasets (DeepScaler-40k and ORZ-56k) exhibits clear reward hacking trend, while smaller datasets settings are on its steady or rise stage, indicating that current intrinsic methods may see its short-sighted incremental improvements at the early stage, while extending it much larger training corpora, it inevitably encounter the reward hacking.

\subsection{Does Intrinsic Reward Methods Truly Improve Capabilities?}

\begin{wraptable}{r}{0.57\textwidth}
% \begin{table}[!t]
  \centering
    \caption{Comparisons before and after TTRL. The results show that TTRL-trained models significantly surpass the base models' majority-vote performance in accuracy.}
    \resizebox{.57\textwidth}{!}{
    \begin{tabular}{lcc}
    \toprule
    \textbf{Metric} & \textbf{Qwen2.5-Math-1.5B} & \textbf{Qwen2.5-Math-7B} \\ \midrule
    \textit{maj@2}   & $28.09$ & $33.23$ \\
    \textit{maj@4}   & $33.68$ & $41.20$ \\
    \textit{maj@8}   & $37.23$ & $45.73$ \\
    \textit{maj@16}  & $38.10$ & $47.98$ \\
    \textit{maj@32}  & $38.43$ & $49.19$ \\
    \textit{maj@64}  & $38.17$ & $49.87$ \\
    \textit{maj@128} & $37.86$ & $50.14$ \\
    \textit{maj@256} & $37.55$ & $50.40$ \\
    \textit{maj@512} & $37.41$ & $50.60$ \\
    \rowcolor{gray!20} \textit{maj@1024} & $37.30$ & $50.79$ \\ \midrule
    \rowcolor{lightblue} \textit{avg@32}~(w/ TTRL) & $48.90$ & $68.10$ \\
    $\Delta$ & $+11.60$ & $+17.31$ \\
    \bottomrule
    \end{tabular}
    }
    \label{tab:ttrl_maj1024}
% \end{table}
\end{wraptable}
\Cref{sec:rise} demonstrates that prevailing intrinsic reward approaches predominantly leverage uncertainty reduction as a mechanism for enhancing performance.
This observation motivates a critical question: do such approaches truly enhance a model’s capability, or do they merely improve the self-consistency of its outputs?
Take the TTRL method as an example. TTRL explicitly models the self-consistency across $m$ model outputs through majority voting and leverages it as a supervisory signal. This design seems to suggest that TTRL may simply push the model toward the performance upper bound implied by the base model under majority voting. In other words, while TTRL might steadily improve the $pass@1$ metric, it would be unlikely to surpass the base model’s $maj@m$ performance, thereby failing to deliver substantive gains beyond consistency alignment.

However, our experiments reveal the opposite, as shown in \Cref{tab:ttrl_maj1024}.
Specifically, we applied TTRL to Qwen2.5-Math-1.5B and Qwen2.5-Math-7B on AIME 2024 (30 samples) with a train batch size of 30 for 100 epochs, and directly compared the base models’ $maj@1024$ with the $pass@1$ ($avg@32$) of the TTRL-trained models. Since the majority voting performance converges rapidly once the sample size reaches $32$, $maj@1024$ can be reasonably regarded as a close approximation to $maj@\infty$. Strikingly, our results show that even the $pass@1$ metric of the TTRL-trained models significantly exceeds the $maj@\infty$ performance of the base models.
This finding demonstrates that TTRL does far more than enforce internal self-consistency: it genuinely enhances the model’s ability to generate accurate predictions. Put differently, TTRL enables the model to solve a broader range of problems than the base model, thereby delivering meaningful improvements in real-world performance.

% \section*{Author Contributions}

% \section*{Acknowledgments}
% Use unnumbered first level headings for the acknowledgments. All

% \section*{Ethics Statement}
% Authors can add an optional ethics statement to the paper. 
% For papers that touch on ethical issues, this section will be evaluated as part of the review process. The ethics statement should come at the end of the paper. It does not count toward the page limit, but should not be more than 1 page. 

\end{document}